\documentclass[12pt]{article}
\usepackage[ruled,vlined]{algorithm2e}
\usepackage{xr-hyper}
\usepackage[affil-it]{authblk}

\usepackage{graphicx}
\usepackage{subcaption}
\usepackage{float}

\usepackage{amsmath, amssymb, amsfonts, amsthm}
\usepackage{comment}
\usepackage{natbib}
\usepackage{enumitem} % for algorithm/list spacing
\usepackage{xcolor}
\usepackage{array}

\usepackage[utf8]{inputenc} % allow utf-8 input
\usepackage[T1]{fontenc}    % use 8-bit T1 fonts
\usepackage{hyperref}       % hyperlinks
\usepackage{url}            % simple URL typesetting
\usepackage{booktabs}       % professional-quality tables
\usepackage{nicefrac}       % compact symbols for 1/2, etc.
\usepackage{microtype}      % microtypography
\usepackage[symbol]{footmisc}

\newtheorem{prop}{Proposition}
\newtheorem{theorem}{Theorem}
\newtheorem{lem}{Lemma}
\newtheorem{remark}{Remark}
\newtheorem{corollary}{Corollary}

\newtheorem{definition}{Definition}
\newtheorem{claim}{Claim}

\newtheorem{assumption}{Assumption}

\providecommand{\keywords}[1]{%
    \small
    \textbf{\textit{Keywords---}} #1
}

\DeclareMathOperator*{\argmin}{arg\,min}

\newcommand{\bbE}{\mathbb{E}}
\newcommand{\bbR}{\mathbb{R}}
\newcommand{\bbC}{\mathbb{C}}
\newcommand{\bbP}{\mathbb{P}}

\newcommand{\bq}{\boldsymbol{q}}
\newcommand{\bu}{\boldsymbol{u}}
\newcommand{\bv}{\boldsymbol{v}}

\newcommand{\bX}{\boldsymbol{X}}
\newcommand{\bB}{\boldsymbol{\beta}}

\newcommand{\bY}{\boldsymbol{Y}}
\newcommand{\bZ}{\boldsymbol{Z}}

\newcommand{\cA}{\mathcal{A}}
\newcommand{\cP}{\mathcal{P}}
\newcommand{\cU}{\mathcal{U}}

\newcommand{\cF}{\mathcal{F}}

\newcommand{\cN}{\mathcal{N}}

\newcommand{\no}{\backslash}
\newcommand{\del}{\no}
\newcommand{\ind}{\mathbb{I}}
\newcommand{\stb}{\mathrm{stb}}

\newcommand{\til}{\widetilde}

\newcommand{\RNum}[1]{\mathrm{\uppercase\expandafter{\romannumeral #1\relax}}}

\newcommand{\bQ}[3]{Q^{*(\no \{#1,#2\}, #3)}}
\newcommand{\qvec}[2]{\boldsymbol{q}^{*(\no #1, #2)}}

\makeatletter

\newcommand*{\addFileDependency}[1]{% argument=file name and extension
\typeout{(#1)}% latexmk will find this if $recorder=0
\@addtofilelist{#1}
\IfFileExists{#1}{}{\typeout{No file #1.}}
}\makeatother

\begin{document}
\title{LOCO-AdaMP: Built-in LOCO Inference for Adaptive Minipatch Ensembles with Enhanced Prediction}
\author{
Yinan Cheng and Lili Zheng\\
Department of Statistics\\
University of Illinois Urbana-Champaign
}
\date{}
\maketitle
\begin{abstract}
As black-box machine learning models become increasingly common, extracting interpretations with uncertainty quantification has become a critical challenge. One popular type of interpretation is leave-one-covariate-out (LOCO) feature importance, while prior LOCO inference methods often require data-splitting or model-refitting. A recent ensemble framework, LOCO-MP, addresses these challenges using minipatches that subsample both observations and features, but massive feature subsampling can hurt prediction in high-dimensional sparse settings. Motivated by this limitation, we consider minipatch ensembles with adaptive feature sampling guided by LOCO importance, and propose LOCO-AdaMP, which enables free LOCO inference for the resulting adaptive minipatch ensemble. We show that LOCO-AdaMP yields {\em substantially improved predictive models} while retaining {\em asymptotically valid feature importance inference without data-splitting}, despite the complex dependence between the adaptive sampling distribution and the LOCO importance statistics. Our analysis relies on a careful leave-two-out perturbation bound for the iteratively updated sampling probabilities together with the stability of LOCO scores induced by observation subsampling. Empirical results on synthetic and real datasets demonstrate advantages of LOCO-AdaMP over existing methods in predictive performance, inferential power, and stability. Overall, LOCO-AdaMP provides a flexible ensemble framework (agnostic to base models) that delivers both strong predictive performance and asymptotically valid, powerful feature importance inference for regression.
\end{abstract}
\section{Introduction}
Interpretable machine learning has become an essential component of modern statistical practice, particularly as complex models achieve state-of-the-art predictive performance at the cost of transparency. Among various interpretability tools, feature importance quantifies the contribution of individual covariates to a model's predictions, helping interpret those predictions \citep{ribeiro2016should} and supporting model debugging \citep{adebayo2020debugging}.

A variety of feature importance methods have been proposed to interpret machine learning models. Traditionally, model-intrinsic measures of feature importance are frequently used, such as linear model coefficients. Recently, many model-agnostic feature importance measures have been proposed, including permutation‑based importance \citep{breiman2001random, altmann2010permutation}, partial dependence‑based approaches \citep{friedman2001greedy}, and Shapley values \citep{lundberg2017unified, lundberg2017consistent}. As \cite{gan2022model} and \cite{williamson2021nonparametric} note, many model-agnostic feature importance measures can be classified into two different, albeit related, categories: population vs. machine learning (or algorithmic) feature importance. Population feature importance examines the underlying data-generating mechanism, with the machine learning model being only a tool; while machine learning feature importance cares about how much the given machine learning model relies on each feature. Additionally, machine learning feature importance can also be closely related to population importance under certain consistency or modeling assumptions \citep{gan2022model}. In this paper, we focus on machine learning feature importance, which is essential for guiding the deployment of and diagnosing a trained machine learning model.

\paragraph{LOCO Importance.} Given the critical application of machine learning feature importance, quantifying its uncertainty level is ubiquitous. One broadly applied approach for machine learning feature importance inference is the leave-one-covariate-out (LOCO) framework, which measures the importance of a feature through the change in predictive performance induced by excluding that feature from the learning procedure \citep{lei2018distribution,rinaldo2019bootstrapping}. Consider i.i.d. samples $Z_1,\dots,Z_N \sim \cP$, where each $Z_i = (X_i,Y_i)\in \bbR^M\times\bbR$ is a pair of $M$-dimensional features and scalar response. Let $\bZ$ denote the training dataset $(Z_1,\dots, Z_N)$ of size $N$, and $Z^* = (X^*, Y^*)\sim \cP$ denote an unseen test point independent from $\bZ$. Suppose that we have trained a black-box model $\mu(\cdot;\bZ)$ for predicting $Y^*$ given $X^*$. The LOCO importance for feature $j$ is defined as
\begin{equation}\label{eq:target}
    \Delta_j(\bZ) = \Delta(\mu_{\no j},\mu) := \bbE_{Z^*\sim \cP}\left[\ell(Z^*,\mu_{\no j}(\cdot;\bZ_{:,\no j})) - \ell(Z^*,\mu(\cdot;\bZ))\,\middle|\,\bZ\right].
\end{equation}
Here, $\mu_{\no j}(\cdot;\bZ_{:,\no j})$ is the reduced predictive model trained using the same algorithm on the same data set $\bZ$ but without feature $j$, and $\ell(Z^*,\mu)$ is a loss function that evaluates the prediction error of applying model $\mu$ to sample $Z^*$. The expectation is taken over the unseen test data conditioning on the training data and consequently, the trained model $\mu$ and the reduced model $\mu_{\no j}$. Therefore, the LOCO importance $\Delta_j(\bZ)$ characterizes how important feature $j$ is to the specific trained model $\mu$, by comparing it with $\mu_{\no j}$.

\paragraph{Challenges in LOCO Importance Inference.}  
To estimate and provide inference for the LOCO importance~\eqref{eq:target}, the original LOCO method relies on data splitting, with training set yielding $\mu$ and $\mu_{\no j}$ and test set approximating the expectation over $Z^*\sim \cP$. However, data-splitting creates an unavoidable tension between predictive accuracy and inferential efficiency. Recently, \cite{gan2022model} propose an attractive ensemble-based alternative, the LOCO-MP framework, to the original LOCO method: LOCO-MP utilizes the minipatch ensemble, constructed by repeatedly subsampling both observations and features to form small training subsets, called minipatches. LOCO-MP estimates the LOCO importance of the trained minipatch ensemble in a leave-one-observation-out (LOO) fashion, thereby avoiding data splitting. Moreover, this double-subsampling structure in minipatch learning eliminates the need to refit the model separately for each left-out feature or observation given the trained ensemble. Nevertheless, one limitation of LOCO-MP lies in the uniform subsampling of features; if massive feature subsampling is in use due to memory and computational constraint, the resulting minipatch ensemble could have substantially degraded predictive accuracy, especially in high-dimensional sparse settings where only a small subset of features is informative. Motivated by this, our goal is to develop a new ensemble framework that delivers both {\em strong predictions} and {\em valid built-in LOCO importance inference that requires no extra data after training}.

\subsection{Main Contributions}
In this paper, we propose LOCO-AdaMP, an ensemble regression framework built from adaptive feature sampling and uniform observation sampling, which enjoys both {\em strong prediction performance and almost-free, asymptotically valid LOCO inference without data splitting}. Specifically, our main contribution and paper organization are summarized as follows.
\begin{itemize}
    \item In Section~\ref{sec:method}, we introduce our LOCO-AdaMP framework, which leverages the adaptive feature sampling scheme originally proposed by \cite{Liu2026model}. Different from their method which focuses on feature selection via the updated sampling probability, we focus on the adaptive minipatch ensemble itself as a predictor and its feature importance inference. Compared to minipatch ensembles built from uniform double-subsampling, we show that adaptive minipatch ensemble has substantially improved prediction performance while it also preserves the key attractive property of LOCO-MP: almost free inference that requires no extra data or computation after model training, as we can simply aggregate trained minipatches to obtain leave-one-out estimates for the LOCO importance. We include a detailed discussion on the interpretation of our LOCO-AdaMP inference target in Section~\ref{sec:disc_target}.
    \item Although adaptive sampling dramatically improves prediction, it also brings nontrivial challenges to establishing inference validity due to the involved dependency between the iteratively updated feature sampling distribution and the feature importance estimates. In Section~\ref{sec:theory}, we address this challenge and show that our almost ``free" confidence intervals achieve asymptotically valid coverage given any bounded base models and any data distributions subject to mild assumptions on ensemble hyperparameters. Our proof develops some new technical tools such as the leave-two-out analysis and a careful induction argument over iterations (see proof sketch in Section~\ref{sec:sketch}). 
    \item In Section~\ref{sec:sims}, we conduct extensive simulation studies comparing LOCO-AdaMP with prior LOCO methods. The results show that given the same amount of data for training and inference, LOCO-AdaMP yields substantially stronger predictors with valid inference; it also achieves greater statistical efficiency as reflected by improved power and more informative confidence intervals.
    \item In Section~\ref{sec:realdata}, we demonstrate the practical utility of LOCO-AdaMP through a case study on Alzheimer’s disease using the ROSMAP RNA-sequencing dataset. Compared with existing model-agnostic feature importance inference methods, LOCO-AdaMP achieves improved predictive performance and identifies more plausible genes with greater stability.
\end{itemize}

\subsection{Related Work}
A large body of recent work has sought to develop statistical inference for population feature importance. Notable examples include the Floodgate framework \citep{zhang2020floodgate, pmlrwang24ad}, $R^2$-based and Shapley-based variable importance measures \citep{williamson2021nonparametric, williamson2020efficient}, the generalized covariance measure (GCM) \citep{shah2020hardness, scheidegger2022weighted}, the projected covariance measure (PCM) \citep{lundborg2024projected}, and disentangled feature importance (DFI) \citep{du2025disentangled}. The cost of repeatedly fitting reduced models has also motivated more efficient approaches, including lazy training \citep{gao2022lazy} and warm starts with early stopping \citep{sun2024reliable}. While these methods focus on algorithm-independent notions of a feature's predictive contribution under the underlying distribution, machine learning feature importance offers a complementary perspective by assessing its contribution to the predictive performance of a fitted model. We refer readers to \cite{gan2022model} for a more detailed discussion of the relationship between these perspectives.

Motivated by this complementary perspective, various methods have been developed for evaluating machine learning feature importance. While the LOCO importance considered in our work concerns a fitted machine learning model, \cite{verdinelli2024feature} examine LOCO and Shapley-based importance at the population level, with particular attention to how feature dependence affects their interpretation. Among other methods for assessing machine learning feature importance, one prominent class consists of conditional permutation-based methods: \cite{strobl2008conditional} develop a conditional permutation scheme for computing feature importance in random forests, \cite{watson2021testing} propose the model-agnostic conditional predictive impact (CPI) using model-X knockoffs, and \cite{chamma2023statistically} provide alternative conditional permutation procedures with asymptotic type-I error control. Another approach, model class reliance (MCR), characterizes the range of feature importance across an entire class of well-performing prediction models \citep{fisher2019all}.

Ensemble and subsampling methods have long been used to improve prediction and enhance stability \citep{ho1998random,soloff2024bagging}. In high-dimensional settings, random-subspace and random-projection aggregation methods have been studied both for their predictive performance and for their ability to identify important variables \citep{cannings2017random,tian2021rase,wang2024sharp}. On the theoretical side, appropriately tuned observation-subsampled ensembles can reduce prediction risk \citep{patil2023bagging}; under an orthogonal design, feature-subsampled ensembles of greedy forward-selection estimators can reduce both bias and variance \citep{chen2025revisiting}. Beyond prediction, ensemble methods have enabled uncertainty quantification through random-forest-based inference \citep{mentch2016quantifying,cattaneo2026inference} and minipatch-based LOCO inference \citep{gan2022model}. Meanwhile, minipatch learning and adaptive subsampling methods have been proposed for feature selection or feature ranking in large-scale learning \citep{yao2020feature, toghani2021mp, chen2025top}. The works most closely related to our paper are \cite{gan2022model}, which develops feature importance inference based on uniformly subsampled minipatches, and \cite{Liu2026model}, which introduces an adaptive feature subsampling strategy for minipatch ensembles in the context of feature selection.

\section{LOCO Inference for Adaptive Minipatch Ensembles}\label{sec:method}
Throughout this paper, we focus on the regression setting where the response is continuous. Given a trained black-box model $\mu(\cdot;\bZ)$, we are interested in {\em the predictive importance of any given feature $j$ to $\mu$}. One natural metric is the LOCO importance defined in \eqref{eq:target}, which measures the prediction error change if {\em removing feature $j$ from the training process} by taking expectation over $Z^*$. In \eqref{eq:target}, the loss function $\ell$ often takes the form of the squared error or absolute error in regression. We write $\Delta_j(\bZ)$ when the fitted models are clear from context, and $\Delta(\mu_{\no j},\mu)$ when emphasizing that the contrast is between the reduced model $\mu_{\no j}$ and the full model $\mu$.

Through a comparison between $\mu(\cdot)$ with the reduced model $\mu_{\no j}(\cdot)$, $\Delta_j$ measures the additional predictive value provided by feature $j$ for $\mu(\cdot)$ beyond the remaining features. A positive (negative) $\Delta_j$ indicates that the feature contributes positively (resp. negatively) to the model. 
The LOCO importance metric has been widely adopted in the feature importance literature; detailed discussions of LOCO and related metrics can be found in the literature \citep{lei2018distribution,gan2022model,Covert2020ExplainingBR}.

\subsection{Background: Prior Inference Methods for LOCO}

\begin{figure}[htbp]
\centering
\begin{subfigure}[t]{0.4\textwidth}
\centering
\includegraphics[width=1\linewidth]{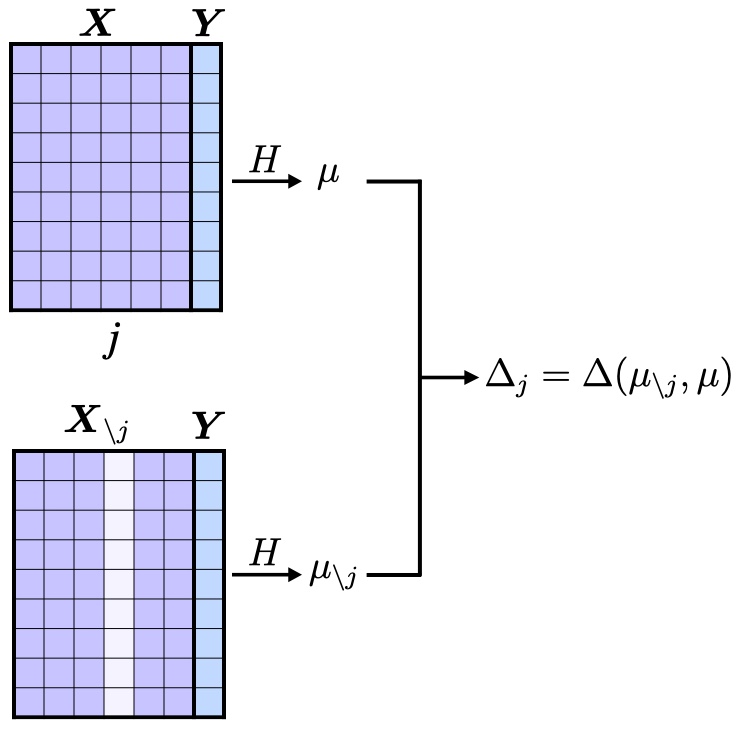}
\caption{LOCO Metric}
\label{fig:target}
\end{subfigure}
\hfill
\begin{subfigure}[t]{0.58\textwidth}
\centering
\includegraphics[width=1\linewidth]{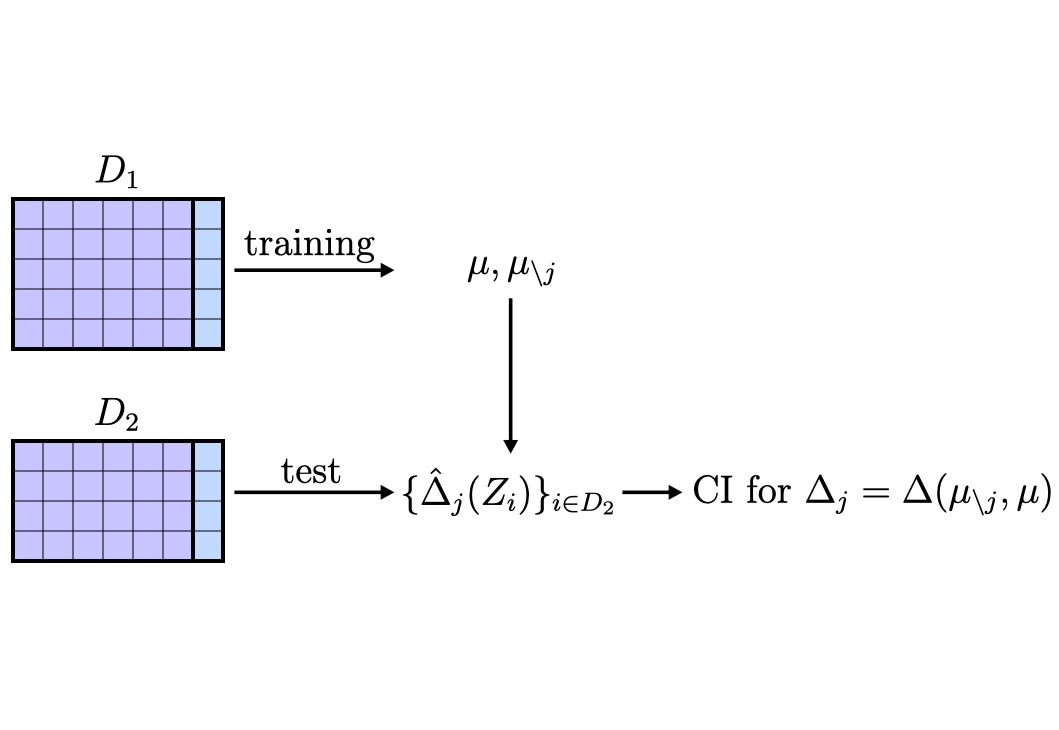}
\caption{Inference method for LOCO: LOCO-Split}
\label{fig:locosplit}
\end{subfigure}
\caption{\small Illustration of the LOCO metric and LOCO-Split.}
\end{figure}

Before proposing our inferential framework, we would like to start by introducing two major prior methods for LOCO inference that motivate our approach: LOCO-Split and LOCO-MP.

\paragraph{LOCO-Split.} To perform statistical inference for the LOCO importance \eqref{eq:target}, \cite{lei2018distribution} propose to split data into a training set $D_1$ that leads to $\mu$ and $\mu_{\no j}$ and a test set $D_2$ for evaluating their prediction error difference. The confidence interval is then constructed via a normal approximation or a nonparametric test. We refer to this method as ``LOCO-Split'' to distinguish it from other methods that do not use data-splitting. An illustration of LOCO-Split is presented in Figure~\ref{fig:locosplit}. 

\paragraph{LOCO-MP: LOCO inference is almost free within the minipatch framework.} Although LOCO-Split comes with assumption-light guarantees, its reliance on data-splitting inherently creates a tension between training and inference power. Interestingly, \cite{gan2022model} propose LOCO-MP (LOCO inference for minipatch ensembles), an ensemble framework with built-in inference for LOCO importance that requires no additional data or computation after training. Specifically, ``minipatches'' refers to randomly subsampled small subsets of observations and features $\{(I_k,F_k)\}_{k=1}^K$ where $I_k,\,F_k$ are independently, uniformly sampled from $[N]$ and $[M]$ without replacement, with sizes $n,\,m$, respectively. On each minipatch $(I_k,F_k)$, they apply any given base learner to $\bZ_{I_k,F_k}$ to train a model $\mu_k(\cdot)$, and the final ensembled predictor is their average $\mu(\cdot)=\frac{1}{K}\sum_{k=1}^K\mu_k(\cdot)$. This is very similar to bagging where observations are subsampled/bootstrapped, while in minipatch learning, features are also subsampled together with observations. After the minipatch ensembles are trained, given any feature $j$, they aim to perform inference for $\Delta_j(\bZ)$ in \eqref{eq:target}, with $\mu(\cdot)$ being the ensembled minipatch predictor and $\mu_{\no j}(\cdot)$ being a hypothetical reduced model trained using minipatch learning but without feature $j$. 

\begin{figure}[!htb]
\centering
\begin{subfigure}[t]{\textwidth}
\centering
\includegraphics[width=1\linewidth]{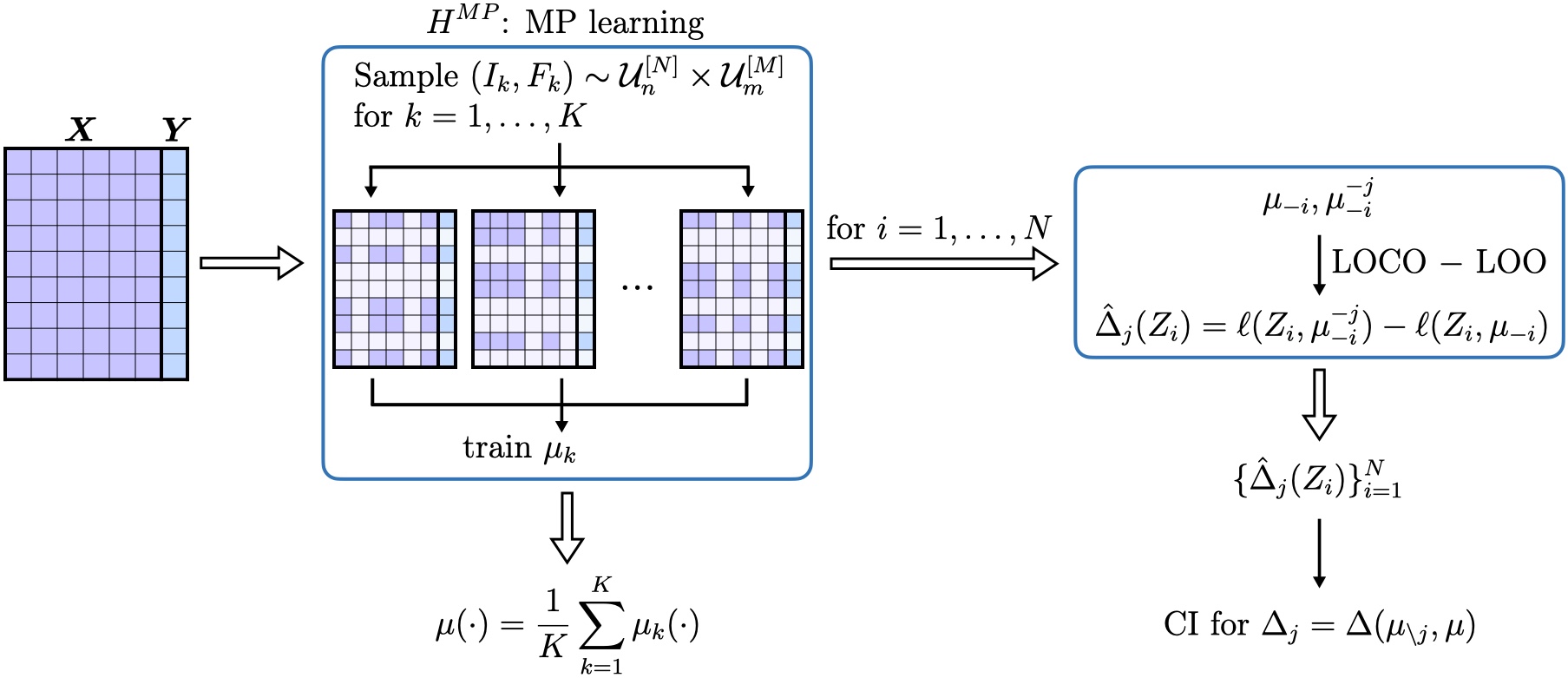}
\caption{LOCO-MP}
\label{fig:locomp}
\end{subfigure}

\vspace{0.2cm}

\begin{subfigure}[t]{\textwidth}
\centering
\includegraphics[width=1\linewidth]{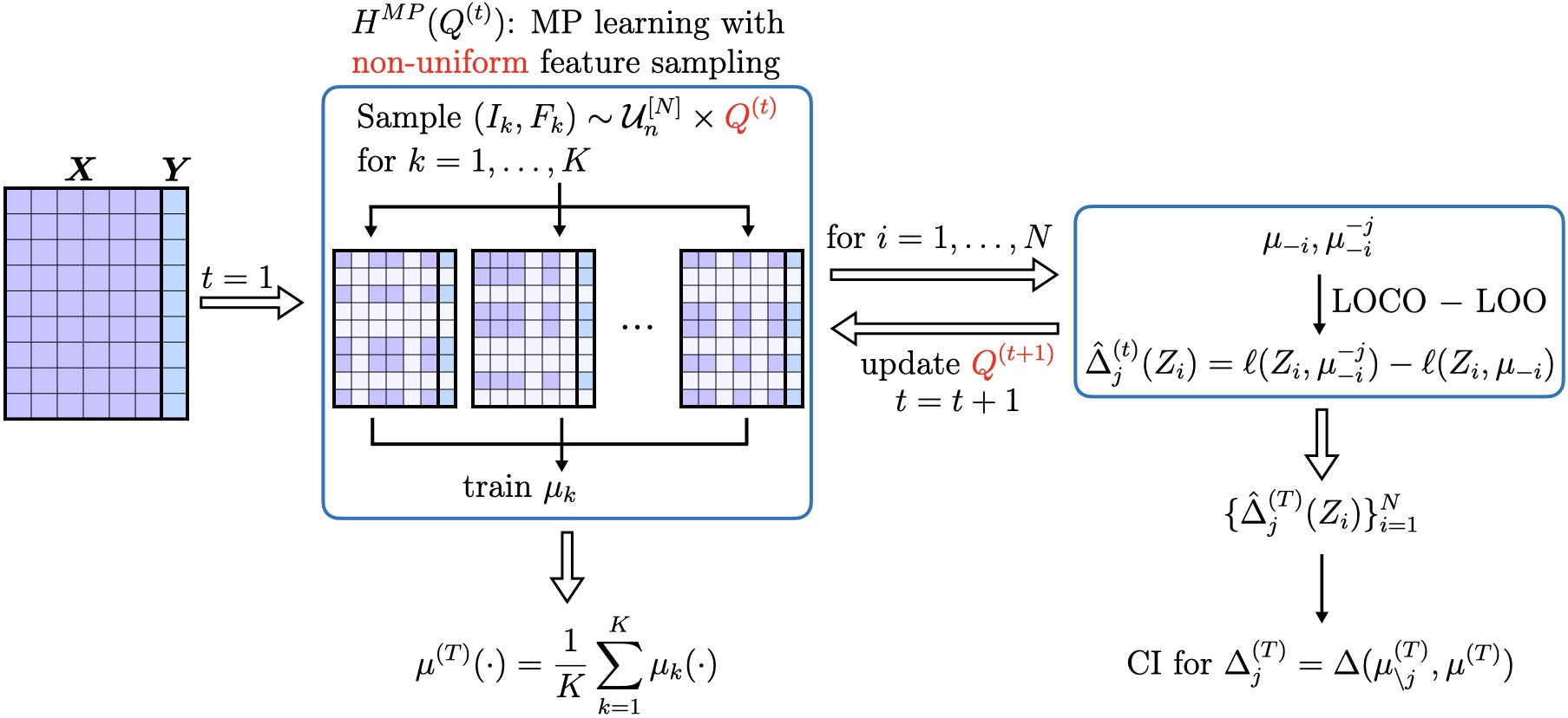}
\caption{LOCO-AdaMP}
\label{fig:locoadamp}
\end{subfigure}
\caption{\small LOCO-MP utilizes uniform feature sampling to yield minipatches, while LOCO-AdaMP utilizes non-uniform feature sampling to yield minipatches. $\cU_n^{[N]}$ and $\cU_m^{[M]}$ represent uniform distributions, and $Q^{(t)}$ represents the non-uniform feature sampling distribution. Formal definitions of these sampling distributions are included in Section~\ref{sec:app_def_not} in the appendix.}
\end{figure}

To achieve the inference goal, LOCO-MP does not require extra model refitting or calibration data. Instead, they
use leave-one-observation-out to evaluate the full model's performance: for $1\leq i\leq N$, compute $\ell(Z_i, \mu_{-i}(\cdot;\bZ))$ with $\mu_{-i}(   X_i;\bZ) = \frac{1}{\sum_{k=1}^{K} \mathbb{I}(i \notin I_k)} \sum_{k=1}^{K} \mathbb{I}(i \notin I_{k}) \mu_{k} (X_{i};\bZ)$, which consists of simple model averaging. They then leave both observation $i$ and feature $j$ out to evaluate the LOCO-LOO error: $\ell(Z_i, \mu_{-i}^{-j}(\cdot;\bZ))$ with $\mu_{-i}^{-j}(X_i;\bZ) = \frac{1}{\sum_{k=1}^{K} \mathbb{I}(i \notin I_k) \mathbb{I}(j \notin F_k)} \sum_{k=1}^K \mathbb{I}(i\notin I_k)\mathbb{I}(j\notin F_k) \mu_{k}(X_i;\bZ)$. Putting the LOCO-LOO and LOO errors together, they collect $\{\hat{\Delta}_j(Z_i)=\ell(Z_i, \mu_{-i}^{-j}(\cdot;\bZ)) - \ell(Z_i, \mu_{-i}(\cdot;\bZ))\}_{i=1}^N$ and construct an asymptotic normal confidence interval.
Importantly, the procedure above requires neither model refitting nor additional held-out data once the minipatch learning process is complete. Thus, LOCO importance inference incurs essentially no additional cost under the LOCO-MP framework, making it an attractive low-cost byproduct of training for practitioners. An illustration of LOCO-MP is presented in Figure~\ref{fig:locomp}.

One limitation of LOCO-MP, however, is that both prediction and inference are tied to the minipatch ensemble framework. This naturally raises the question of whether minipatch predictors are themselves competitive from a predictive perspective. Feature subsampling in minipatch ensembles can induce implicit regularization similar to a ridge penalty, potentially improving prediction in noisy and correlated settings \citep{gan2022model,yao2021minipatch,lejeune2020implicit}. {\em However, in sparse settings where only a small subset of features carries strong signal, uniform feature subsampling may be less effective}: when the minipatch size is small, many minipatches may exclude all signal features. This motivates us to consider alternative strategies beyond uniform subsampling.

\subsection{Background: Adaptive Sampling for Minipatches} In fact, the idea of adaptive feature subsampling in minipatches has also been explored for related, though distinct, problems. \cite{yao2020feature} propose a minipatch-based method for feature selection in linear models, where adaptive sampling substantially improves selection performance. More directly related to our work, \cite{Liu2026model} design an adaptive feature subsampling strategy within the LOCO-MP framework for model-agnostic feature selection, called LOCO-guided Adaptive Minipatch Sampling (LAMPS). LAMPS starts with running LOCO-MP with uniform sampling over all features; then each feature is sampled with probability approximately proportional to its LOCO importance score from the previous iteration, so that the sampling distribution gradually shifts toward more promising features. While this adaptive strategy is effective for feature selection, it remains unclear whether it improves the ensemble predictor itself, which we call AdaMP (adaptive minipatch ensemble), and whether valid LOCO importance inference can still be obtained for AdaMP with iteratively updated sampling probabilities.

\paragraph{Boosted predictive performance via AdaMP.} 
To evaluate the predictive gain through the adaptive sampling strategy in AdaMP \citep{Liu2026model}, we conduct simulation studies across linear, non-linear data generating processes with minipatch base learners being ridge regression, decision trees, and kernel ridge regression. A low-dimensional setting $(M=50,N=200)$ and a high-dimensional setting $(M=500,N=200)$ are considered. Figure~\ref{fig:locoada_pred} shows that adaptive sampling substantially improves the predictive performance of minipatch ensembles, and the improvement is more significant in high dimensions. More simulation details are included in Section~\ref{sec:sim_pred}. This clear predictive advantage of AdaMP motivates us to consider LOCO inference built within this new framework, with the hope that this will lead to a much stronger ensemble predictor also equipped with the almost-free inference for feature importance.
\begin{figure}[htb]
    \centering
    \includegraphics[width=1\linewidth]{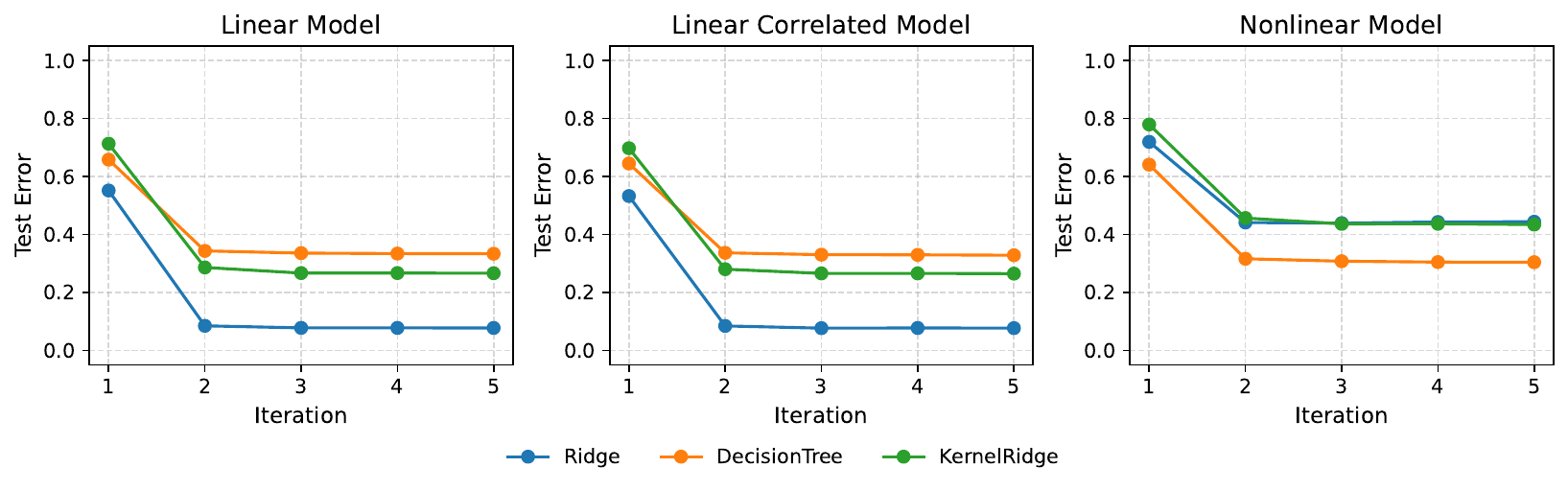}

    \includegraphics[width=1\linewidth]{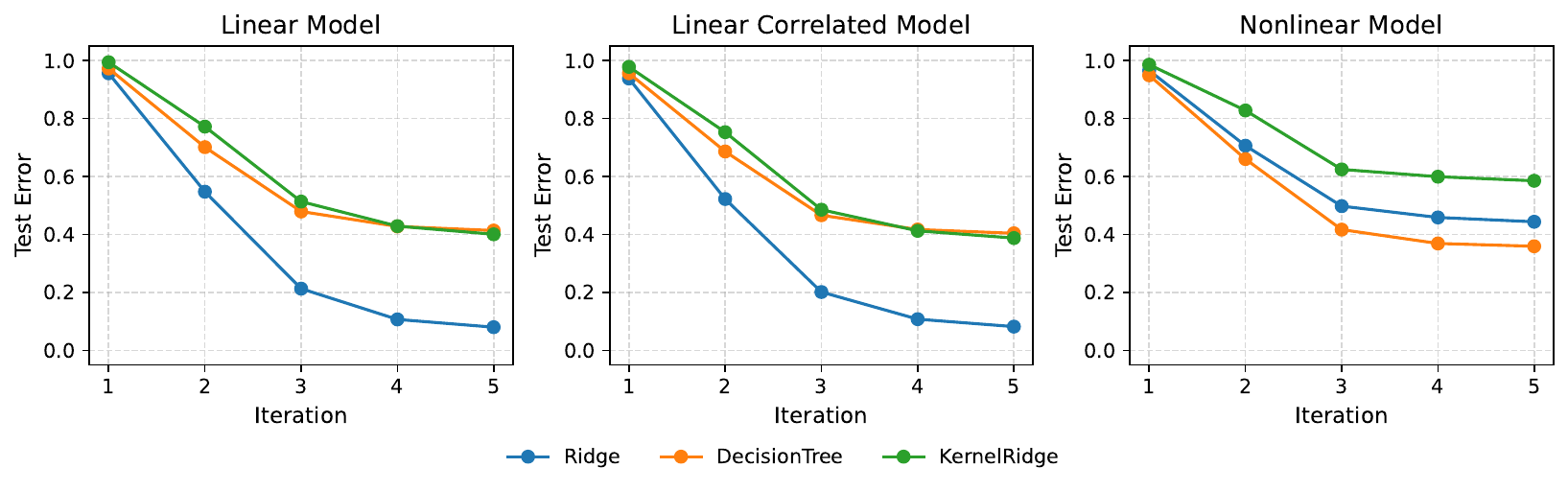}
 
    \captionsetup{font=normalsize}
    \caption{\small Prediction errors over iterations in the low-dimensional ($M=50,N=200$, top row) and high-dimensional ($M=500,N=200$, bottom row) settings, averaged over $100$ simulation replicates. We use the same simulation setup as in Section~\ref{sec:sim_pred}.}
    \label{fig:locoada_pred}
\end{figure}

\subsection{Built-in LOCO Inference for Adaptive Minipatch Ensembles: LOCO-AdaMP}
In this section, we introduce the LOCO-AdaMP framework, which performs LOCO inference for Adaptive minipatch ensembles without requiring extra model fitting or held-out data after training. The idea is very similar to LOCO-MP \citep{gan2022model}, and the main difference we made is adopting the adaptive feature subsampling procedure initially developed for feature selection by \cite{Liu2026model}. Although the inference idea seems a straightforward extension of LOCO-MP, we will see in Section~\ref{sec:theory} that the main challenge lies in the validity theory despite the adaptive procedure that iteratively reuses all observations during training. In Section~\ref{sec:disc_target}, we will also discuss in detail the interpretation of the inference target in LOCO-AdaMP, and how it relates to and differs from the targets of LOCO-Split and LOCO-MP.

\paragraph{Adaptive Minipatch Ensembles.} We first review the adaptive minipatch ensemble framework proposed in \cite{Liu2026model}. Motivated by LOCO-MP \citep{gan2022model}, the adaptive minipatch ensemble framework iterates between two modules for $T$ iterations: 
\begin{itemize}
    \item[(i)] {\em Training} module: for $1\leq k\leq K$, independently sample random indices $I_k\subset [N]$ and $F_k\subset[M]$, then repeatedly apply the given base learner on each minipatch $(I_k,F_k)$ to obtain $\mu_k(\cdot)$. The sampling mechanism is as follows: $I_k$ is a size-$n$ subset sampled uniformly at random without replacement from $[N]$; while $F_k=\{1\leq j\leq M: S_{k,j}=1\}$ where $\{S_{k,j}\}_{j=1}^M$ independently follow Bernoulli distribution with probability $q_j$, conditioning on the event that $\sum_{j=1}^M S_{k,j}>0$. Here, the probability vector $\bq=(q_1,\dots,q_M)^\top$ determines the distribution of feature set $F_k$, satisfying $\sum_{j=1}^M q_j=m$ for a given feature size parameter $m$. $\bq$ is initialized as $\frac{m}{M}\mathbf{1}_M$ in the first iteration and updated later in module (ii). The training module is summarized in Algorithm~\ref{algo:adamp_training}. 
    \item[(ii)] {\em Feature importance} module: Given the trained minipatch ensemble $\{\mu_k\}_{k=1}^K$, for each feature $1\leq j\leq M$, compute the feature importance scores $\{\hat{\Delta}_j(Z_i)\}_{i=1}^N$. Here, $\hat{\Delta}_j(Z_i)=\ell(Z_i,\mu^{-j}_{-i})-\ell(Z_i,\mu_{-i})$, where $\mu_{-i}$ and $\mu^{-j}_{-i}$ are the LOO and LOCO-LOO predictors obtained through simple model averaging based on the trained ensemble $\{\mu_k\}_{k=1}^{K}$. The detailed steps correspond to Algorithm~\ref{algo:adamp_FI} with \verb|Inference| set as FALSE. The feature sampling probability vector $\bq$ is then updated based on the feature importance estimates $\{\bar{\Delta}_j=\frac{1}{N}\sum_{i=1}^N\hat{\Delta}_j(Z_i)\}_{j=1}^M$, using the scheme summarized in Algorithm~\ref{algo:samp_prob}.   
\end{itemize}

\begin{algorithm}[htbp]
\caption{AdaMP: Training Module}
\label{algo:adamp_training}
		\noindent{\textbf{Input}}: Training pairs $\bZ=(\bX,\bY)$ with $M$ features and sample size $N$, minipatch observational size $n$, feature sampling probability $\bq=(q_1,\dots,q_M)$; number of minipatches $K$; base learner $H$.\\
        \noindent{For $k=1,...,K$:}
				\begin{enumerate}
                    \item Randomly subsample $n$ observations without replacement, $I_k \subset [N]$.
                    \item Let $F_k = \emptyset$. 
                    \item While $F_k=\emptyset$:
                    \begin{itemize}
                        \item For $j=1,\dots,M$, sample  $S_{k,j}\sim \text{Bernoulli}(q_j)$. If $S_{k,j}=1$, $F_k = F_k \cup \{j\}$. 
                    \end{itemize}
				    \item Train prediction model $\mu_k$ on $\boldsymbol{Z}_{I_k,F_k}$: for any $X\in \bbR^M$, $\mu_k(X;\bZ) = H(\boldsymbol{Z}_{I_k,F_k})(X_{F_k})$, also denoted by $\mu_{I_k,F_k}(X;\bZ) $.
                \end{enumerate}
		\textbf{Output}: $\{\mu_k(\cdot)\}_{k=1}^K$, $\{(I_k,F_k)\}_{k=1}^K$.
\end{algorithm}

\cite{Liu2026model} considers $O(\log(M/m))$ iterations of the training and feature importance modules and focuses on the final feature sampling probability vector $\bq$ for feature selection. However, another important output of this framework is an updated minipatch ensemble predictor, built from uniform observational subsampling and weighted feature subsampling with the iteratively updated $\bq$. In the following, we denote the probability vector used in the $T$th iteration by $\bq^{(T)}$, and the minipatch predictor built from $\bq^{(T)}$ by $\mu^{(T)}$. As shown in Figure~\ref{fig:locoada_pred}, the prediction performance of $\mu^{(T)}$ has substantial improvement over the uniform-sampling-based minipatch ensembles $\mu^{(1)}$.

\begin{algorithm}[!htbp]
\caption{AdaMP: Feature Importance Module with Optional Inference}
\label{algo:adamp_FI}
		\noindent{\textbf{Input}}: Training pairs $\bZ=(\bX,\bY)$ with $M$ features and sample size $N$, trained minipatch predictors $\{\mu_k(\cdot)\}_{k=1}^K$ and associated minipatch indices $\{(I_k,F_k)\}_{k=1}^K$, indicator variable \verb|Inference|.
            \begin{enumerate}
            
            \item Obtain LOO predictions,  LOCO-LOO predictions, and LOO feature occlusions for each feature $j$:
			\begin{enumerate}
				\item Obtain the ensembled LOO prediction:\\ $\mu_{-i}(X_i;\bZ) = \frac{1}{\sum_{k=1}^{K} \mathbb{I}(i \notin I_k)} \sum_{k=1}^{K} \mathbb{I}(i \notin I_{k}) \mu_{k} (X_{i};\bZ)$;
				
				\item Obtain the ensembled LOCO-LOO prediction:\\ $\mu_{-i}^{-j}(X_i;\bZ) = \frac{1}{\sum_{k=1}^{K} \mathbb{I}(i \notin I_k) \mathbb{I}(j \notin F_k)} \sum_{k=1}^K \mathbb{I}(i\notin I_k)\mathbb{I}(j\notin F_k) \mu_{k}(X_i;\bZ)$;
			 	\item Calculate LOO feature occlusions: \\
                $\hat{\Delta}_j(Z_{i}) = \ell(Z_i,\mu_{-i}^{-j}(\cdot;\bZ)) - \ell(Z_i,\mu_{-i}(\cdot;\bZ))$.
            \end{enumerate} 

            \item Obtain the sample mean  $\bar{\Delta}_j=\frac{1}{N}\sum_{i=1}^N\hat{\Delta}_j(Z_i)$.
            \item If \verb|Inference| = TRUE, compute a $1-\alpha$ interval for each $\Delta_j$: \\
            $\hat{\mathbb{C}}_j=
			\left[\bar{\Delta}_j - \frac{z_{\alpha/2}\hat{\sigma}_j}{\sqrt{N}},\bar{\Delta}_j +\frac{z_{\alpha/2}\hat{\sigma}_j}{\sqrt{N}}\right]$, with $\hat{\sigma}_j= \sqrt{\frac{\sum_{i=1}^N(\hat{\Delta}_j(Z_i)-\bar{\Delta}_j)^2}{N-1}}$ being the sample standard deviation, and $z_{\alpha/2}$ being the upper $\frac{\alpha}{2}$ standard normal quantile.
		\end{enumerate}
		\textbf{Output}: $(\bar{\Delta}_1,\dots,\bar{\Delta}_M)$ and if \verb|Inference| = TRUE, $(\hat{\mathbb{C}}_1, \dots, \hat{\mathbb{C}}_M)$.
\end{algorithm}

\paragraph{LOCO inference target for the adaptive minipatch ensemble.} Here, we aim to interpret the trained adaptive minipatch ensemble via LOCO feature importance inference. That is, we aim to understand each feature's contribution to $\mu^{(T)}(\cdot)$ by constructing confidence intervals for $\Delta_j^{(T)}$:
\begin{equation}\label{eq:loco_adamp_target}
    \Delta_j^{(T)} := \bbE\left[\ell(Z^*,\mu_{\no j}^{(T)}(\cdot;\bZ)) - \ell(Z^*,\mu^{(T)}(\cdot;\bZ))\,\middle|\,\bZ,\cA_{T}\right].
\end{equation}
where $\mu_{\no j}^{(T)}$ is a hypothetical reduced model if we've trained the minipatch ensembles without seeing feature $j$, adopting uniform observational subsampling with size $n$ and feature subsampling with probability vector $\bq^{(T)}$. The error function $\ell$ can be any appropriate error metric specified by practitioners, and we focus on the squared error loss for theoretical development, as it facilitates closed-form derivations of the error differences. The expectation in \eqref{eq:loco_adamp_target} is taken over the random test data $Z^*$, conditioning on the training data $\bZ$ and the randomness in the adaptive minipatch ensemble training, denoted by $\cA_T$. Therefore, $\Delta_j^{(T)}$ probes into the importance of feature $j$ by contrasting the trained adaptive minipatch ensemble $\mu^{(T)}$ with its reduced version $\mu_{\no j}^{(T)}$ that does not require feature $j$ as input. A more extensive discussion on the interpretation of $\Delta_j^{(T)}$ is included in Section~\ref{sec:disc_target}.

\paragraph{LOCO-AdaMP.} Similar to LOCO-MP, we propose to conduct LOCO inference for $\mu^{(T)}$ through a normal approximation for the LOCO-LOO scores in the final iteration. The full algorithm that includes both training and inference is summarized as Algorithm~\ref{algo:loco_adaptive}, consisting of the AdaMP training module and feature importance module with \verb|Inference| set as TRUE. An illustration figure for LOCO-AdaMP is included in Figure~\ref{fig:locoadamp}.

\begin{algorithm}[htbp]
\caption{Update Sampling Probabilities Using Feature Importance}
\label{algo:samp_prob}
\noindent\textbf{Input}: Feature importance scores $\{\bar{\Delta}_1, \dots, \bar{\Delta}_M\}$; minipatch feature size $m$; probability upper bound $\delta$; constant buffer $c_0>0$.

\begin{enumerate}
    \item Compute non-negative scores: $\tilde{\Delta}_j = \bar{\Delta}_j - \min_\ell \bar{\Delta}_\ell + \frac{c_0}{M}$ for all $j$.
        \item Define $\omega_j(t) = \tilde{\Delta}_j \wedge t$ as the thresholded weight of feature $j$, a function of $t>0$. Find the desired threshold
        \begin{equation}\label{eq:threshold}
            t^*: = \sup\left\{0<t\leq \max_l\tilde{\Delta}_l:\frac{m\max_l\omega_l(t)}{\sum_{l=1}^M\omega_l(t)}\leq \delta\right\}
        \end{equation} 
        using Algorithm \ref{algo:find_threshold}.
        \item Compute updated feature sampling probabilities $q_j = \frac{m \omega_j(t^*)}{\sum_{k=1}^M \omega_k(t^*)}$ for $j = 1, \dots, M$.
    \end{enumerate}
\noindent\textbf{Output}: Feature sampling probability vector $\boldsymbol{q} = (q_1, \dots, q_M)$.
\end{algorithm}

Therefore, LOCO-AdaMP is a one-step add-on to the AdaMP training procedure, an almost-free built-in LOCO inference for AdaMP. Compared to LOCO-Split, LOCO-AdaMP does not require data splitting or post-training model refitting. Although we need to be constrained within the minipatch framework like LOCO-MP, as an ensemble framework, the base learner can be arbitrarily chosen. Furthermore, the adaptive feature subsampling also significantly improves the predictive performance of the trained minipatch ensemble, compared to LOCO-MP, which we will extensively validate in Sections~\ref{sec:sims} and \ref{sec:realdata}.

\begin{algorithm}[htbp]
\caption{Adaptive Minipatch LOCO Inference}
\label{algo:loco_adaptive}
		\noindent{\textbf{Input}}: Training pairs $\bZ=(\bX,\bY)$ with $M$ features and sample size $N$, minipatch sizes $n$, $m$; number of training iterations $T$; number of minipatches within each iteration $(K_1,\dots,K_T)$; base learner $H$; the highest sampling probability $\delta\in (m/M,1)$; confidence level $1-\alpha$; feature set $\cF$ to perform LOCO inference for.

            \begin{enumerate}
            \item Initialize feature sampling probability vector $\bq^{(1)}=(\frac{m}{M},\dots,\frac{m}{M})^\top$.
            \item For $t = 1,\dots, T$,
            \begin{enumerate}  
                \item Apply Algorithm~\ref{algo:adamp_training} to $\bZ$ with observational subsample size $n$, $\bq^{(t)}$, number of minipatches $K_t$, and base learner $H$ to obtain $\{\mu_k^{(t)}\}_{k=1}^{K_t}$ and $\{(I_k,F_k)\}_{k=1}^{K_t}$.
                \item Apply Algorithm~\ref{algo:adamp_FI} to $\bZ$, $\{\mu_k^{(t)}\}_{k=1}^{K_t}$, and $\{(I_k,F_k)\}_{k=1}^{K_t}$ to obtain $(\bar{\Delta}_1^{(t)},\dots,\bar{\Delta}_M^{(t)})$. If $t<T$, set \verb|Inference| as FALSE; otherwise, set \verb|Inference| as TRUE.
                \item If $t<T$, apply Algorithm \ref{algo:samp_prob} to $(\bar{\Delta}_1^{(t)},\dots,\bar{\Delta}_M^{(t)})$ to obtain the updated probability vector $q^{(t+1)}$.
            \end{enumerate}
            \item For $j\in \cF$, obtain a $1-\alpha$ confidence interval $\hat{\mathbb{C}}_j$ for each $\Delta_j$ from the output of Algorithm \ref{algo:adamp_FI} in the $T$th iteration.
		\end{enumerate}
		\textbf{Output}: $\{\hat{\mathbb{C}}_j: j\in \cF\}$
\end{algorithm}

\subsection{Discussion on the Inferential Target}\label{sec:disc_target}
The inferential target of LOCO-AdaMP, $\Delta_j^{(T)}$, defined in \eqref{eq:loco_adamp_target}, is a measure of machine-learning feature importance that explains the trained model $\mu^{(T)}$, rather than population feature importance that explains the underlying data generating mechanism~\citep{gan2022model}. Similar to other LOCO importance metrics~\citep{lei2018distribution,gan2022model}, $\Delta_j^{(T)}$ compares the predictive performance of the full model $\mu^{(T)}$ with that of a reduced model, $\mu^{(T)}_{\no j}$, obtained by aggregating minipatch predictors trained without feature $j$. The magnitude of $\Delta_j^{(T)}$ reflects how much $\mu^{(T)}$ relies on feature $j$, while its sign indicates whether feature $j$ is helpful or detrimental to the predictive performance of $\mu^{(T)}$.

\paragraph{Comparison to other LOCO inferential targets.} Our inferential target differs from those of LOCO-Split and LOCO-MP because the three methods explain different predictive models: LOCO-Split explains a model trained only on the training split; LOCO-MP explains a minipatch ensemble constructed via uniform subsampling; LOCO-AdaMP explains an adaptive minipatch ensemble predictor. As we will show in Section~\ref{sec:sims}, the model we explain, AdaMP, achieves superior predictive performance compared with the predictors explained by LOCO-Split and LOCO-MP. It even outperforms the full model trained on the entire data set without minipatch ensembling, whose feature importance cannot be estimated or tested without additional data. This comparison highlights that AdaMP provides a more competitive model for future deployment and is therefore particularly relevant to explain. Furthermore, as shown in Section~\ref{sec:sims}, LOCO-AdaMP is also more effective in detecting truly important features in the data-generating model, perhaps because AdaMP provides a better approximation to the underlying predictive relationship.

\paragraph{How adaptive sampling shapes the inferential target.} To further understand how adaptive sampling shapes the target of LOCO-AdaMP, we conduct toy simulations and track the LOCO importance of three features across iterations. Results under a linear data-generating model are shown in Figure~\ref{fig:target50_ind}. As we update the sampling probability, the strongest feature (feature 1) has increasing importance as it gets sampled and used more often. Feature 5, the weakest signal feature, also shows increasing importance with ridge and kernel ridge base learners. Its importance remains at a similar level across iterations when the base learner is a decision tree, perhaps because trees are less effective at capturing weak signals in linear models. The noise feature always has a low importance from the first iteration onward, revealing a self-reinforcing pattern in which low importance leads to a low sampling probability and then subsequent low importance in the following iterations.
\begin{figure}[htbp]
    \centering
    \includegraphics[width=\linewidth]{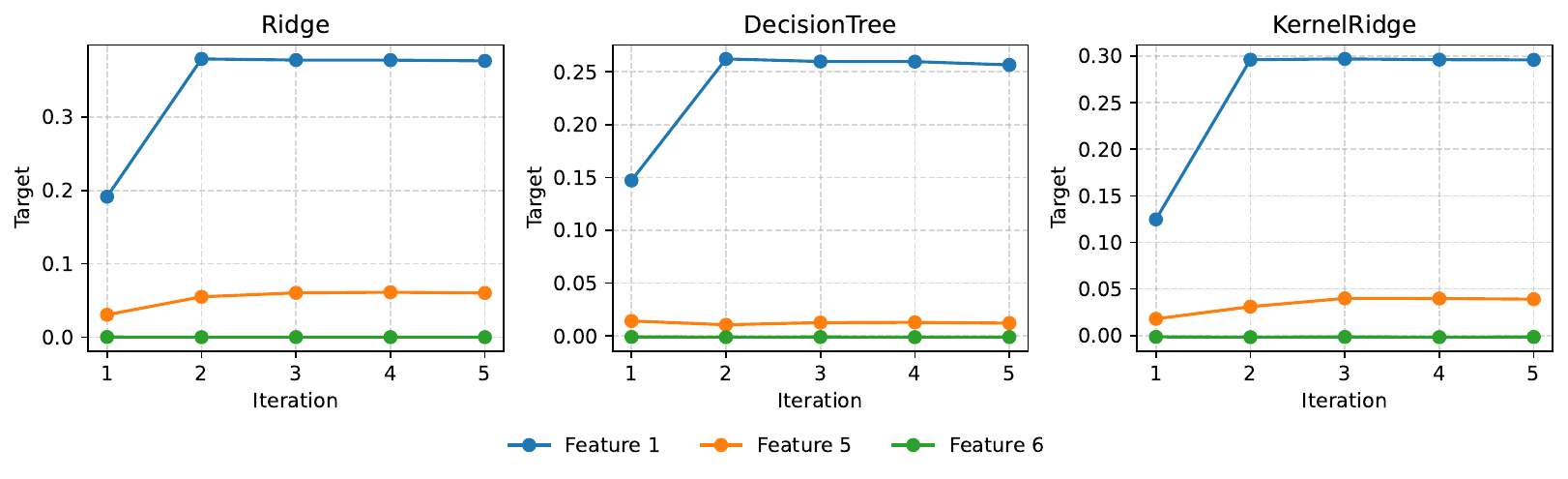}
    \caption{\small Inference targets $\Delta_j^{(t)}$ of LOCO-AdaMP over iterations $t$ in the low-dimensional setting ($M = 50$) for default linear model setup. Feature 1 and feature 5 are signal features with SNR$=2.5$ and SNR$=1$ respectively, while feature 6 is a noise feature. Results are averaged over $100$ simulation replicates. We use the same simulation setup as in Section~\ref{sec:sim_pred}. Figures with y-axis on a log-scale and results under correlated linear model, nonlinear model and high-dimensional settings are included in the Appendix~\ref{add_inf_tar}.}
    \label{fig:target50_ind}
\end{figure}
\paragraph{The role of the reduced model.} The LOCO metric requires a reduced model as a baseline for assessing how predictions would change when feature \(j\) is excluded. As discussed by \cite{Covert2020ExplainingBR}, when quantifying the contribution of feature \(j\) to a given machine learning model, the reduced model can be constructed in many ways. For example, one may replace feature \(j\) with a fixed value or a knockoff counterpart at prediction time, permute its values, or retrain the model without it. LOCO methods take the last approach, not requiring knowing the joint feature distribution. In LOCO-AdaMP, our reduced model is \(\mu_{\no j}^{(T)}\), which aggregates the trained minipatch predictors that exclude feature \(j\) while preserving the adaptive sampling path $Q^{(T)}$ used to construct the full model \(\mu^{(T)}\). Alternatively, one could define \(\widetilde{\mu}_{\no j}^{(T)}\) by rerunning the entire adaptive minipatch training algorithm on \(\bZ_{\no j}\). Unlike our reduced model, \(\widetilde{\mu}_{\no j}^{(T)}\) relearns the feature-sampling probabilities after feature \(j\) is removed and therefore follows a different adaptive sampling path. For the purpose of explaining the fitted full model \(\mu^{(T)}\), we argue that \(\mu_{\no j}^{(T)}\) provides a more appropriate baseline, as it yields a more local comparison and more directly isolates the effect of excluding feature \(j\) from the fitted ensemble \(\mu^{(T)}\).

\section{Coverage Guarantees}\label{sec:theory}
In this section, we establish asymptotic validity of the LOCO-AdaMP confidence interval. 
We note that our validity theory does not follow directly from the theory for LOCO-MP \citep{gan2022model}, although the final inferential steps look very similar. The key challenge here is brought by the adaptive nature of Algorithm~\ref{algo:loco_adaptive}, which iteratively updates the feature sampling probabilities based on previously estimated feature importance scores from all samples. Consequently, the LOO and LOCO-LOO predictors in the final iteration also depend on the left-out observation through the adaptive feature sampling distribution. Therefore, we develop new proof techniques to address the involved dependency not appearing before in prior work. 

\subsection{Theoretical Statement}
We start by defining necessary notations used in our coverage theory. For training data $\bZ=(\bX,\bY)$ and new data point $Z^* = (X^*, Y^*)$, we define $\mu_{I,F}(X^*;\bZ)$ as the predictor trained on the minipatch $(I, F)$ given base learner $H$, i.e.,  $\mu_{I,F}(X^*;\bZ) = H(\bZ_{I,F})(X^*_{F})$. Given a feature sampling distribution function $Q:2^{[M]}\rightarrow[0,1]$, let $Q_{\no j}$ be the conditional distribution for $F\sim Q$ given that $F\not\owns j$.
We define $\mu^*(X^*;Q,\bZ)$ as the expectation taken over random subsample of minipatches, i.e.,  $\mu^*(X^*;Q,\bZ)=\sum_{I\subset[N], F\subset[M]}\mu_{I,F}(X^*;\bZ)\frac{1}{{N \choose n}}\ind(|I|=n)Q(F)$. Let $h_j(Z^*, \bZ, Q)$ denote the LOCO importance of feature $j$ associated with this expected ensemble predictor:
\begin{equation}\label{eq:h_j_def_main}
    h_j(Z^*, \bZ, Q)=\ell(Z^*,\mu^*(\cdot;Q_{\no j},\bZ)) - \ell(Z^*,\mu^*(\cdot;Q,\bZ)).
\end{equation}
For technical purpose, we define $Q^{*(T)}$ as the final-iteration feature sampling probability mass function, updated using the LOCO importance computed from combinatorial averages of minipatch predictors, i.e., let $K_t\rightarrow \infty$ (detailed definition presented in Algorithm~\ref{algo:loco_adaptive_combinatorial}). Then, let $h_j(Z^*,\bZ)=h_j(Z^*, \bZ, Q^{*(T)}(\bZ))$, $h_j(Z^*)=\bbE_{\bZ}[h_j(Z^*, \bZ)\mid Z^*]$, and we denote the variance $\text{Var}(h_j(Z^*))$ by $\sigma_j^2$.

\begin{assumption} [Bounded predictions and responses]
\label{ass:adp_bounded}
     Suppose that the predictions of base minipatch predictors and the response variable are bounded: $\forall I\subset[N]$ with size $n$, $F\subset [M]$,
    \begin{align*}
    |\mu_{I,F}(X^*;\bZ)|\le C,\quad |Y|\leq C,
    \end{align*}
    where $\mu_{I,F}(\cdot;\bZ)$ is the trained base model on minipatch $(I,F)$ as defined in Algorithm~\ref{algo:adamp_training}.
\end{assumption}

\begin{assumption}[Base model stability]\label{ass:stb}
    There exists a function $\stb(n):\bbR\rightarrow \bbR^{+}$, such that, for any training set $\bZ$ consisting of $n-1$ i.i.d. samples from $\cP$ and any feature subset $F\subset [M]$ satisfying $|F|\leq 4\max\{m,\log(n)\}$, 
    \begin{equation*}
        \begin{split}
            \left\|H([\bZ_F;\tilde{Z}_F])(X^*_F) - H([\bZ_F;\tilde{Z}^{'}_F])(X^*_F)\right\|_{\psi_2|\bZ,X^*}\leq {\rm stb}(n),\\
            \bbE\big[H([\bZ_F;\tilde{Z}_F])(X^*_F) - H([\bZ_F;\tilde{Z}^{'}_F])(X^*_F)\big]^2\leq {\rm stb}^2(n),
        \end{split}
    \end{equation*}
    where $\bZ_F = (\bX_{:,F},\bY)$, $\tilde{Z},\,\tilde{Z}^{'}$ are i.i.d. copies from $\cP$. $\|\cdot\|_{\psi_2|\bZ,X^*}$ in the first inequality denotes the $\psi_2$-norm of the given random variable conditioning on $\bZ,\,X^*$.
\end{assumption}
Assumption~\ref{ass:stb} is a stability assumption on the base learner, requiring that the random prediction change due to replacing one single training sample is sub-Gaussian-$\stb(n)$ as long as the feature dimension remains controlled. Note that Assumption~\ref{ass:adp_bounded} implies Assumption~\ref{ass:stb} with $\stb(n)=O(1)$. While for certain base learners such as the ridge and kernel ridge regression with bounded kernel diagonals and lower bounded regularization parameter, as well as locally strongly convex empirical risk minimization, Assumption~\ref{ass:stb} is satisfied with $\stb(n)\asymp \frac{1}{n}$ (details presented in Appendix~\ref{sec:stb_examples}). We will see in Assumption~\ref{ass:adp_prob_stb} that a more stable base learner allows greater flexibility in the choice of the minipatch sample size $n$.
\begin{assumption}[Minipatch size and stability]
\label{ass:adp_prob_stb}
     $\frac{n}{N}, \frac{\delta}{1-e^{-m}}, \frac{m}{\delta M} \le c_1$ for some constant $0 < c_1 < 1$, $\delta\geq c_2$ for some $0<c_2<1$. In addition, $$N\gg \frac{m^{8(T-1)}}{\sigma_j^2}\left(n^2\stb^2(n)\left(\log M +\log N\right)+1\right).$$
\end{assumption}
Assumption~\ref{ass:adp_prob_stb} ensures that the ensembled minipatch predictor is sufficiently stable. Similar to \cite{gan2022model}, when the base model is more stable (smaller $\stb(n)$), our condition on $n$ is milder. For instance, if $\stb(n)\asymp \frac{1}{n}$ (satisfied by the examples in Appendix~\ref{sec:stb_examples}), then $n$ can be arbitrary and Assumption~\ref{ass:adp_prob_stb} is implied by $N\gg \frac{m^{8(T-1)}}{\sigma_j^2}(\log M+\log N)$. On the other hand, if $\stb(n)\leq C$ (immediately implied by Assumption~\ref{ass:adp_bounded}), minipatch size $n$ needs to satisfy $n\ll \frac{\sigma_j}{m^{4(T-1)}}\sqrt{\frac{N}{\log M+\log N}}$.

Compared to the small-$n$ assumption for LOCO-MP \citep{gan2022model}, our assumption involves an explicit dependence on the minibatch size $m$, which arises from the iterative minipatch-based update procedure with $T$ iterations. This dependence disappears when $T=1$. In practice, $T$ can be kept small (e.g., $T=5$ in Figure~\ref{fig:locoada_pred}), and despite this additional structure, the condition remains mild in high dimensions if we choose $m\asymp \log M$, leading to a polylog dependency of sample size $N$ on the dimension $M$.

\begin{assumption}[Number of minipatches]
\label{ass:adp_n_minipatches}
     The number of random minipatches $K_t$ in each iteration satisfies $\min_{1\leq t\leq T}K_t \gg  m^{8T-9}\big(\frac{N}{\sigma_j^2}+1\big)M (\log M+\log N)$.
\end{assumption}
Assumption~\ref{ass:adp_n_minipatches} requires a sufficiently large number of random minipatches. In practice, due to parallel computing and the memory efficiency of training the base model on each minipatch, setting a large $K_t$ often induces a mild cost. 
Next, we introduce a technical assumption that forbids the least important feature in any iteration be heavily sampled in prior iterations, which helps ensure the stability of the adaptive sampling procedure. A similar assumption was made by \cite{Liu2026model}. Formally, let $j_{\min}^{(t)} = \argmin_{j}\bar{\Delta}_j^{(t)}$, the feature with the lowest estimated LOCO importance in the $t$th iteration; let $j_{\min}^{*(t)}$ be defined similarly, but replacing the ensembled minipatch predictor by its combinatorially averaged counterpart when computing the LOCO importance; also let $j_{\min}^{*(\del S,t)}$ be similarly defined but we apply the AdaMP procedure on $\bZ_{\no S,:}$ for a subset $S\subset[N]$. In addition, let $\bq^{(t)}$, $\bq^{*(t)}$, and $\bq^{*(\del S,t)}$ be the updated sampling probability vectors of these three described procedures.

\begin{assumption}[Least important features were not over-sampled]\label{ass:minDelta_q_bnd}
    $\forall 1\leq t'\leq t \leq T$, $\bq_{j_{\min}^{(t)}}^{(t')},  \bq_{j_{\min}^{*(t)}}^{*(t')} \leq \frac{Cm}{M}$, and $\forall S,\,S'\subset[N]$ with size $0\leq |S|,|S'|\leq 2$,  $\bq_{j_{\min}^{*(\del S,t)}}^{*(\del S',t')}\leq \frac{Cm}{M}$ for a universal constant $C>0$.
\end{assumption}

To understand when Assumption~\ref{ass:minDelta_q_bnd} holds, we conduct simulation studies across linear and non-linear settings and compute the sampling probabilities of least important features (see Figure~\ref{fig:ass500} in Appendix), which partially validates this assumption.

\begin{assumption}[Bounded third-moment of standard feature importance]\label{ass:3moment}
The standardized feature importance function has a bounded third moment:
    \begin{equation*}
        \bbE\left[|h_j(Z^*)-\bbE(h_j(Z^*))^3|/\sigma_j^3\right]\leq C.
    \end{equation*}
\end{assumption}
Assumption~\ref{ass:3moment} is a mild moment condition that guarantees uniform integrability and supports the central limit theorem argument.

\begin{theorem}[Coverage Guarantee]\label{thm:main}
    Under Assumptions~\ref{ass:adp_bounded}-\ref{ass:3moment}, if we run Algorithm~\ref{algo:loco_adaptive} with $Z_1,\dots,Z_N \overset{i.i.d.}\sim \cP$, the squared error loss $\ell(Y,\hat{Y})=(Y-\hat{Y})^2$, and the number of iterations $T\leq MN$, then $\forall j\in \cF$, our confidence interval $\hat{\bbC}_j$ has asymptotically valid coverage for $\Delta_j^{(T)}$ defined in~\eqref{eq:loco_adamp_target}:
    $$
    \lim_{N\rightarrow \infty}\bbP(\Delta_j^{(T)}\in \hat{\bbC}_j)=1-\alpha.
    $$
\end{theorem}

Theorem~\ref{thm:main} shows that the proposed confidence interval $\hat{\mathbb C}_j$ achieves asymptotically valid coverage for $\Delta_j$ despite the nontrivial dependence structure induced by adaptive feature sampling and iterative LOCO evaluation. The coverage guarantee in Theorem~\ref{thm:main} considers the marginal coverage, as the probability is taken over the randomness of $\Delta_j^{(T)}$ itself, similar to the coverage of LOCO-MP. In order to address the challenges brought by the adaptive procedure, the main price we pay is the polynomial dependence on $m$ in Assumptions~\ref{ass:adp_prob_stb} and \ref{ass:adp_n_minipatches}. As discussed earlier, when setting $m\asymp \log M$ and bounded $T$, this extra term scales as ${\rm polylog}(M)$, a mild factor in high-dimensional settings. In practice, small $T$ and $m$ are often sufficient for substantially improved prediction and valid inference.

Notably, the theorem does not require the sample size to exceed the number of features, and therefore applies to high-dimensional settings with $N<M$ that commonly arise in modern machine learning applications.

\subsection{Proof Sketch}\label{sec:sketch}
As noted at the beginning of Section~\ref{sec:theory}, establishing coverage guarantees for LOCO-AdaMP is highly nontrivial because the adaptive sampling distribution depends on the entire data set. Here, we provide a brief outline of the proof, focusing on how we address the dependence induced by adaptive sampling.
One major component of the proof of Theorem~\ref{thm:main} is to show that
\begin{equation}\label{eq:proof_sketch_CLT}
\frac{1}{\sigma_j\sqrt{N}}\sum_{i=1}^N \left(h_j(Z_i,\bZ_{\no i,:}, Q^{(T)})-\bbE[h_j(Z^*,\bZ,Q^{(T)})\mid \bZ]\right)\overset{d.}{\rightarrow} N(0,1),    
\end{equation}
where $h_j(Z_i,\bZ_{\no i,:},Q^{(T)})$, defined as in~\eqref{eq:h_j_def_main}, is the leave-one-out feature importance score evaluated at sample $i$ when the feature sampling distribution is $Q^{(T)}$; $h_j(Z^*,\bZ,Q^{(T)})$ is the feature importance score evaluated at an unseen test data $Z^*$. When $K$ is sufficiently large (Assumption~\ref{ass:adp_n_minipatches}), the LOCO-LOO scores $\hat{\Delta}_j(Z_i)$ in Algorithm~\ref{algo:adamp_FI} and our inferential target $\Delta_j^{(T)}$ are very close to $h_j(Z_i,\bZ_{\no i,:},Q^{(T)})$ and $\bbE[h_j(Z^*,\bZ,Q^{(T)})\mid \bZ]$, respectively, and hence \eqref{eq:proof_sketch_CLT} together with a variance consistency result leads to Theorem~\ref{thm:main}. 

\paragraph{Why the CLT for cross-validation is not sufficient.} At first glance, the interval center $\frac{1}{N}\sum_{i=1}^N h_j(Z_i,\bZ_{\no i,:}, Q^{(T)})$ seems to resemble a cross-validation statistic: the leave-one-out (LOO) predictors $\mu_{-i}$, $\mu_{-i}^{-j}$ are evaluated on the left-out sample $Z_i$ and we take an average over $i\in[N]$. Thus, one natural idea is to invoke the central limit theorem for cross-validation~\citep{bayle2020cross}, similar to the proof in \cite{gan2022model}. However, the LOO predictors in AdaMP depend not only on $\bZ_{\no i,:}$, but also on $Z_i$ through the feature sampling distribution $Q^{(T)}$ iteratively updated for $T$ iterations based on all $N$ samples. Such dependence between $Q^{(T)}$ and each left-out sample breaks the cross-validation structure of L.H.S. of \eqref{eq:proof_sketch_CLT}, making the arguments in \cite{gan2022model} not directly applicable. Furthermore, although \cite{Liu2026model} establish feature-selection guarantees for AdaMP, their analysis replaces \(Q^{(T)}\) with a nearby deterministic sampling distribution independent of \(\bZ\). This approximation may suffice for feature selection, but valid statistical inference would require the approximation error to be \(o(N^{-1/2})\), a much stronger accuracy requirement that is generally difficult to satisfy.
\begin{figure}[!htb]
    \centering
    \begin{subfigure}[t]{0.33\textwidth}
    \includegraphics[width=1\linewidth]{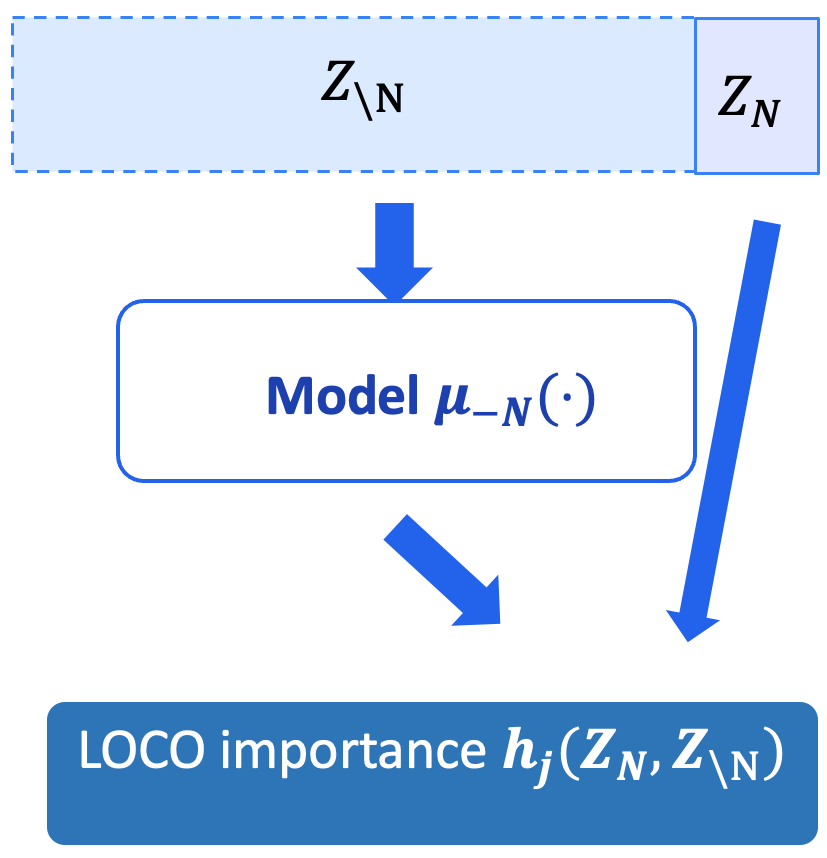}
    \caption{LOCO-MP}
    \end{subfigure}
    \begin{subfigure}[t]{0.4\textwidth}
    \includegraphics[width=1\linewidth]{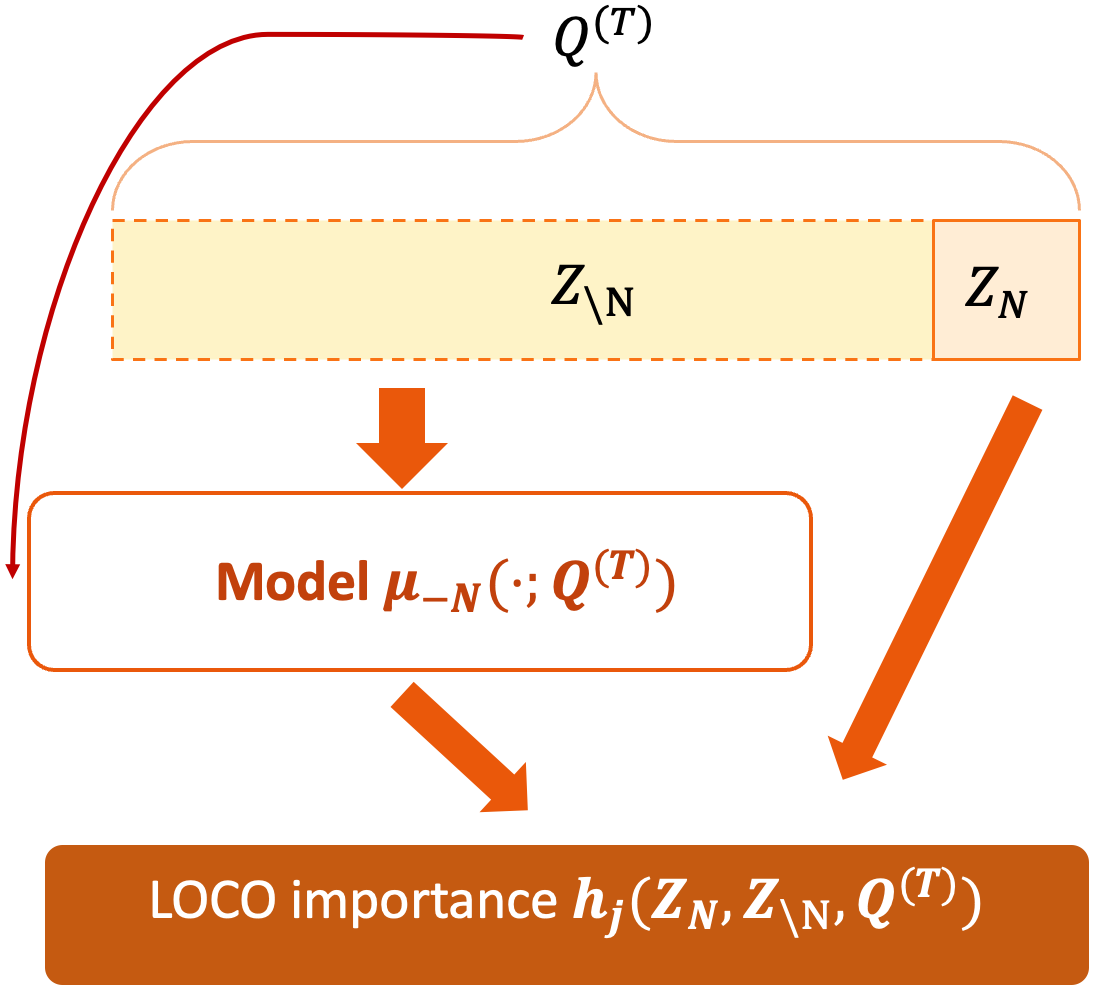}
    \caption{LOCO-AdaMP}
    \end{subfigure}
    \caption{\small Illustrative figure for the LOCO-LOO feature importance scores in LOCO-MP and LOCO-AdaMP. Different from LOCO-MP, the LOO predictor also depends on the held-out sample through the sampling distribution $Q^{(T)}$.}
    \label{fig:sketch1}
\end{figure}

\begin{figure}[!htb]
    \centering
    \includegraphics[width=1\linewidth]{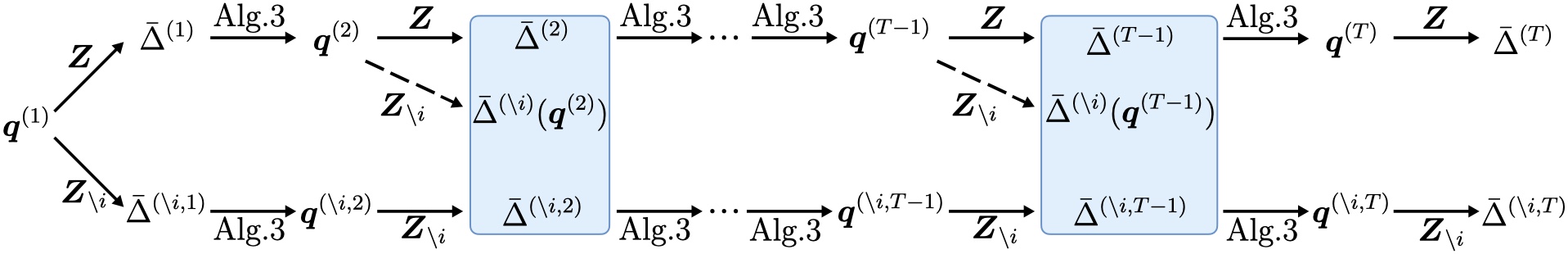}
    \caption{\small Illustration of the iterative updates for $\bq^{(T)}$ and $\bq^{(\no i,T)}$.}
    \label{fig:sketch2}
\end{figure}
\paragraph{Main idea: replacing $Q^{(T)}$ with $Q^{(\no i,T)}$.} To address the dependence between $Q^{(T)}$ and each $Z_i$, we introduce $Q^{(\no i, T)}$, the adaptive feature-sampling distribution obtained by running the same update procedure with sample \(i\) excluded throughout, and denote its sampling-probability vector by $\bq^{(\no i, T)}$. If we can replace $Q^{(T)}$ with $Q^{(\no i, T)}$ in $h_j(Z_i,\bZ_{\no i,:},Q^{(T)})$ with a negligible error, then the CLT for cross-validation and the proof ideas in \cite{gan2022model} are applicable. Therefore, the key here lies in controlling the difference between $\bq^{(T)}$ and $\bq^{(\no i,T)}$. As shown in Figure~\ref{fig:sketch2}, $\bq^{(T)}$ and $\bq^{(\no i,T)}$ are updated iteratively using training data $\bZ$ versus $\bZ_{\no i}$. Let $\bar{\Delta}^{(\no i)}(\bq)$ denote the LOCO importance estimates for minipatch ensembles trained on data $\bZ_{\no i}$ with feature sampling probability $\bq$. Define the distance metric $d(\cdot,\cdot)$ by $d(\bu,\bv)=\|\bu-\bv\|_1+M|\min_k u_k -\min_k v_k|$, where $\bu=(u_1,\dots,u_M)^\top$, $\bv=(v_1,\dots,v_M)^\top$. We establish a smoothness property for Alg.~\ref{algo:samp_prob} and show that
\begin{equation}\label{eq:sketch_q_loo}
\begin{split}
    \|\bq^{(T)}-\bq^{(\no i, T)}\|_1\leq &Cm^3 d(\bar{\Delta}^{(T-1)},\bar{\Delta}^{(\no i,T-1)})\\
    \leq &Cm^3\big(\underbrace{d(\bar{\Delta}^{(\no i)}(\bq^{(T-1)}),\bar{\Delta}^{(\no i)}(\bq^{(\no i, T-1)})}_{\text{controlled by $\bq^{(T-1)}-\bq^{(\no i, T-1)}$}}+\underbrace{d(\bar{\Delta}^{(T-1)}, \bar{\Delta}^{(\no i)}(\bq^{(T-1)}))}_{\text{LOO error in the $(T-1)$th iteration}}\big),
\end{split}
\end{equation}
where the first term captures the discrepancy in the sampling probabilities inherited from the previous iteration, while the second quantifies the effect of leaving out one sample on the LOCO importance estimates when $\bq^{(T-1)}$ is held fixed. We control the first term through a careful induction on $\|\bq^{(t)}-\bq^{(\no i, t)}\|_1$ over $1\leq t\leq T$. The main technical bottleneck lies in bounding the second term, discussed in greater detail below.

\paragraph{Bounding the bagging stability averaged over random $F\subset [M]$.} For the second term in \eqref{eq:sketch_q_loo}, we show that it can be roughly controlled on the scale of
\begin{equation}\label{eq:sketch_LOO}
m\left|\sum_{F\subset[M]}\big[\mu_F(X^*;\bZ)-\mu_F(X^*;\bZ_{\no i,:})\big]\big(Q^{(T-1)}(F)+\max_{j\in [M]}Q^{(T-1)}_j(F)+\max_{j\in [M]}Q^{(T-1)}_{\no j}(F)\big)\right|,    
\end{equation}
where $\mu_F(X^*;\bZ)$ (resp. $\mu_F(X^*;\bZ_{\no i,:})$) denotes the prediction at $X^*$ given by the bagged ensembles trained on $\bZ_{:,F}$ (resp. $\bZ_{\no i,F}$), and $Q^{(T-1)}_j$ (resp. $Q^{(T-1)}_{\no j}$) denotes the conditional distribution of $F\sim Q^{(T-1)}$ given that $F\owns j$ (resp. $F\not\owns j$). Thus, \eqref{eq:sketch_LOO} involves the average stability of the bagged ensembles over random feature subsets drawn from $Q^{(T-1)}, Q^{(T-1)}_j, Q^{(T-1)}_{\no j}$. Although bagging has been shown to enjoy assumption-free stability~\citep{soloff2024bagging}, here we require exponential-tail control to accommodate high-dimensional features. We therefore take a different route and use the sub-Gaussian stability condition in Assumption~\ref{ass:stb} to show that $\mu_F(X^*;\bZ)-\mu_F(X^*;\bZ_{\no i,:})$ is centered sub-Gaussian with parameter $n\stb(n)/N$. This still leaves one major challenge unsolved: this centered sub-Gaussian r.v. is averaged over $F$ drawn from $Q^{(T-1)}$, $Q^{(T-1)}j$, or $Q^{(T-1)}_{\no j}$, all three of which depend on the full sample set when $T>1$. Hence, \eqref{eq:sketch_LOO} is no longer centered or sub-Gaussian. A union bound over $F\subset [M]$ or a covering argument over $Q$ would both introduce an undesirable factor of $C^M$ in the tail probability.

\paragraph{Leave-two-out analysis.} To handle the dependence between $\mu_F(X^;\bZ)-\mu_F(X^;\bZ_{\no i,:})$ and $Q^{(T-1)}(F)$, we make the following observation:
\begin{equation}\label{eq:sketch_fixedF_loo}
    \begin{split}
        \bbE_{F\sim Q^{(T-1)}}\big|\mu_F(X^*;\bZ)-\mu_F(X^*;\bZ_{\no i,:})\big|\leq \frac{n}{N(N-1)}\sum_{i'\neq i}\bbE_{F\sim Q^{(T-1)}}\big|g_F(X^*,\bZ_{\no (i,i')},Z_i,Z_{i'})\big|,
    \end{split}
\end{equation}
where, conditional on $X^*$ and $\bZ_{\no(i,i')}$, $g_F(X^,\bZ_{\no(i,i')},Z_i,Z_{i'})$ is centered and $\stb(n)$-sub-Gaussian when the size of $F$ is controlled. If, for each $i'$ in the summation, we can replace $Q^{(T-1)}$ by its leave-two-out counterpart $Q^{(\no(i,i'),T-1)}$ with negligible error, then we can show that $\bbE_{F\sim Q^{(T-1)}}\big|\mu_F(X^;\bZ)-\mu_F(X^*;\bZ_{\no i,:})\big|$ can be approximated by a sub-Gaussian random variable, up to the leave-two-out approximation error for $Q$. This observation motivates us to establish the stability of $Q^{(t)}$ under the removal of an arbitrary subset $S\subset[N]$ of size bounded by two. Built on top of the leave-two-out analysis and a careful induction argument over $T$ iterations, we achieve the following key lemma: 
\begin{lem}[Difference between $\bq^{(T)}$ and $\bq^{(\no i, T)}$]\label{lem:Q_loo_err_brief}
    Suppose that Assumptions~\ref{ass:adp_bounded}-\ref{ass:minDelta_q_bnd} hold. With probability at least $1-(MN)^{-c}$, we have
    \begin{equation*}
        \begin{split}
            \|\bq^{(T)}-\bq^{(\no i, T)}\|_1\lesssim m^{4(T-1)}\left(\frac{n\stb(n)\sqrt{\log N+\log M}+1}{N} + \sqrt{\frac{M(\log N+\log M)}{m\min_{1\leq t\leq T}K_t}}\right).
        \end{split}
    \end{equation*}
\end{lem}
Lemma~\ref{lem:Q_loo_err_brief} is a direct consequence of Lemmas~\ref{lem:Q_K_err} and \ref{lem:Q_loo_err} in the Appendix. Therefore, when $N$ and $K_t$ are sufficiently large (Assumptions~\ref{ass:adp_prob_stb}-\ref{ass:adp_n_minipatches}), we have $\|\bq^{(t)}-\bq^{(\no i, t)}\|_1=o(\frac{\sigma_j}{\sqrt{N}})$, implying that replacing $Q^{(T)}$ with $Q^{(\no i, T)}$ in \eqref{eq:proof_sketch_CLT} leads to negligible errors ($o_p(1)$) that does not affect the asymptotic normality. Combining Lemma~\ref{lem:Q_loo_err_brief} and the cross-validation CLT as well as the stability of the adaptive minipatch ensemble, we arrive at the asymptotic coverage guarantee in Theorem~\ref{thm:main}.

\section{Simulations}\label{sec:sims}

In this section, we evaluate our LOCO-AdaMP framework through simulation studies. We first validate the coverage of LOCO-AdaMP intervals; we further show that given the same sample budget for training and inference altogether, LOCO-AdaMP yields both stronger predictors and more efficient inference than prior LOCO methods. We consider a low-dimensional ($(M,\,N)=(50,200)$) and a high-dimensional ($(M,\,N)=(500,200)$) setting; for conciseness, we mainly report low-dimensional results in the main paper and defer the high-dimensional coverage, interval, and power results in Section~\ref{add_high}, which are qualitatively similar. The code for reproducing our results is available at \url{https://github.com/yinan-cc/LOCO-AdaMP}.

\paragraph{Simulation framework.} We consider a low-dimensional and a high-dimensional setting with $(M,\,N)=(50,200)$ and $(M,\,N)=(500,200)$, where $M$ and $N$ are feature and sample sizes, respectively. For each dimensionality, we generate samples independently from sparse linear and non-linear models with standard normal observational noise: $Y=f(X)+\varepsilon$, $\varepsilon\sim\mathcal{N}(0,1)$. Let $\bB=[\beta_1,\ldots,\beta_M]^\top$, where $(\beta_1,\dots,\beta_5)^\top=[2.5,2.1,1.7,1.3,1]^\top$ and $\beta_6=\cdots=\beta_M=0$. We set $f(X)=X^\top\bB$ for the linear setting and for the nonlinear setting, we let $f(X)=\sum_{j=1}^5\beta_jg_j(X_j)$ where $g_j(X_j)=\frac{h_j(X_j)}{{\rm sd}(h_j(X_j))}$ are standardized functions with unit variance;
$$(h_1(X_1),\dots,h_5(X_5))=(\max\{X_1,0\},\, \sin(X_2), \ind\{X_3<-1\}X_3,\, \cos(X_4),\, \mathrm{sign}(X_5)).$$ For both linear and nonlinear settings, the default feature joint distribution is standard multivariate normal distribution: $X\sim N(0,I_M)$. Additionally, to examine the effect of feature correlations, we consider a correlated structure under the linear setting where all features still follow independent standard normal distributions aside from one correlated pair: the weakest signal feature $X_5$ and noise feature $X_6$ are jointly normal with $\rho(X_5,\,X_6)=0.9$; we refer to this setting as linear correlated setting. We standardize the response to have zero mean and unit variance in all settings.

\paragraph{LAMPS Implementation Details.} Throughout this section, we set $\ell(\cdot)$ as the squared error loss, same as the setup studied in Section~\ref{sec:theory}. Additional simulation results based on the absolute loss function are provided in Appendix \ref{abs_loss}. We note that the prediction performance of AdaMP often improves quickly after a few iterations, so we set $T=5$. While for the minipatch size $(m,\,n)$, recall that our theory suggests that $m$ should not be too large; $n$ should also be smaller if the base model is not sufficiently stable. Hence, we set $m=\lfloor 4\log M\rfloor$; following the setup in \cite{gan2022model}, we set $n=\lfloor N^{0.8}\rfloor$. In addition, we set the number of minipatches within each iteration as $K_1=\cdots=K_4=2,000, K_5 = 10,000$, the probability upper bound as $\delta = 0.9$ in Algorithm \ref{algo:loco_adaptive}, and the constant buffer as $c_0 = 0.1$ in Algorithm \ref{algo:samp_prob}. We apply the LOCO-AdaMP framework with three base learners: (i) ridge regression with $\lambda = 0.001$, (ii) decision tree regression, and (iii) kernel ridge regression with $\lambda = 0.001$. More details are included in Appendix~\ref{add_sim_set}. Results for the high-dimensional setting are qualitatively similar and can be found in the Appendix.

\subsection{Validation of Coverage}
We validate the coverage of confidence intervals obtained in Algorithm~\ref{algo:loco_adaptive} for a signal feature (feature $5$) and a noise feature (feature $6$) across different base models and data generating processes, when the nominal level $1-\alpha$ is set as $0.9$. For each combination of settings, we set sample size $N\in\{200, 500, 1000, 2000\}$ to examine the asymptotic trend of coverage. Since the inference target~\eqref{eq:loco_adamp_target} involves an expectation and has no clean analytical form, we estimate the expectation by Monte Carlo approximation with $10,000$ independent test samples. We note that using this surrogate for target yields slight underestimation for the theoretical coverage especially for large $N$, since the surrogate has its own uncertainty and fluctuates around the target, hence is more difficult to cover than the true target; see more discussions on this in Appendix~\ref{adj_cover}. Figure~\ref{fig:cover50} presents the average coverage for the surrogate target over $100$ replicates in the low-dimensional setting; the error bars represent 90\% confidence intervals for the coverage. Figure~\ref{fig:cover50} shows that LOCO-AdaMP attains valid coverage most of the time, although the coverage drops below 0.9 in a few settings. To understand whether such coverage drop could be due to the variance associated with the surrogate target, we report adjusted coverage rates in Appendix~\ref{adj_cover} that take into account the surrogate randomness. The adjusted rates achieve nominal coverage in almost all settings
\begin{figure}[!htb]
    \centering
    \includegraphics[width=\linewidth]{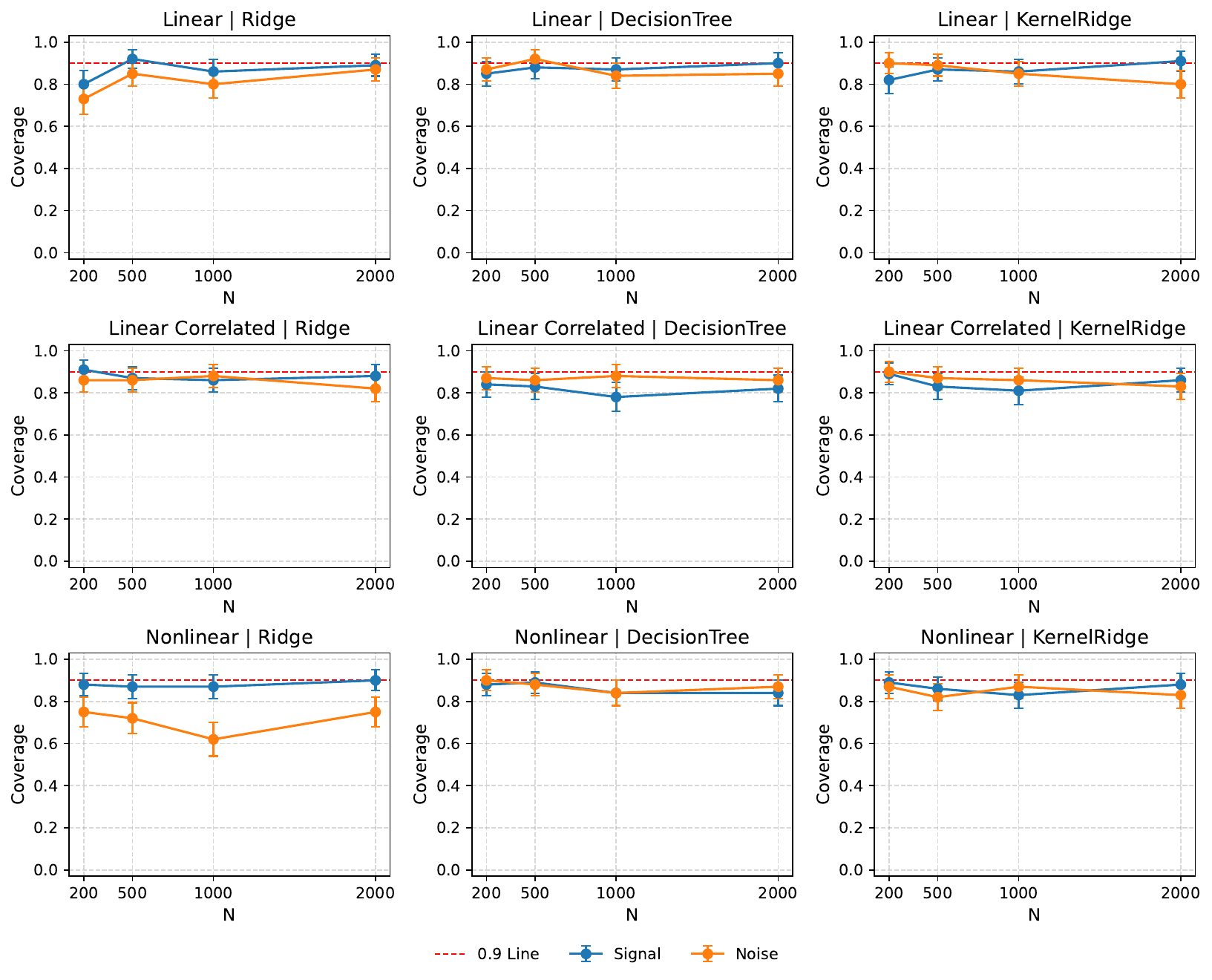}
    \caption{\small Coverage rates of $90\%$ confidence intervals for the targets in the low-dimensional setting $(M=50)$ over $100$ replicates. The first and third rows use the default setting with independent features, whereas the second uses the setting with one correlated feature pair. The blue and orange lines represent a signal feature (feature 5) and a noise feature (feature 6) respectively. The true target value is approximated via Monte Carlo method with $10,000$ test observations; hence the reported coverage rates slightly underestimate the true coverage. To address the underestimation effect due to Monte Carlo randomness, adjusted coverage rates are reported in Figure~\ref{fig:adjcover50}. Results for the high-dimensional setting ($M=500$) are qualitatively similar and can be found in Figures~\ref{fig:cover500} and \ref{fig:adjcover500}.}
    \label{fig:cover50}
\end{figure}

\subsection{Comparative Studies}
Here, we compare the performance of LOCO-AdaMP with LOCO-MP~\citep{gan2022model} and LOCO-Split~\citep{lei2018distribution}. In particular, we ask: given a fixed set of data that can be freely allocated for training and inference, how would these different frameworks perform in terms of training a strong predictor and delivering efficient LOCO inference? For our LOCO-AdaMP and LOCO-MP, all samples are used for training and inference simultaneously; for LOCO-Split, we consider two different allocations: (i) a $50\%$ - $50\%$ split, denoted as LOCO-Split$0.5$; and (ii) a $75\%$ (training) - $25\%$ (inference) split, denoted as LOCO-Split$0.75$. 

\paragraph{Implementation of baselines.} LOCO-MP is implemented without the variance barrier, with minipatch size set the same as LOCO-AdaMP; the number of minipatches is set as $18,000$, same as the total number of minipatches sampled by LOCO-AdaMP across all iterations. LOCO-MP is implemented with the same set of base models and regularization parameters as LOCO-AdaMP.
Note that when the base model is ridge or kernel ridge regression, we fix $\lambda$ as a small value $0.001$ for both LOCO-AdaMP and LOCO-MP due to the implicit regularization property of minipatch ensembles; while for LOCO-Split, we tune $\lambda$ via $5$-fold cross-validation over a grid of 50 logarithmically spaced values between $10^{-3}$ and $10^1$. In addition, when LOCO-AdaMP and LOCO-MP use decision tree as the base model, LOCO-Split uses random forest, which behaves similarly to minipatch ensemble of trees~\citep{gan2022model}.

\subsubsection{Comparison of Prediction Performance}\label{sec:sim_pred}
We compare the prediction errors of the predictive models trained by our LOCO-AdaMP with those by LOCO-MP, and LOCO-Split with 50\% and 75\% training splits. More precisely, the predictors are the intermediate products of the three LOCO-frameworks obtained before LOCO inference enters: LOCO-AdaMP yields the adaptive minipatch ensemble predictor ($\mu^{(T)}$ in Figure~\ref{fig:locoadamp}), LOCO-MP gives the minipatch ensemble built from uniform subsampling ($\mu$ in Figure~\ref{fig:locomp}), while LOCO-Split simply applies whatever base learner on the training split ($\mu$ in Figure~\ref{fig:locosplit}). Additionally, we train the same base learners adopted by LOCO-Split on the full data set, labeled as ``Full Data'', to understand what can be achieved if LOCO inference is not required. We evaluate the predictors on $10,000$ test data points from the same data generating process as the training data.
Figure~\ref{fig:pred} presents the prediction errors for each method and setting, averaged over $100$ replicates. LOCO-AdaMP always yields the lowest prediction error amongst all LOCO methods; it even gives better prediction than the full data model in many settings, even though the full data model spends all data on training and does not come with LOCO inference. The prediction advantage of LOCO-AdaMP is especially pronounced in high-dimensional settings, perhaps due to the implicit feature selection of the adaptive sampling scheme~\citep{Liu2026model}. 
\begin{figure}[!htb]
    \centering
    \includegraphics[width=1\linewidth]{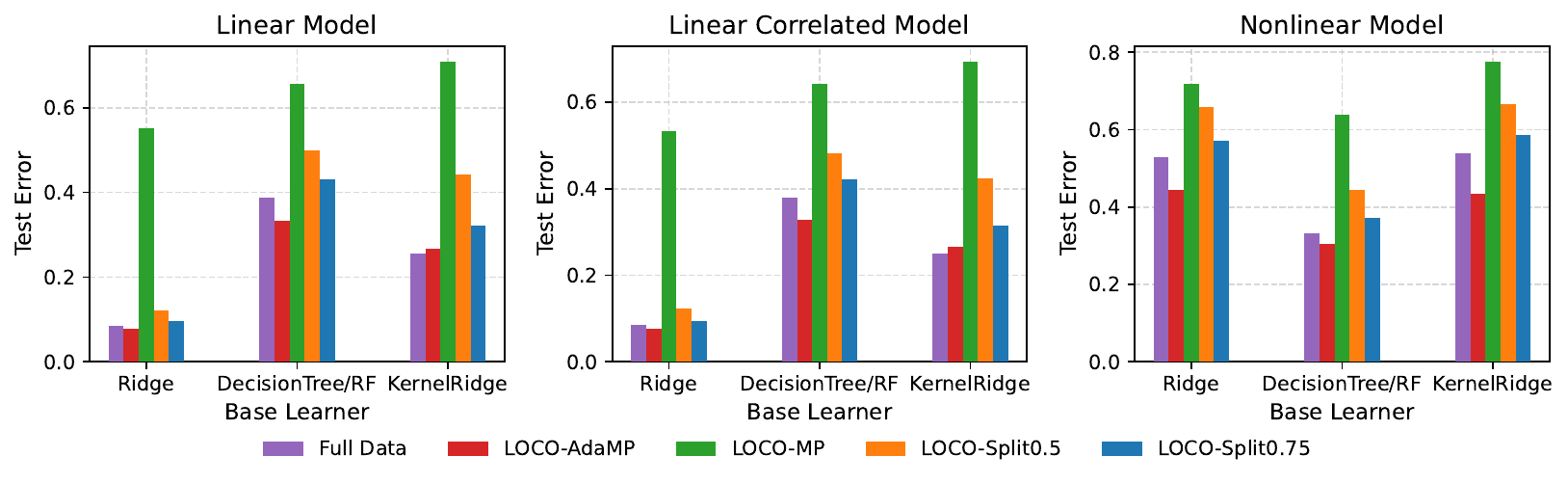}

    \includegraphics[width=1\linewidth]{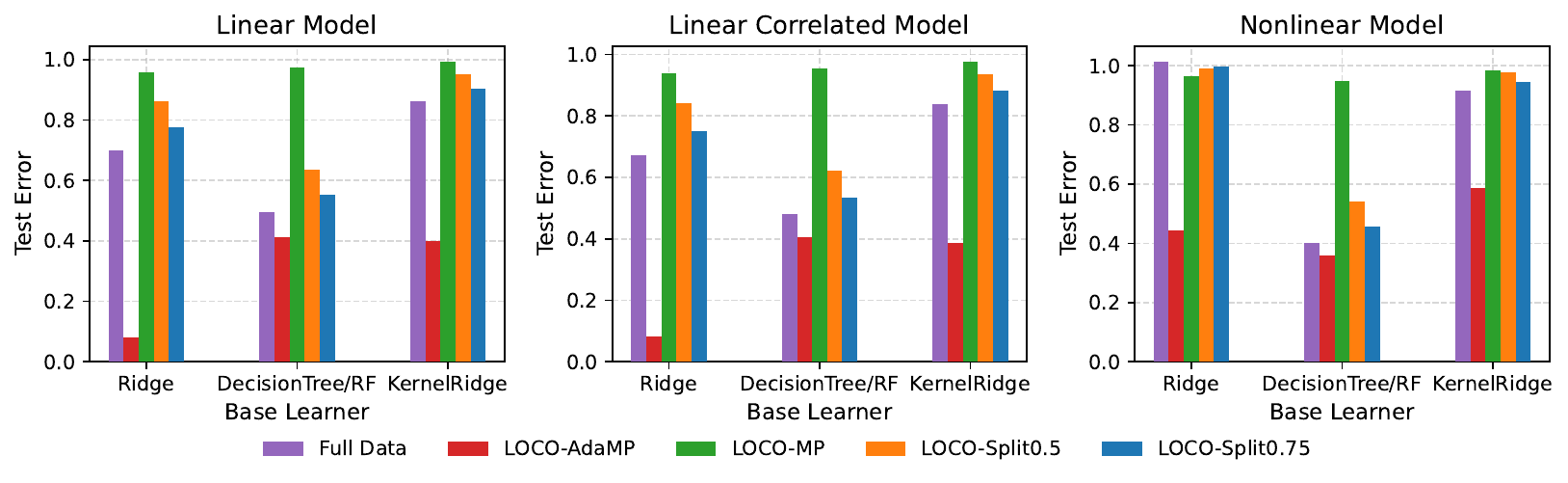}

    \captionsetup{font=normalsize}  
    \caption{\small Prediction error comparison between our LOCO-AdaMP and baseline LOCO methods; the top row is the low-dimensional setting ($M=50, N=200$) and the bottom row is the high-dimensional setting ($M=500, N=200$).
    ``Full Data'' denotes the corresponding base learner trained on the entire data set and is shown as reference for what can be achieved if LOCO inference is not required.  
    The first and third columns correspond to the default independent features, whereas the second considers one correlated feature pair.} 
    \label{fig:pred}
\end{figure}

\subsubsection{Comparison of Inference: Confidence Intervals and Statistical Power for Identifying Population Signal Features}\label{sec:sims_power}
We present the $90\%$ confidence intervals for the first 8 features given by LOCO-AdaMP, LOCO-MP, and LOCO-Split in Figure~\ref{fig:ci50} ($M=50, N=200$); the first 5 are signal features with decreasing signal strength. Results for the high-dimensional setting ($M=500$) are qualitatively similar and can be found in the Appendix. As discussed in Section~\ref{sec:disc_target}, the inferential target in each framework characterizes how much each feature contributes to its own predictive model, and hence the comparison is not straightforward. Nonetheless, we highlight three key observations. First, the interval centers of these LOCO frameworks all have different magnitudes, reflecting the magnitude of their inference targets. LOCO-MP has the smallest targets due to the massive uniform subsampling of features, where each feature only has limited contribution to the final predictor. On the contrary, the LOCO importance target of our LOCO-AdaMP framework has a similar or larger magnitude to LOCO-Split, since signal features are frequently sampled in the final iteration and contribute substantially to the final model; this also aligns the prediction advantage of LOCO-AdaMP. Second, the interval width of LOCO-AdaMP is often narrower compared to LOCO-Split even when their interval centers are similar. For the weakest signal feature $5$, LOCO-AdaMP and LOCO-MP often have positive lower confidence bounds while LOCO-Split intervals often cover zero. Third, when feature 5 is strongly correlated with feature 6, LOCO-AdaMP often identifies feature 5 with a confidence interval above zero, while the intervals from other three methods all cover zero.

\begin{figure}[!htb]
    \centering
    \includegraphics[width=\linewidth]{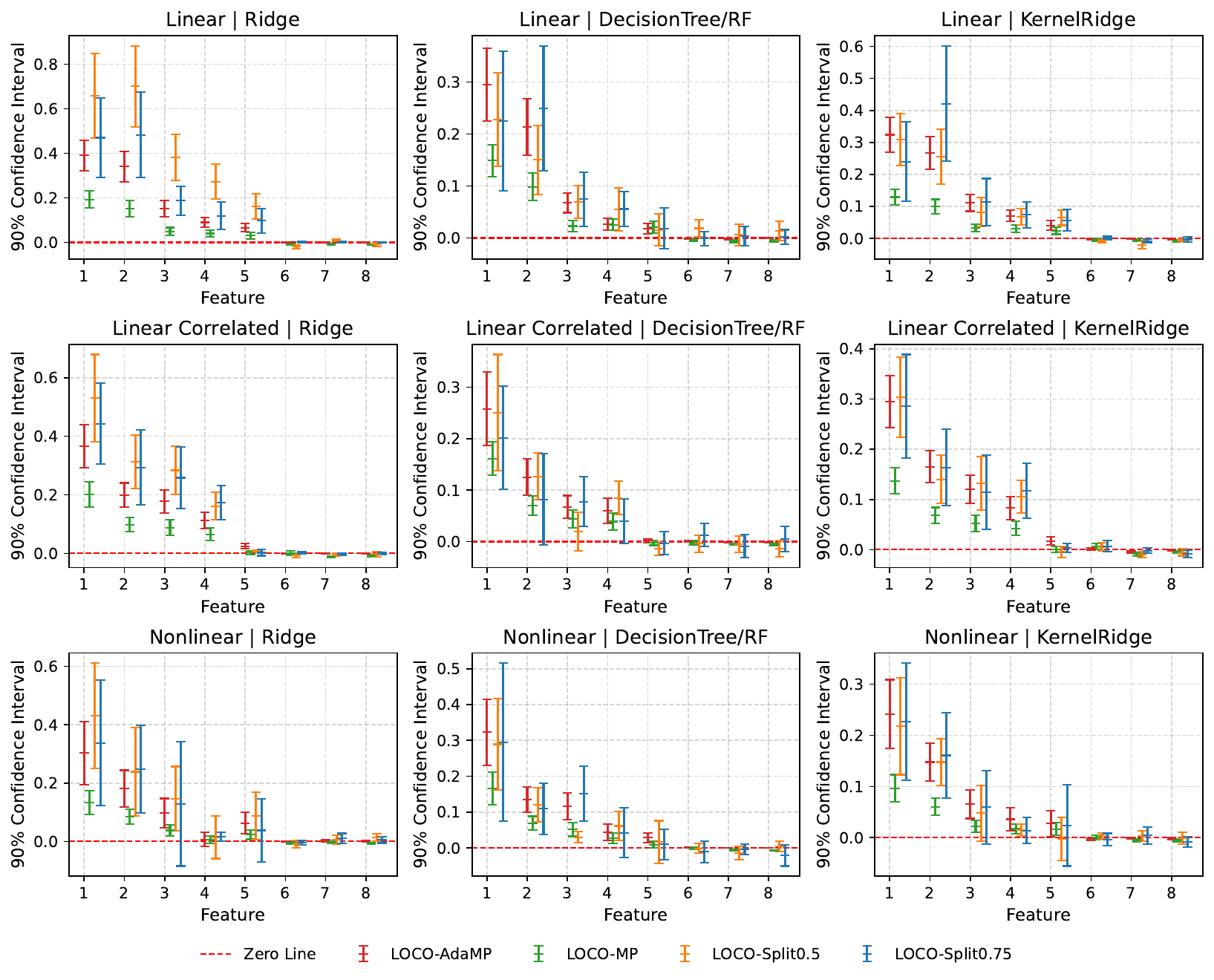}
    \caption{\small 90\% confidence intervals for the first 8 features in the low-dimensional setting ($M=50, N=200$). The first $5$ features are signal features with decreasing strengths. The first and third rows use the default setting with independent features, whereas the second uses the setting with one correlated feature pair. For the second column, decision trees are used for LOCO-MP and LOCO-AdaMP, and random forests are used for LOCO-Split. Results for the high-dimensional setting ($M=500$) are qualitatively similar and are presented in Figure~\ref{fig:ci500}.}
    \label{fig:ci50}
\end{figure}

In addition to confidence intervals, each LOCO framework naturally induces a one-sided test of whether each feature has positive LOCO importance for the corresponding predictive model. As discussed in Section~\ref{sec:disc_target}, however, the methods target different predictive models and therefore test different null hypotheses, so their statistical power is not directly comparable in the usual sense. Nevertheless, because LOCO importance is closely related to population feature importance~\citep{gan2022model,williamson2023general,verdinelli2024feature}, practitioners may also use these tests as proxies for identifying features that are important in the underlying data-generating process. We therefore compare how often each method rejects the null as the signal strength of a given feature varies. We view this as an assessment of signal-identification ability, rather than a formal power comparison under a common null hypothesis. Specifically, we focus on feature $5$ and vary $\beta_5\in\{0, 0.4, 1.0, 2.5, 5.0\}$, while fixing all other coefficients as specified in the beginning of Section~\ref{sec:sims}. The significance level of the test is set as 10\%. We report the rejection rate averaged over $100$ independent replicates in Figure~\ref{fig:power50}; the error bars represent 90\% confidence intervals for the unknown rejection probability. Compared with LOCO-MP and LOCO-Split, LOCO-AdaMP achieves higher rejection rates for weak but nonzero signals, while keeping the rejection rate below the nominal level of $0.1$ when $\beta_5=0$.

\begin{figure}[!htb]
    \centering
    \includegraphics[width=\linewidth]{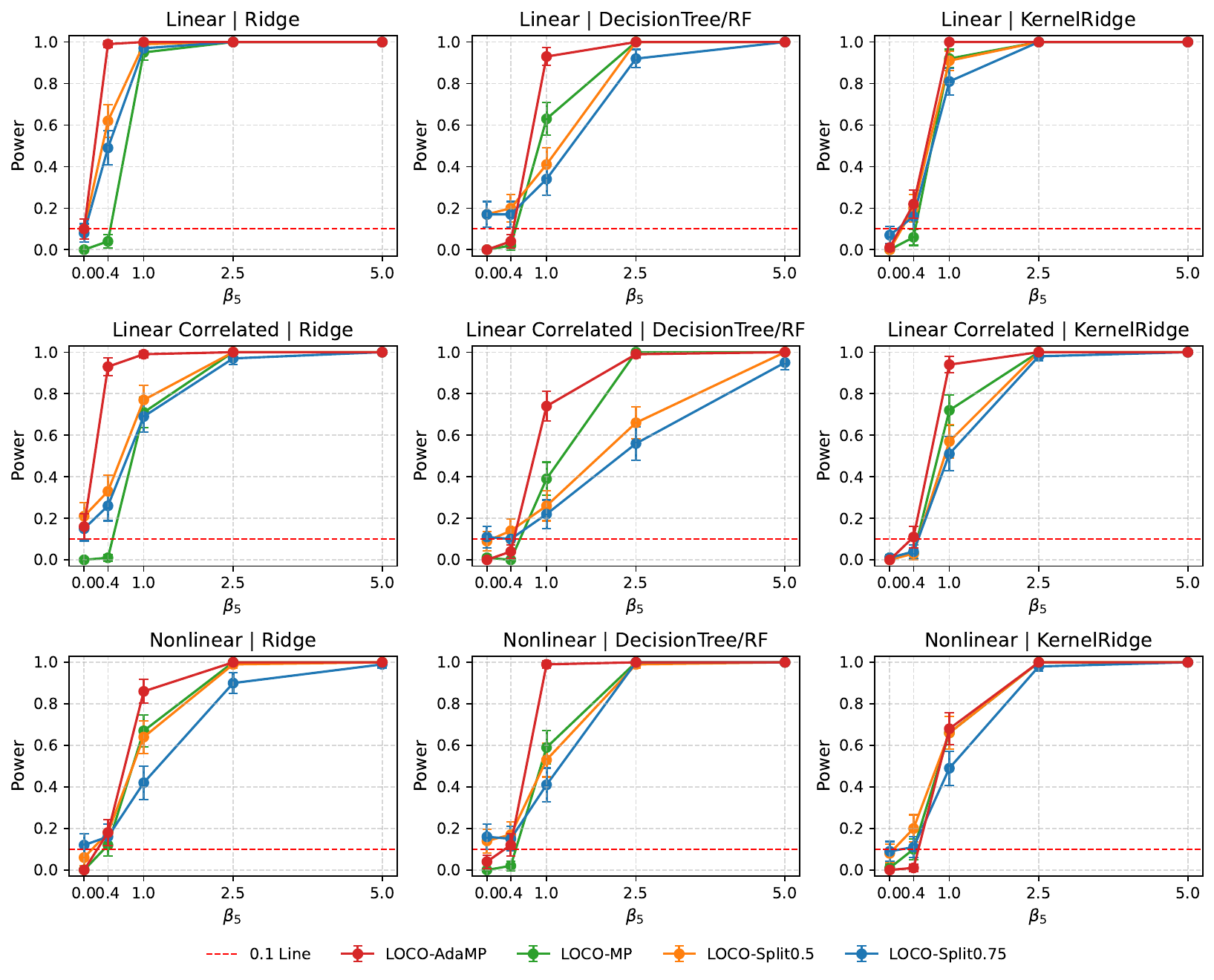}
    \caption{\small Rejection rates of one-sided hypothesis tests ($\alpha = 0.1$) for feature 5 in the low-dimensional setting $(M = 50, N = 200)$, averaged over 100 replicates. The x-axis represents values for $\beta_5$. LOCO-AdaMP quickly attains higher rejection rates when $\beta_5$ increases above zero. The first and third rows use the default setting with independent features, whereas the second uses the setting with one correlated feature pair. Results for the high-dimensional setting ($M=500$) are qualitatively similar and are presented in Figure~\ref{fig:power500}.}
    \label{fig:power50}
\end{figure}

\section{Case Study}\label{sec:realdata}

Alzheimer’s disease (AD) is a progressive neurodegenerative disorder characterized by cognitive decline and substantial heterogeneity across individuals. Understanding the molecular mechanisms underlying cognitive impairment is a central problem in AD research, and transcriptomic data from brain tissue provides valuable insights into gene expression patterns linked to disease progression~\citep{patel2019meta, wang2020deciphering}. In this section, we analyze the RNA sequencing data derived from postmortem brain tissue in the Religious Orders Study and Memory and Aging Project (ROSMAP), a well-established longitudinal cohort designed to investigate aging and neurodegenerative diseases \citep{bennett2018religious}. 
We use the global cognition score as our response variable and adopt the same variance-filtering preprocessing pipeline as in \cite{gan2022model}, resulting in $N=507$ samples and $M=86$ features (genes). 
Our objective is to compare various baseline methods in this regression task in terms of prediction and identifying important genes related to cognition. 

\paragraph{Comparison of prediction.} To ensure stable comparison, we generate 10 random splits for the ROSMAP data with $70\%$ training data and $30\%$ test data. For each split, we standardize the features using the training means and standard deviations and apply the same transformation to the test data. We compare the test errors of the predictors given by LOCO-AdaMP and LOCO baselines (LOCO-MP and LOCO-Split with 50\% and 75\% training splits), trained on each training fold and evaluated on the test fold. We set the base learner as decision tree regression for LOCO-AdaMP and LOCO-MP, and random forest for LOCO-Split, similar to Section~\ref{sec:sims}. The minipatch hyperparameters are set similarly to the simulation studies in Section~\ref{sec:sims}; the minipatch sample size is set as $n=\lfloor N^{0.8}\rfloor$, whereas the feature size $m$ is selected from $\{\lfloor \log M\rfloor,\, \lfloor 2\log M\rfloor,\,\lfloor 4\log M\rfloor,\,\lfloor 8\log M\rfloor\}$ by minimizing the LOO prediction error, tuned separately for LOCO-AdaMP and LOCO-MP and each data split. More implementation details are included in Section~\ref{add_sim_set}. Figure~\ref{fig:rosmap_pred} presents the average prediction errors of these LOCO methods across 10 random splits, where LOCO-AdaMP achieves the lowest test error, with approximately $20\%-30\%$ reduction compared to the baselines.

\paragraph{Comparison of inference: predictive efficacy of significant features.} We further compare the inference results of LOCO-AdaMP with feature importance inference baselines, including the LOCO baselines and four other feature importance inference methods, CPI~\citep{watson2021testing}, VIMP~\citep{williamson2023general}, Floodgate~\citep{zhang2020floodgate}, and GCM~\citep{shah2020hardness}. The latter four methods use the same base learner as LOCO-Split. Their target feature importance notions are different from LOCO importance and are therefore not included in our earlier comparative studies. Nevertheless, in this application, all of these methods provide reasonable approaches for identifying genes associated with subject cognition; we therefore compare their empirical findings for this purpose. We conduct inference with $\alpha = 0.1$, using a Bonferroni correction for multiple testing. In particular, for each training split, we obtain a $p$-value for each gene from each method’s built-in inference procedure. We then rank features from the most significant one to the least based on the p-value, and fit a random forest model with only the top $L$ $(L = 1,\dots,50)$ features on the training split. More implementation details are included in Section~\ref{add_sim_set}. We evaluate the test errors of these trained random forests on the test split and report the average test errors across 10 random splits in Figure~\ref{fig:rosmap_tmse}. LOCO-AdaMP achieves the smallest test error among all methods, using only two features. 
\begin{figure}[!htb]
\centering
\begin{subfigure}[t]{0.4\textwidth}
\centering
\includegraphics[height=5cm]{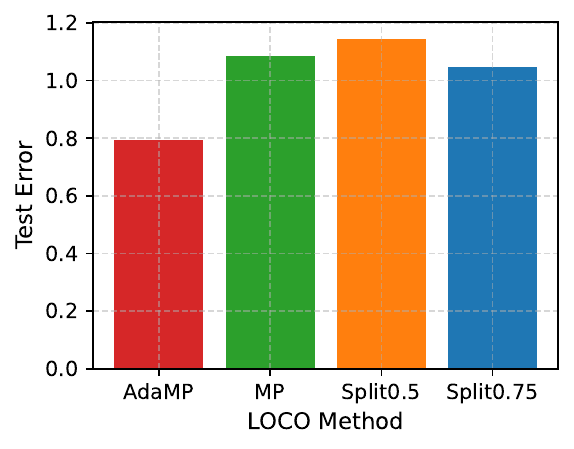}
\caption{Test error for LOCO methods}
\label{fig:rosmap_pred}
\end{subfigure}
\hfill
\begin{subfigure}[t]{0.56\textwidth}
\centering
\includegraphics[height=5cm]{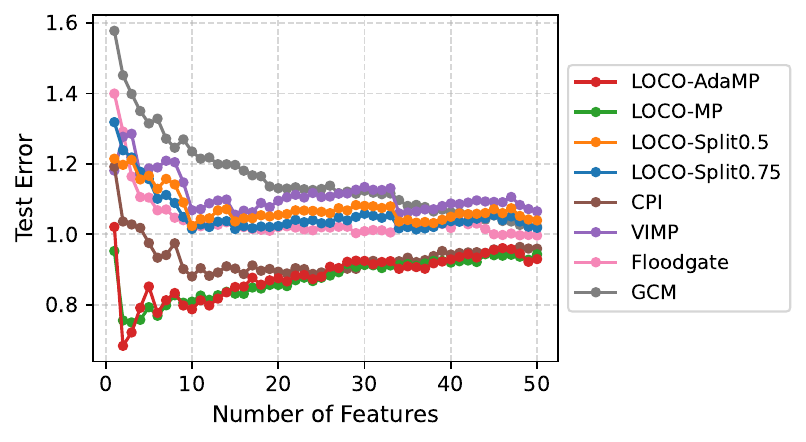}
\caption{Test error for top significant features}
\label{fig:rosmap_tmse}
\end{subfigure}
\caption{\small Comparison of predictive and inferential performance between LOCO-AdaMP and the baseline methods on the ROSMAP data. The left panel compares the predictive performance of the models constructed by different LOCO methods. The right panel compares the predictive usefulness of the top significant features identified by each method by training a separate random forest using the selected features and evaluating it on the test data. LOCO-MP and LOCO-AdaMP use decision trees as their base learners, whereas the other methods use random forests.}
\end{figure} 
\paragraph{Comparison of inference: stability of significant features.} Additionally, we compare the stability of the inference results given by each method across the 10 random splits in Figures~\ref{fig:rosmap_hist0} and \ref{fig:rosmap_hist}. A feature is selected if its corresponding p-value is less than $0.1$, and Figure~\ref{fig:rosmap_hist0} shows the distribution of selection frequencies across all features over the 10 splits. Features with zero selection frequency are plotted in blue to distinguish them from those identified as significant in at least one split. LOCO-Split, CPI, VIMP, Floodgate, and GCM identify most features as nonsignificant, with no feature selected in 5 or more splits. In contrast, LOCO-AdaMP and LOCO-MP consistently select a few features in 9 or 10 splits, while most features are never selected. Furthermore, Figure~\ref{fig:rosmap_hist} highlights the selection frequencies of features selected at least once by LOCO-MP and LOCO-AdaMP. LOCO-AdaMP yields more stable results, with two features (RP11-599B13.6 and AL162497.1) selected across all 10 splits and one additional feature (TTTY14) selected in 1 splits. Notably, the two consistently selected features also yield the lowest test error in Figure~\ref{fig:rosmap_tmse}. These three genes selected by LOCO-AdaMP are also reported as significant in \cite{gan2022model} and are supported by prior AD literature~\citep[see further discussion in][]{gan2022model}.  

\begin{figure}[!htb]
    \centering
    \includegraphics[width=\linewidth]{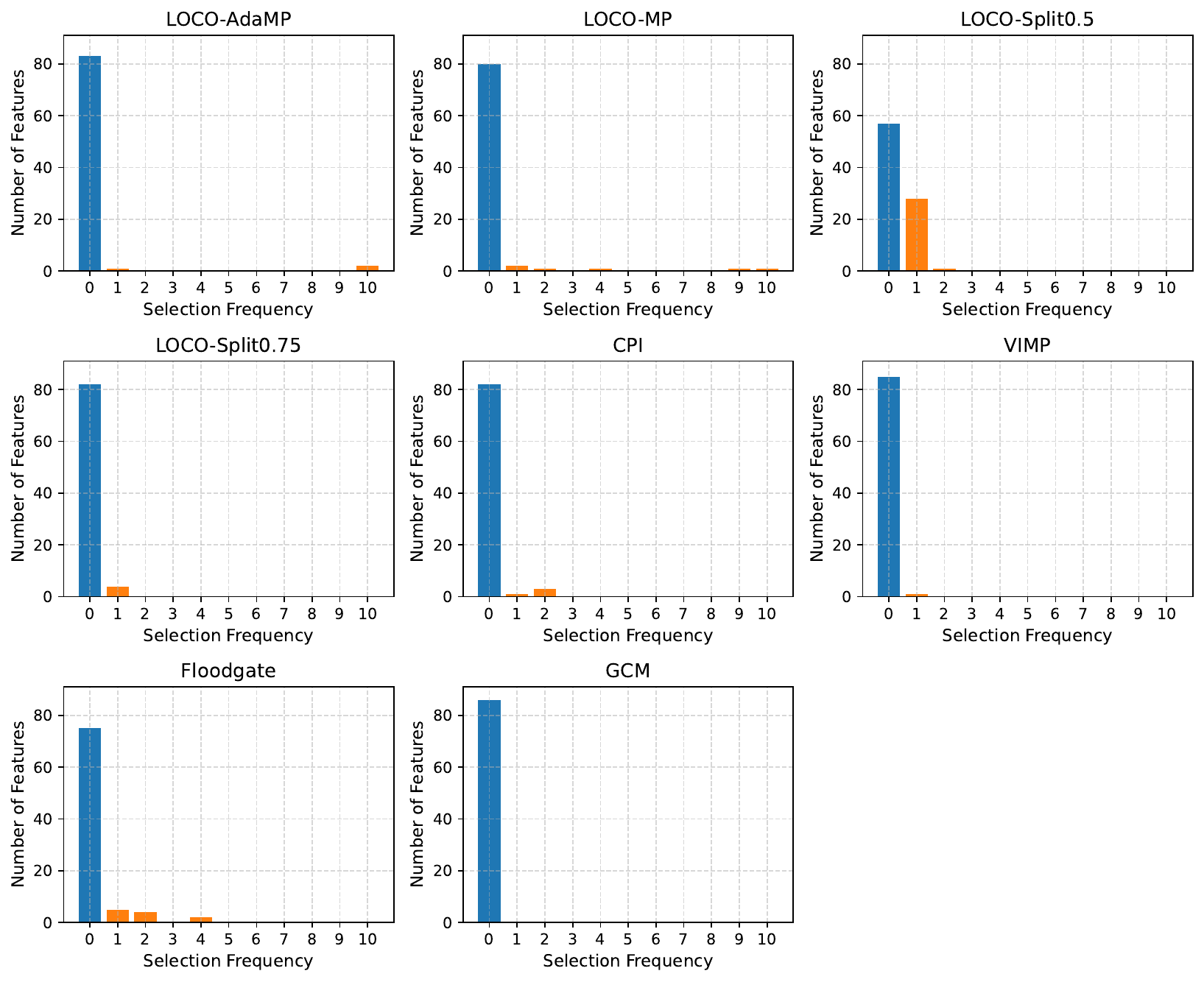}
    \caption{\small Histograms of the feature selection frequencies (selected if declared as significant with level 0.1) across the 10 random splits for each method. X-axis represents the number of times a feature is selected as significant. Y-axis shows the number of features corresponding to each selection count.}
    \label{fig:rosmap_hist0}
\end{figure}

\begin{figure}[!htb]
    \centering
    \includegraphics[width=0.7\linewidth]{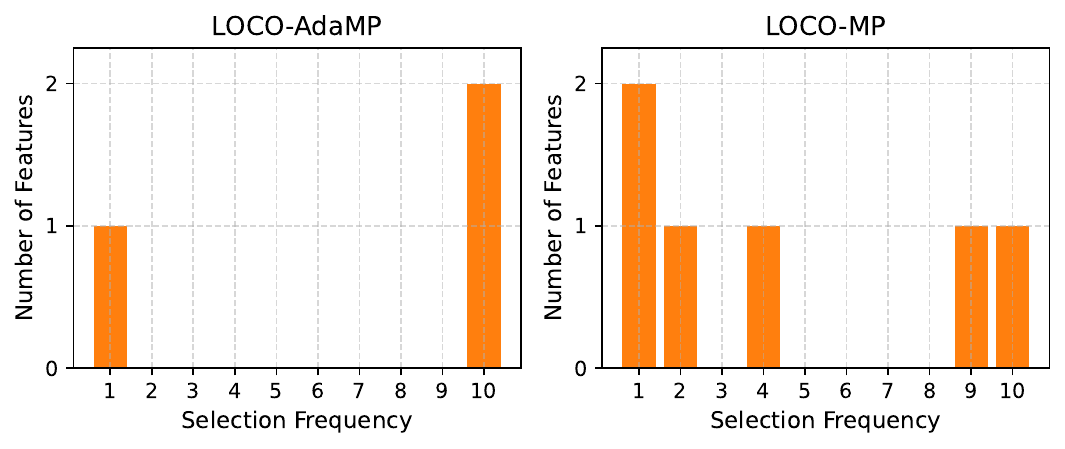}
    \caption{\small Histograms of the feature selection frequencies across the 10 random splits for LOCO-AdaMP and LOCO-MP. Only features selected by at least one split are included.}
    \label{fig:rosmap_hist}
\end{figure}

\section{Discussion}
In this work, we propose LOCO-AdaMP, a flexible minipatch ensemble framework for regression with built-in LOCO feature importance inference. LOCO-AdaMP is almost model-agnostic as it can be deployed with any base regression model; it requires no extra data or computation for inference after model training is complete. Notably, compared to the prior LOCO-MP framework~\citep{gan2022model}, our LOCO-AdaMP utilizes an adaptive feature subsampling scheme to achieve significantly improved prediction performance, while preserving the asymptotic validity of LOCO inference despite the involved dependency between adaptive sampling, minipatch training, and inference, all carried out on the same data set. 
We address the theoretical challenge brought by the non-trivial dependency by developing new leave-two-out analysis and establishing stability properties for the adaptive sampling procedure. Extensive empirical experiments on simulated and real data sets demonstrate that LOCO-AdaMP achieves enhanced predictive performance, more efficient confidence intervals, stronger signal-detection abilities, as well as more stable inference compared to various baselines. These results establish LOCO-AdaMP as a powerful and flexible framework that seamlessly integrates competitive predictive model training and efficient feature importance inference, allowing arbitrary black-box models as the base learner.

Although LOCO-AdaMP achieves efficient inference for individual feature importance, several extensions remain open. First, in high-dimensional applications, practitioners are often interested in simultaneous inference for multiple features or a set of features selected in a data-driven manner. Developing principled multiplicity adjustments and selective inference strategies within the LOCO-AdaMP framework is an important direction for future research. Second, another interesting extension of LOCO-AdaMP is to provide LOCO inference for groups of features formed through correlation or domain knowledge. Third, LOCO-AdaMP can be naturally extended to the classification setting if the squared loss is replaced by cross-entropy loss, although the inference validity theory remains open beyond the squared loss. Lastly, our empirical results demonstrate the superior predictive performance of adaptive minipatch ensembles, while a theoretical characterization of their predictive risk remains an interesting open problem. These developments could further enhance the applicability and our understanding of adaptive minipatch inference in modern machine learning applications.

\bibliographystyle{plainnat}
\bibliography{ref}

\appendix

\section{Additional Algorithmic Details}
Here, we present Algorithm~\ref{algo:find_threshold} for solving \eqref{eq:threshold} in Algorithm~\ref{algo:samp_prob}, i.e., finding the threshold for the shifted importance scores.

\begin{algorithm}[!htbp]
\caption{Solving \eqref{eq:threshold} for the Threshold}
\label{algo:find_threshold}
\noindent\textbf{Input}: Feature weights $\{\tilde{\Delta}_1, \dots, \tilde{\Delta}_M\}$; minipatch feature size parameter $m$; probability upper bound $\delta \in (0, 1)$.
\begin{enumerate}
    \item Sort $\{\tilde{\Delta}_1, \dots, \tilde{\Delta}_M\}$ in decreasing order: $\tilde{\Delta}_{(1)}\geq \tilde{\Delta}_{(2)}\geq \cdots \geq \tilde{\Delta}_{(M)}$.
    \item Find smallest index $0\leq k <\frac{m}{\delta}$ such that:
    \begin{equation}\label{eq:find_threshold}
        \sum_{\ell > k} \tilde{\Delta}_{(\ell)} > \left(\frac{m}{\delta} - k\right)\tilde{\Delta}_{(k+1)}.
    \end{equation}
    \item  Let $t^* := \frac{\sum_{\ell > k} \tilde{\Delta}_{(\ell)}}{m/\delta - k} \wedge \tilde{\Delta}_{(1)}$.
\end{enumerate}
\noindent\textbf{Output}: Threshold $t^*$.
\end{algorithm}

\section{Justification of the Stability Assumption}\label{sec:stb_examples}
We provide several examples of learning algorithms $H$ satisfying the stability condition in Assumption~\ref{ass:stb}. 
The following proposition gives a general sufficient condition for strongly convex empirical risk minimization. Throughout this section, let $d_{n,m}=4\max\{m,\log(n)\}$
$\mathcal F_{n,m}:=\{F\subset[M]:|F|\leq d_{n,m}\}$.

\begin{prop}[Strongly convex ERM stability]
\label{prop:erm_stability}
Suppose the learning algorithm $H$ satisfies $H(\bZ)(X^*)=h(\widehat\theta(\bZ),X^*)$ for some function $h$, where
$$
\widehat\theta(\bZ)\in
\arg\min_{\theta\in\Theta}\left\{
Q(\theta;\bZ):=\frac1n\sum_{i=1}^n\ell(\theta;Z_i)+r(\theta)\right\}.
$$
Suppose that the following assumptions hold.
\begin{enumerate}
\item[(A1)]
For any training data $\bZ$, the objective function $Q$ is $\mu$-strongly convex on $\Theta$ for any training data $\bZ$, i.e., for every $\theta,\theta'\in\Theta$ and
$g\in\partial Q(\theta;\bZ)$,
\[
Q(\theta';\bZ) \ge Q(\theta;\bZ)+\langle g,\theta'-\theta\rangle+\frac{\mu}{2}\|\theta'-\theta\|^2.
\]
\item[(A2)]
For every deterministic sample $\bZ$, the minimum of $Q(\theta;\bZ)$ is attained.
\item[(A3)]
There exists $G>0$ such that, for every pair of training data sets $(\bZ,\bZ')$, $|\ell(\hat\theta(\bZ);Z)
-\ell(\hat\theta(\bZ');Z)|\le G\|\hat\theta(\bZ)-\hat\theta(\bZ')\|$,
almost surely for $Z\sim \cP$.
\item[(A4)]
There exists $L>0$ such that, for every pair $(\theta,\theta')$,
$|h(\theta,X)-h(\theta',X)|\le L\|\theta-\theta'\|$. 
\end{enumerate}
Then, for any training set
$\mathbf Z=(Z_1,\ldots,Z_{n-1})$ consisting of $n-1$ i.i.d.
observations in $\cP$, independent copies $\widetilde Z,\widetilde Z'$ of $Z$,
and an independent test covariate $X^*$,
\[
\left\|
H([\mathbf Z;\widetilde Z])(X^*)
-
H([\mathbf Z;\widetilde Z'])(X^*)
\right\|_{\psi_2|\mathbf Z,X^*}
\le
\frac{2LG}{\mu n\sqrt{\log 2}},
\]
and
\[
\mathbb E
\left(H([\mathbf Z;\widetilde Z])(X^*)-H([\mathbf Z;\widetilde Z'])(X^*)\right)^2\le\frac{4L^2G^2}{\mu^2n^2}.
\]
\end{prop}

\begin{corollary}[Ridge regression]
\label{cor:ridge_stability}
Suppose the base learner $H$ is ridge regression with regularization parameter $\lambda$. Formally,
$$
H(\bZ_F)(X_F^*):=X_F^{*\top}\widehat\beta_F(\bZ_F),
$$
where  $\bZ_F:=(Z_{1,F},\ldots,Z_{n,F})$, $Z_{i,F}:=(X_{i,F},Y_i)$, and
\[
\widehat\beta_F(\bZ_F)= \arg\min_{\beta\in\mathbb R^{|F|}} \left\{\frac1n\sum_{i=1}^n
(Y_i-X_{i,F}^{\top}\beta)^2+\lambda\|\beta\|_2^2\right\}.
\]
If $|X_j|\le R$ for every $j\in[M]$, $|Y|\le B$, then, uniformly over $F\in\mathcal F_{n,m}$, $H$ satisfies Assumption~\ref{ass:stb} with
\[
\stb(n)
=
\frac{2R^2Bd_{n,m}}{\sqrt{\log 2}\lambda n}
\left(
1+R\sqrt{\frac{d_{n,m}}{\lambda}}
\right).
\]
\end{corollary}
\begin{remark}
    Recall that $d_{n,m}=4\max\{m,\log n\}$. When $R,\,B$ are fixed constants (bounded features and responses), the stability parameter for ridge regression $\stb(n)\lesssim \big[1+\big(\frac{m+\log n}{\lambda}\big)^{3/2}\big]/n$. Therefore, when $m$ is bounded and when $\lambda$ is appropriately chosen, $\stb(n)=O(1/n)$.
\end{remark}
\begin{corollary}[Kernel ridge regression]
\label{cor:krr_stability}
Suppose the base learner $H$ is kernel ridge regression with regularization parameter $\lambda$. Formally,
$$
H(\bZ_F)(X_F^*):=\hat{f}_F(X_F^*),
$$
where
\[
\widehat f_F=\arg\min_{f\in\mathcal H_F}\left\{
\frac1n\sum_{i=1}^n(Y_i-f(X_{i,F}))^2+\lambda\|f\|_{\mathcal H_F}^2\right\}.
\]
Suppose uniformly over $F\in\mathcal F_{n,m}$,
\[
K_F(X_F,X_F)\le \kappa^2,\qquad|Y|\le B
\]
almost surely. Then, $H$ satisfies Assumption~\ref{ass:stb} with
\[
\stb(n)=\frac{2\kappa^2B}
{\sqrt{\log2}\lambda n}
\left(
1+\frac{\kappa}{\sqrt\lambda}
\right).
\]
\end{corollary}
\begin{remark}
With bounded response, kernel ridge regression has stability parameter $\stb(n)\lesssim \big[1+\big(\frac{\kappa^2}{\lambda}\big)^{3/2}\big]/n$. For the Gaussian radial basis function (RBF) kernel, the Laplace, Mat\'ern, and rational-quadratic kernels, $\kappa$ is trivially bounded as a function of the kernel parameters.
\end{remark}

\begin{proof}[Proof of Proposition~\ref{prop:erm_stability}]
Fix $\mathbf Z=(Z_1,\ldots,Z_{n-1})$ and a realization $X^*=x^*$.
For two additional observations $z,z'$, let
\[
\bZ_z=[\mathbf Z;z], \qquad \bZ_{z'}=[\mathbf Z;z'],
\]
and write
\[
\widehat\theta_z=\widehat\theta(\bZ_z), \qquad
\widehat\theta_{z'}=\widehat\theta(\bZ_{z'}).
\]
Assumption (A2) ensures that these minimizers exist. By (A1) and their
optimality,
\[
Q(\widehat\theta_{z'};\bZ_z)
\ge
Q(\widehat\theta_z;\bZ_z)
+\frac{\mu}{2}
\|\widehat\theta_z-\widehat\theta_{z'}\|^2,
\]
\[
Q(\widehat\theta_z;\bZ_{z'})
\ge
Q(\widehat\theta_{z'};\bZ_{z'})
+\frac{\mu}{2}
\|\widehat\theta_z-\widehat\theta_{z'}\|^2.
\]
Adding the two inequalities yields
\[
\begin{aligned}
\mu\|\widehat\theta_z-\widehat\theta_{z'}\|^2
\le&
Q(\widehat\theta_{z'};\bZ_z)-Q(\widehat\theta_z;\bZ_z)+Q(\widehat\theta_z;\bZ_{z'})-Q(\widehat\theta_{z'};\bZ_{z'})\\
\le&\frac1n\Big\{\ell(\widehat\theta_{z'};z)-\ell(\widehat\theta_z;z)
+\ell(\widehat\theta_z;z')-\ell(\widehat\theta_{z'};z')
\Big\}\\
\le&\frac{2G}{n}\|\widehat\theta_z-\widehat\theta_{z'}\|,
\end{aligned}
\]
where the last inequality is due to (A3), and consequently,
$\|\widehat\theta_z-\widehat\theta_{z'}\|\le \frac{2G}{\mu n}$.
Therefore, by (A4), for every test covariate $x^*$,
\[
\begin{aligned}
|H(\bZ_z)(x^*)-H(\bZ_{z'})(x^*)|
&=
|h(\widehat\theta_z,x^*)-h(\widehat\theta_{z'},x^*)|
\\
&\le
L\|\widehat\theta_z-\widehat\theta_{z'}\|
\le
\frac{2LG}{\mu n}.
\end{aligned}
\]
Now set
\[
D=
H([\mathbf Z;\widetilde Z])(X^*)
-H([\mathbf Z;\widetilde Z'])(X^*).
\]
Conditional on $\mathbf Z$ and $X^*$, the preceding bound gives $|D|\le \frac{2LG}{\mu n}$ almost surely. Therefore,
\[
\|D\|_{\psi_2\mid\mathbf Z,X^*}\le \frac{2LG}{\mu n\sqrt{\log 2}},\quad
\mathbb E\left(D^2\right)\le \frac{4L^2G^2}{\mu^2n^2}.
\]
\end{proof}

\begin{proof}[Proof of Corollary~\ref{cor:ridge_stability}]
We verify the conditions of Proposition~\ref{prop:erm_stability}. For every $F\in\mathcal F_{n,m}$, coordinate-wise boundedness gives
\[
\|X_F\|_2^2
=
\sum_{j\in F}X_j^2
\le
|F|R^2
\le
d_{n,m}R^2=:R_{n,m}.
\]
The ridge objective is continuous and $2\lambda$-strongly convex, with minimum always attained.
Moreover, given any training data $\bZ_F$,
\[
\lambda\|\widehat\beta_F(\bZ_F)\|_2^2
\le
\frac1n\sum_{i=1}^nY_i^2
\le B^2.
\]
For any two ridge empirical minimizers $\beta=\hat{\beta}_F(\bZ_F)$ and $\beta'=\hat{\beta}_F(\bZ'_F)$, the preceding
bound gives
\[
\begin{aligned}
&
|(y-x_F^\top\beta)^2-(y-x_F^\top\beta')^2|
\\
&=
|x_F^\top(\beta-\beta')|
|2y-x_F^\top\beta-x_F^\top\beta'|
\\
&\le
2R_{n,m}B
\left(
1+\frac {R_{n,m}}{\sqrt\lambda}
\right)
\|\beta-\beta'\|_2 .
\end{aligned}
\]

Furthermore, noting that the prediction map
\[
h(\beta,x_F)=x_F^\top\beta
\]
is $R_{n,m}$-Lipschitz, Proposition~\ref{prop:erm_stability} immediately implies Corollary~\ref{cor:ridge_stability}.
\end{proof}

\begin{proof}[Proof of Corollary~\ref{cor:krr_stability}]
The RKHS objective is $2\lambda$-strongly convex. By optimality,
\[
\lambda
\|\widehat f_F\|_{\mathcal H_F}^2
\le B^2.
\]
For any two kernel-ridge empirical minimizers $f$ and $g$, the preceding
bound gives
\[
\begin{aligned}
|(y-f(x_F))^2-(y-g(x_F))^2|=&|f(x_F)-g(x_F)|
|2y-f(x_F)-g(x_F)|\\
\le&2\kappa B
\left(
1+\frac{\kappa}{\sqrt\lambda}
\right)\|f-g\|_{\mathcal H_F}.
\end{aligned}
\]
Furthermore, the prediction map $h(f,x_F)=f(x_F)$ satisfies
\[
|h(f,x_F)-h(g,x_F)|
\le
\kappa\|f-g\|_{\mathcal H_F},
\]
hence invoking Proposition~\ref{prop:erm_stability} immediately gives
\[
\stb(n)
=
\frac{2\kappa^2B}
{\sqrt{\log2}\lambda n}
\left(
1+\frac{\kappa}{\sqrt\lambda}
\right).
\]
\end{proof}

\section{Proofs of Main Theoretical Results}
\subsection{Definitions and Notations}\label{sec:app_def_not}
We define all notations needed for proofs in this section, although some of them are already defined in the main paper. We begin with the definition of some essential notations.
\begin{definition}[Uniform distribution $\cU_n^S$]
    Define $\cU_n^S$ as the uniform distribution on all size $n$ subsets of set $S$. More specifically, if $U\sim \cU_n^S$, then for any $u \subset S$ with $|u|=n$, we have
    \begin{align*}
        \bbP(U=u) = \frac{1}{{|S| \choose n}}.
    \end{align*}
\end{definition}

\begin{definition}[Probability mass function for $F$]
\label{def:prob_Q} Given a feature sampling probability vector $\bq\in [0,1]^M$ with $\sum_{j=1}^M q_j > 0$, we define its associated probability mass function of subsampled feature sets as $Q_{\bq}: 2^{[M]} \rightarrow [0,1]$, satisfying
\begin{align*}
Q_{\bq} (F) = \frac{\prod_{j\in F}q_j\prod_{k\notin F}(1-q_k)}{1-\prod_{k=1}^M(1-q_k)}\ind(F\neq\emptyset).
\end{align*}
As a special case, the p.m.f. $Q_{\bq^{(t)}}$ associated with the probability vector $\bq^{(t)}$ in iteration $t$ is sometimes denoted by $Q^{(t)}$; we also denote by $Q^{(t)}_{\no j}$ the conditional distribution for $F\sim Q^{(t)}$ given that $F\not\owns j$. Formally:
\begin{align*}
Q_{\no j}^{(t)}(F) = \frac{\ind(j\notin F)Q^{(t)}(F)}{\sum_{F' \subset [M]}\ind(j \notin F')Q^{(t)}(F')}.
\end{align*}

\end{definition}
\begin{table}[!htb]
    \centering
    \begin{tabular}{c|c}
    \hline
        $\mu_{I,F}(X^*;\bZ)$ & $H(\bZ_{I,F})(X^*_{F})$\\
        $\mu^{(t)}_k(X^*;\bZ)$ & $H(\bZ_{I_k^{(t)},F_k^{(t)}})(X^*_{F_k^{(t)}})$\\
        $\mu^{(t)}(X^*;\bZ)$ &  $\frac{1}{K}\sum_{k=1}^K\mu_k^{(t)}(X^*;\bZ)$\\
        $\mu_{\no j}^{(t)}(X^*;\bZ)$ & $\frac{1}{K}\sum_{k=1}^K\mu_{\tilde{I}_k^{(t)},\tilde{F}_k^{(t)}}(X^*;\bZ)$\\
        $\mu_{-i}^{(t)}(X^*;\bZ)$ & $\frac{1}{\sum_{k=1}^K \ind(i\notin I_k^{(t)})}\sum_{k=1}^K\ind(i\notin I_k^{(t)})\mu_k^{(t)}(X^*;\bZ)$\\
        $\mu_{-i}^{(t)-j}(X^*;\bZ)$ & $\frac{1}{\sum_{k=1}^K \ind(i\notin I_k^{(t)})\ind(j\notin F_k^{(t)})}\sum_{k=1}^K\ind(i\notin I_k^{(t)})\ind(j\notin F_k^{(t)})\mu_k(X^*;\bZ)$\\
        $\mu^*(X^*;Q,\bZ)$ & $\sum_{I\subset[N], F\subset[M]}\frac{1}{{N \choose n}}\ind(|I|=n)Q(F)\mu_{I,F}(X^*;\bZ)$\\
        $\tilde{\mu}_{-i}^{(t)}(X^*;\bZ)$ & $\frac{N}{N-n}\frac{1}{K}\sum_{k=1}^K\ind(i\notin I_k^{(t)})\mu_{I_k^{(t)},F_k^{(t)}}(X^*;\bZ)$\\
        $\tilde{\mu}_{-i}^{(t)-j}(X^*;\bZ)$ & $\frac{N}{N-n}\frac{1}{\bbP_{F\sim Q}(j\notin F)}\frac{1}{K}\sum_{k=1}^K\ind(i\notin I_k^{(t)})\ind(j\notin F_k^{(t)})\mu_{I_k^{(t)},F_k^{(t)}}(X^*;\bZ)$\\
    \hline
    \end{tabular}
    \caption{List of notations for predictions. In the fourth row, $\tilde{I_k}^{(t)}$'s are i.i.d. randomly sampled indices with the same uniform distribution as  $I_k^{(t)}$'s and $\tilde{F_k}^{(t)}$'s are i.i.d. randomly sampled indices with distribution $Q_{\no j}^{(t)}$. When the epoch $t$ is clear from context, the superscript $t$ can be omitted.}
    \label{tab:notations_mu}
\end{table}

\begin{table}[!htb]
    \centering
    \begin{tabular}{|p{3cm}|p{10cm}|}
    \hline
       $h_j(Z^*, \bZ, Q)$  &  $\ell(Z^*,\mu^*(X^*;Q_{\no j},\bZ)) - \ell(Z^*,\mu^*(X^*;Q,\bZ))$\\
       \hline
        $h_j^{(t)}(Z^*,\bZ)$ & $\ell(Z^*, \mu^{(t)}_{\no j}(X^*;\bZ)) - \ell(Z^*, \mu^{(t)}(X^*,\bZ))$\\
        \hline
        $h_j(Z^*,\bZ)$ & $h_j(Z^*, \bZ, Q^{*(T)}(\bZ))$ where $Q^{*(T)}(\bZ)$ is the final-iteration feature sampling mass function if applying Algorithm~\ref{algo:loco_adaptive_combinatorial} on $\bZ$. \\
        \hline
        $h_{j,N}(Z^*)$ & $\bbE_{\bZ}(h_j(Z^*, \bZ)|Z^*)$ where $Z^*$ is independent from $\bZ$, $\bZ = (Z_1,\dots,Z_N)$ with $Z_i\overset{i.i.d.}{\sim}\cP$. We sometimes omit the subscript $N$ when it is clear from the context.\\
        \hline
        $\til{h}_j(Z^*, \bZ)$  &  $h_j(Z^*,\bZ) - h_{j,N}(Z^*)$\\
        \hline
        $\hat{h}_j^{(t)}(Z_i,\bZ_{\no i,:})$ & $\hat{\Delta}^{(t)}_j(Z_i)$\\
        \hline
    $\hat{h}_j^{(t)}(Z_i,\bZ_{\no i,:}, Q)$& The LOCO-LOO score $\hat{\Delta}^{(t)}_j(Z_i)$ obtained by running Algorithms~\ref{algo:adamp_training} and \ref{algo:adamp_FI} sequentially, enforcing feature sampling distribution $Q$\\
    \hline
    \end{tabular}
    \caption{List of notations for feature importance functions.}
    \label{tab:notations_h}
\end{table}
For training data $\bZ=(\bX,\bY)$, we let $\bZ_{I,F}=(\bX_{I,F},\bY_I)$, $\bZ_{:,\no j}=(\bX_{:,\no j},\bY)$, $\bZ_{\no i,:} = (\bX_{\no i,:}, \bY_{\no i})$, and $\bZ_{\no i,\no j} = (\bX_{\no i,\no j}, \bY_{\no i})$. We summarize notations for predictions and notations for feature importance functions in Tables~\ref{tab:notations_mu} and \ref{tab:notations_h} respectively. In Table~\ref{tab:notations_mu}, for any new data point $Z^* = (X^*, Y^*)$, we define $\mu_{I,F}(X^*;\bZ)$ as the prediction of the base learner $H$ trained on minipatch $(I,F)$ at $X^*$. In the $t$th iteration, let $\mu^{(t)}_k(X^*;\bZ)$ be the prediction of the model trained on the $k$th minipatch $(I_k^{(t)},F_k^{(t)})$, and $\mu^{(t)}(X^*;\bZ)$ be the ensembled prediction. $\mu_{\no j}^{(t)}(X^*;\bZ)$ is the prediction of the reduced model without feature $j$. $\mu_{-i}^{(t)}(X^*;\bZ)$ and $\mu_{-i}^{(t)-j}(X^*;\bZ)$ are the LOO prediction and LOCO-LOO prediction respectively. $\mu^*(X^*;Q,\bZ)$ is the expectation taken over random subsample of minipatches under feature sampling distribution $Q$. We define $\tilde{\mu}_{-i}^{(t)}(X^*;\bZ)$ and $\tilde{\mu}_{-i}^{(t)-j}(X^*;\bZ)$ for technical purposes in the proof. In Table~\ref{tab:notations_h}, $h_j(Z^*, \bZ, Q)$ is the feature importance of feature $j$ evaluated at $X^*$, under $Q$, and $h_j^{(t)}(Z^*,\bZ)$ denotes the empirical counterpart in the $t$th iteration with finite sample. $h_j(Z^*,\bZ)$ denotes the feature importance when the feature sampling distribution is set to $Q^{*(T)}(\bZ)$.  In the main paper, we let $h_{j}(Z^*)=\bbE_{\bZ}(h_j(Z^*, \bZ|Z^*)$, while in the proof, we denote $\bbE_{\bZ}(h_j(Z^*, \bZ|Z^*)$ by $h_{j,N}(Z^*)$. $\hat{h}_j^{(t)}(Z_i,\bZ_{\no i,:})$ is the LOO feature occlusion in the $t$th iteration, and $\hat{h}_j^{(t)}(Z_i,\bZ_{\no i,:}, Q)$ denotes the same feature occlusion when enforcing feature sampling distribution $Q$. We also define $\til{h}_j(Z^*, \bZ)$ for technical purposes.

Recall the definition of the feature with the lowest estimated LOCO importance in the $t$th iteration in the main paper, i.e., $j_{\min}^{(t)} = \argmin_{j}\bar{\Delta}_j^{(t)}$. For any observation subset $S\subset[N]$, denote by $\{q^{*(\del S,t)}\}_{t=1}^T$ the output of Algorithm~\ref{algo:loco_adaptive_combinatorial} when applied to $\bZ_{\del S,:}$, and let their corresponding feature sampling probability mass function be $Q^{*(\del S,t)}$. In addition, let $\bar{\Delta}_j^{*(\del S,t)} = \frac{1}{N-|S|}\sum_{l\in S^c}h_j(Z_l,\bZ_{\del {S\cup \{l\}},:},Q^{*(\del S,t)})$ be the LOCO-LOO importance score of feature $j$ in the $t$th iteration, when infinitely many randomly sampled minipatches are available, and when we leave the observations in $S$ out. Let $j_{\min}^{*(\del S,t)} = \argmin_{j}\bar{\Delta}_j^{*(\del S,t)}$, the least important feature in the $t$th iteration, when leaving observations in $S$ out.
\begin{algorithm}[!htb]
\caption{Adaptive Minipatch LOCO Inference with Combinatorially Averaged Minipatch Ensemble}
\label{algo:loco_adaptive_combinatorial}
		\noindent{\textbf{Input}}: Training pairs $\bZ=(\bX,\bY)$ with $M$ features and sample size $N$, minipatch sizes $n$, $m$; number of training iterations $T$; base learner $H$; the highest sampling probability $\delta$.

            \begin{enumerate}
            \item Initialize feature sampling probability vector $\bq^{*(1)}=(\frac{m}{M},\dots,\frac{m}{M})^\top$.
            \item Given the base learner $H$, for $1\leq j\leq M$, compute $\bar{\Delta}_j^{*(1)} = \frac{1}{N}\sum_{i=1}^Nh_j(Z_i,\bZ_{\no i,:},Q_{\bq^{*(1)}})$.
            \item For $t = 2,\dots, T$,
                \begin{enumerate}     
                    \item Apply Algorithm \ref{algo:samp_prob} to $(\bar{\Delta}_1^{*(t-1)},\dots,\bar{\Delta}_M^{*(t-1)})$ to obtain the probability vector $\bq^{*(t)}$.
                    \item For $1\leq j\leq M$, compute $\bar{\Delta}_j^{*(t)} = \frac{1}{N}\sum_{i=1}^Nh_j(Z_i,\bZ_{\no i,:},Q_{\bq^{*(t)}})$.
                \end{enumerate}
		\end{enumerate}
		\textbf{Output}: Feature sampling probability vector $\bq^{*(t)},\,1\leq t\leq T$.
\end{algorithm}

\subsection{Proof of Theorem~\ref{thm:main}}
First, recall that our target takes the following form:
$\Delta_j^{(T)} = \bbE[h_j^{(T)}(Z^*,\bZ)|\bZ,\cA_{T}]$. Since $\hat{\Delta}_j^{(T)}(Z_i) = \hat{h}_j^{(T)}(Z_i,\bZ_{\backslash i,:})$, our goal is to show \eqref{eq:CLT_goal} and \eqref{eq:var}:
\begin{equation}\label{eq:CLT_goal}
        \frac{1}{\sigma_j\sqrt{N}}\sum_{i=1}^N(\hat{h}_j^{(T)}(Z_i,\bZ_{\backslash i,:})-\Delta_j^{(T)} ) \overset{d.}{\rightarrow} N(0,1),
\end{equation}
    where $\sigma_j^2 = \mathrm{Var}(h_{j,N}(Z^*))$;
\begin{equation}\label{eq:var}
    \frac{\hat{\sigma}_j^2}{\sigma_j^2} \overset{p.}{\rightarrow} 1.
\end{equation}

\subsubsection{Proving Asymptotic Normality in \eqref{eq:CLT_goal}}\label{thm_part1}
In the $t$th iteration of Algorithm \ref{algo:loco_adaptive}, we obtain the feature sampling probability vector $\bq^{(t)}$. Throughout this proof, we denote $Q_{\bq^{(t)}}$, the probability mass function for feature subset $F$ associated with $\bq^{(t)}$ (see formal definition in Def. \ref{def:prob_Q}), by $Q^{(t)}$. In addition, within this proof, we define an auxiliary algorithm (Algorithm~\ref{algo:loco_adaptive_combinatorial}) that updates $Q$ using combinatorially averaged minipatch ensembles; or equivalently, we let $\min_{1\leq t\leq T}K_t\rightarrow \infty$. Denote the output of this algorithm using the same training data by $\bq^{*(t)},\,1\leq t\leq T$, and their corresponding probability mass function by $Q^{*(t)}$. Furthermore, we denote the output of Algorithm~\ref{algo:loco_adaptive_combinatorial} applied to $\bZ_{\no i,:}$ by $\{\bq^{*(\no i,t)}\}_{t=1}^T$, and correspondingly, we denote $Q_{\bq^{(*(\no i,t)}}$ by $Q^{*(\no i,t)}$. 

Now we are ready to decompose $\frac{1}{\sigma_j\sqrt{N}}\sum_{i=1}^N(\hat{h}_j^{(T)}(Z_i,\bZ_{\backslash i,:})-\Delta_j^{(T)})$ into several error terms. We first define the following errors:
\begin{equation}\label{eq:main_decomp_def}
    \begin{split}
        \varepsilon_{i,1} = &\hat{h}_j^{(T)}(Z_i,\bZ_{\backslash i,:}) - h_j(Z_i,\bZ_{\backslash i,:},Q^{(T)}),\\
        \varepsilon_{i,2} = &h_j(Z_i,\bZ_{\backslash i,:},Q^{(T)})-h_j(Z_i,\bZ_{\backslash i,:},Q^{*(T)}),\\
        \varepsilon_{i,3} = &h_j(Z_i,\bZ_{\backslash i,:},Q^{*(T)})-h_j(Z_i,\bZ_{\backslash i,:},Q^{*(\no i, T)}),\\
        \varepsilon_{i,4} = &\til{h}_j(Z_i,\bZ_{\backslash i,:})-\bbE[\til{h}_j(Z^*,\bZ_{\backslash i,:})|\bZ_{\backslash i,:}],\\
        \varepsilon_{i,5} = &\bbE[h_j(Z_i,\bZ_{\backslash i,:},Q^{*(\no i, T)})|\bZ_{\backslash i,:}]-\bbE[h_j(Z^*,\bZ,Q^{*(T)})|\bZ],\\
        \varepsilon_{6} = 
        &\frac{\sqrt{N}}{\sigma_j}\left(\bbE[h_j(Z^*,\bZ,Q^{*(T)})|\bZ]-\bbE[h_j(Z^*,\bZ,Q^{(T)})|\bZ]\right),\\
        \varepsilon_{7} = 
        &\frac{\sqrt{N}}{\sigma_j}\left(\bbE[h_j(Z^*,\bZ,Q^{(T)})|\bZ]-\bbE[h_j^{(T)}(Z^*,\bZ)|\bZ]\right),
    \end{split}
\end{equation}
where $\til{h}_j(Z_i,\bZ_{\no i,:})$ is as defined in Table~\ref{tab:notations_h}:
$$
\til{h}_j(Z_i,\bZ_{\no i,:}) = h_j(Z_i,\bZ_{\no i,:},Q^{*(\no i,T)})-h_{j,N}(Z_i).
$$
Also define $\varepsilon_k = \frac{1}{\sigma_j\sqrt{N}}\sum_{i=1}^N\varepsilon_{i,k}$, for $1\leq k\leq 5$. With the notations above, we can then write
\begin{equation}
\label{eq:thm1_decomp}
    \begin{split}
        &\frac{1}{\sigma_j\sqrt{N}}\sum_{i=1}^N(\hat{h}_j^{(T)}(Z_i,\bZ_{\backslash i,:})-\Delta_j^{(T)})\\
        =&\frac{1}{\sigma_j\sqrt{N}}\sum_{i=1}^N(h_{j,N}(Z_i) - \bbE[h_{j,N}(Z_i)]) + \sum_{k=1}^7\varepsilon_k.
    \end{split}
\end{equation}
Note that $h_{j,N}(Z_i)$ was denoted by $h_j(Z_i)$ in the main paper. In this decomposition, $\varepsilon_1$, $\varepsilon_7$ are the errors in the feature importance scores arising from using finite average $\mu$ of minipatch predictors to approximate the combinatorial average $\mu^*$; $\varepsilon_2$ and $\varepsilon_6$ characterize the difference in the feature importance scores if the adaptive sampling distribution $Q^{(T)}$ is replaced by $Q^{*(T)}$, the updated sampling distribution obtained using the combinatorial average of minipatch predictors; $\varepsilon_3$ and $\varepsilon_5$ characterize the leave-one-out perturbation in the importance scores; while $\varepsilon_4$ is a double-centered, Hoeffding-degenerate remainder, and this form has been shown to be controlled by the stability of the feature importance score w.r.t. replacing one training sample~\citep{bayle2020cross}.

While for the first term in \eqref{eq:thm1_decomp}, Assumption~\ref{ass:3moment} implies that
$$
\frac{1}{(\sigma_j\sqrt{N})^3}\sum_{i=1}^N\bbE|h_{j,N}(Z_i) - \bbE[h_{j,N}(Z_i)]|^3\leq \frac{C}{\sqrt{N}}\rightarrow 0,
$$
and hence the Lyapunov's condition holds. Therefore, $\frac{1}{\sigma_j\sqrt{N}}\sum_{i=1}^N(h_{j,N}(Z_i) - \bbE[h_{j,N}(Z_i)])\overset{d.}{\rightarrow} N(0,1)$. In the following, we will show that $\sum_{k=1}^7\varepsilon_k\overset{p.}{\rightarrow}0$. 

\paragraph{Controlling $\varepsilon_1,\,\varepsilon_7$.} We first note that $\varepsilon_1$ and $\varepsilon_7$ arise from the finite sampling of minipatches. The following Lemma~\ref{lem:epsilon1_7} shows probabilistic bounds for these two error terms, conditioning on the data $\bZ$. Since $Q^{(T)}$ is a function of $\bZ$, the following statements also condition on $Q^{(T)}$, and the only randomness arises from the random sampling of minipatches at the $T$th iteration.
\begin{lem}\label{lem:epsilon1_7}
    Suppose that Assumptions~\ref{ass:adp_bounded} and \ref{ass:adp_n_minipatches} hold. There exist constants $c,\,C>0$, such that with probability at least $1-(MN)^{-c}$: for any $1\leq j\leq M$, $1\leq i\leq N$, and $1\leq t\leq MN$, the following holds: 
\begin{align}
     &\frac{1}{N}\sum_{i=1}^N\left|\hat{h}_j^{(t)}(Z_i,\bZ_{\backslash i,:}) - h_j(Z_i,\bZ_{\backslash i,:},Q^{(t)})\right|\notag\\
     &\quad\leq C\Big(\sqrt{\frac{q_j^{(t)}(\log M+\log N)}{K}}+\frac{\log M+\log N}{K}\Big),\label{eq:epsilon1}\\
     &\Big|h_j^{(t)}(Z^*,\bZ)-h_j(Z^*,\bZ,Q^{(t)})\Big|\leq C\Big(\sqrt{\frac{q_j^{(t)}(\log M+\log N)}{K}}+\frac{\log M+\log N}{K}\Big). \label{eq:epsilon7}
\end{align}
\end{lem}
Eq. \eqref{eq:epsilon1} in Lemma~\ref{lem:epsilon1_7} has been shown before in \cite{Liu2026model} (see Lemma S8 therein). For completeness, we still present both \eqref{eq:epsilon1} and \eqref{eq:epsilon7} in Lemma~\ref{lem:epsilon1_7} and include a brief proof of \eqref{eq:epsilon7} in Section~\ref{sec:proof_technical}.
Assumption~\ref{ass:adp_n_minipatches} together with Lemma~\ref{lem:epsilon1_7} imply that $\varepsilon_1,\,\varepsilon_7\overset{p.}{\rightarrow} 0$.

\paragraph{Controlling $\varepsilon_2,\,\varepsilon_3,\,\varepsilon_6$.} We note that $\varepsilon_2$, $\varepsilon_3$, and $\varepsilon_6$ arise from the differences between feature sampling distributions $Q^{(T)},\,Q^{*(T)},\,Q^{*(\no i,t)}$. The following three lemmas show that (i) the LOCO importance scores associated with similar feature sampling distributions will also be similar; (ii) $Q^{(T)}$ and $Q^{*(T)}$ are close; (iii) $Q^{*(T)}$ and $Q^{*(\no i,t)}$ are close.
\begin{lem}[Lemma S9 in \cite{Liu2026model}]\label{lem:diffQ_err}
Under Assumption~\ref{ass:adp_bounded}, for any training data $\bZ$ and test data $Z^*$, given two feature sampling probability vectors $\bq^{(1)}$ and $\bq^{(2)}$, we have
\begin{equation*}
\begin{split}
    \left|h_j(Z^*,\bZ,Q_{\bq^{(1)}}) -h_j(Z^*,\bZ,Q_{\bq^{(2)}})\right|\leq C\Big(\big(q_j^{(1)}+q_j^{(2)}\big)\big\|q^{(1)}-q^{(2)}\big\|_1 + \big|q_j^{(1)} - q_j^{(2)}\big|\Big),
\end{split}
\end{equation*}
 where $C>0$ is a constant, $Q_{\bq^(1)},\,Q_{\bq^(2)}:2^{[M]}\rightarrow [0,1]$ are the feature sampling distribution functions corresponding to $\bq^{(1)}$ and $\bq^{(2)}$, as defined in Definition~\ref{def:prob_Q}.
\end{lem}

\begin{lem}[Difference between $\bq^{(t)}$ and $\bq^{*(t)}$]\label{lem:Q_K_err}
    Suppose that Assumptions~\ref{ass:adp_bounded}, \ref{ass:adp_n_minipatches}, and \ref{ass:minDelta_q_bnd} hold. For $2\leq t\leq T\leq MN$, the feature sampling probability vectors $\bq^{(t)}$ and $\bq^{*(t)}$, given by Algorithms~\ref{algo:loco_adaptive} and \ref{algo:loco_adaptive_combinatorial}, satisfy the following difference bound:
    \begin{equation*}
        \begin{split}
            \|\bq^{(t)}-\bq^{*(t)}\|_1\leq Cm^{4t-5}\sqrt{\frac{mM(\log N+\log M)}{K}},
        \end{split}
    \end{equation*}
    with probability at least $1-(MN)^{-c}$, for some constants $c,\,C>0$.
\end{lem}
\begin{lem}[Difference between $\bq^{*(t)}$ and $\bq^{*(\no i, t)}$]\label{lem:Q_loo_err}
    Suppose that Assumptions~\ref{ass:stb}-\ref{ass:adp_prob_stb} and Assumption~\ref{ass:minDelta_q_bnd} hold. For any $1\leq t\leq T\leq MN$, the feature sampling probability vectors $\bq^{*(t)}$ and $\bq^{*(\no i, t)}$, the outputs of Algorithm \ref{algo:loco_adaptive_combinatorial} when applied to $\bZ$ and $\bZ_{\no i,:}$, respectively, satisfy the following difference bound:
    \begin{equation*}
        \begin{split}
            \|\bq^{*(t)}-\bq^{*(\no i, t)}\|_1\leq Cm^{4(t-1)}\frac{n\stb(n)\sqrt{\log N+\log M}+1}{N},
        \end{split}
    \end{equation*}
    with probability at least $1-(MN)^{-c}$, for some constants $c,\,C>0$.
\end{lem}
Combining Lemmas~\ref{lem:diffQ_err}-\ref{lem:Q_loo_err} and recalling the correspondence between $Q^{(t)},\,Q^{*(t)},\,Q^{*(\no i,t)}$ and $\bq^{(t)},\,\bq^{*(t)},\,\bq^{*(\no i,t)}$, we are now ready to bound $\varepsilon_2,\,\varepsilon_3,\,\varepsilon_6$. By applying Lemma~\ref{lem:diffQ_err}, Lemma~\ref{lem:Q_K_err} and the fact that $q_j^{(T)},q_j^{*(T)}\le 1$, with probability at least $1-(MN)^{-c}$, we have
\begin{equation}\label{eq:epsilon2_bound}
    \begin{split}
        |\varepsilon_{2}| \le& \frac{1}{\sigma_j\sqrt{N}}\sum_{i=1}^N|h_j(Z_i,\bZ_{\backslash i,:},Q^{(T)})-h_j(Z_i,\bZ_{\backslash i,:},Q^{*(T)})|\\
        \le& C\frac{\sqrt{N}}{\sigma_j}\Big(\big(q_j^{(T)}+q_j^{*(T)}\big)\big\|\bq^{(T)}-\bq^{*(T)}\big\|_1 + \big|q_j^{(T)} - q_j^{*(T)}\big|\Big)\\
        \le& Cm^{4T-5}\sqrt{\frac{mMN(\log N+\log M)}{\sigma_j^2 K}},
    \end{split}
\end{equation}
and
\begin{equation}\label{eq:epsilon7_bound}
    \begin{split}
        |\varepsilon_6| = &\frac{\sqrt{N}}{\sigma_j}\bbE[h_j(Z^*,\bZ,Q^{*(T)})|\bZ]-\bbE[h_j(Z^*,\bZ,Q^{(T)})|\bZ]\\
        \le& C\frac{\sqrt{N}}{\sigma_j}\bbE\left[\Big(\big(q_j^{(T)}+q_j^{*(T)}\big)\big\|\bq^{(T)}-\bq^{*(T)}\big\|_1 + \big|q_j^{(T)} - q_j^{*(T)}\big|\Big)\Big|\bZ\right]\\
        \le& Cm^{4T-5}\sqrt{\frac{mMN(\log N+\log M)}{\sigma_j^2 K}}.
    \end{split}
\end{equation}
By applying Lemma~\ref{lem:diffQ_err}, Lemma~\ref{lem:Q_loo_err} and the fact that $q_j^{*(T)},q_j^{*(\no i,T)}\le 1$, with probability at least $1-(MN)^{-c}$, we have
\begin{equation}\label{eq:epsilon3_bound}
    \begin{split}
        |\varepsilon_{3}| \le &\frac{1}{\sigma_j\sqrt{N}}\sum_{i=1}^N|h_j(Z_i,\bZ_{\backslash i,:},Q^{*(T)})-h_j(Z_i,\bZ_{\backslash i,:},Q^{*(\no i, T)})|\\
        \le &C\frac{\sqrt{N}}{\sigma_j}\Big(\big(q_j^{*(T)}+q_j^{*(\no i,T)}\big)\big\|\bq^{*(T)}-\bq^{*(\no i,T)}\big\|_1 + \big|q_j^{*(T)} - q_j^{*(\no i,T)}\big|\Big)\\
        \le & Cm^{4(T-1)}\frac{n\stb(n)\sqrt{\log N+\log M}+1}{\sigma_j\sqrt{N}}.
    \end{split}
\end{equation}
Thus, under Assumption~\ref{ass:adp_n_minipatches}, we have $\varepsilon_2,\,\varepsilon_6\overset{p.}{\rightarrow} 0$ according to Lemma~\ref{lem:diffQ_err} and Lemma~\ref{lem:Q_K_err}. Under Assumption~\ref{ass:adp_prob_stb}, we have $\varepsilon_3\overset{p.}{\rightarrow} 0$ according to Lemma~\ref{lem:diffQ_err} and Lemma~\ref{lem:Q_loo_err}.

\paragraph{Controlling $\varepsilon_4,\,\varepsilon_5$.} Furthermore, the following two lemmas yield bounds for the remaining two error terms $\varepsilon_4$ and $\varepsilon_5$.

\begin{lem}\label{lem:epsilon4_l2}
    Suppose that Assumptions~\ref{ass:stb}-\ref{ass:adp_prob_stb} and Assumption~\ref{ass:minDelta_q_bnd} hold. Then
    $$
    \bbE\big(\varepsilon_4^2\big)\leq \frac{Cm^{8(T-1)}\big(n^2\stb^2(n)+1\big)}{\sigma_j^2N}.
    $$
\end{lem}
By Assumption~\ref{ass:adp_prob_stb} and Lemma~\ref{lem:epsilon4_l2}, we have $\lim_{N\rightarrow \infty}\bbE\big(\varepsilon_4^2\big)=0$, which implies that $\varepsilon_4\overset{p.}{\rightarrow} 0$.
\begin{lem}\label{lem:epsilon5}
    Suppose that Assumptions~\ref{ass:stb}-\ref{ass:adp_prob_stb} and Assumption~\ref{ass:minDelta_q_bnd} hold. As long as $T\leq MN$, with probability at least $1-(MN)^{-c}$,
    \begin{equation*}
        \begin{split}
            |\varepsilon_5|\leq \frac{Cm^{4(T-1)}(n\mathrm{stb}(n)\sqrt{\log N + \log M}+1)}{\sigma_j\sqrt{N}}.
        \end{split}
    \end{equation*}
\end{lem}
By Assumption~\ref{ass:adp_prob_stb} and Lemma~\ref{lem:epsilon5}, we have $\varepsilon_5\overset{p.}{\rightarrow} 0$. Therefore, Assumptions~\ref{ass:stb}-\ref{ass:minDelta_q_bnd} together imply $\sum_{k=1}^7\varepsilon_k\overset{p.}{\rightarrow}0$. Combining this with the fact that $\frac{1}{\sigma_j\sqrt{N}}\sum_{i=1}^N(h_{j,N}(Z_i) - \bbE[h_{j,N}(Z_i)])\overset{d.}{\rightarrow} N(0,1)$, we establish~\eqref{eq:CLT_goal}.

\subsubsection{Consistency of the Variance Estimate}\label{thm_part2}
Let $\tilde{\Delta}_j(Z_i)= h_j(Z_i, \bZ_{\no i,:})$. By the definition of $h_{j,N-1}(Z)$, we can write
\begin{equation}
    h_{j,N-1}(Z_i) = \mathbb{E}_{\bZ_{\no i,:}}[h_j(Z_i, \bZ_{\no i,:}) \mid Z_i].
\end{equation}
Let $\bar{\tilde{\Delta}}_j = \frac{1}{N} \sum_{i=1}^N \tilde{\Delta}_j(Z_i)$
and
\[
\tilde{\sigma}_j^2 = \mathrm{Var}_Z(h_{j,N-1}(Z^*)), \quad
\hat{\tilde{\sigma}}_j^2 = \frac{1}{N-1} \sum_{i=1}^N \left( \tilde{\Delta}_j(Z_i) - \bar{\tilde{\Delta}}_j \right)^2.
\]
To establish the consistency of the variance estimator, we decompose the argument into three parts: (i) $\lim_{N\to\infty} \frac{{\sigma}_j^2}{\tilde{\sigma}_j^2} = 1$; (ii) $\frac{\hat{\tilde{\sigma}}_j^2}{\tilde{\sigma}_j^2} \xrightarrow{p} 1$; and (iii) $\frac{\hat{\sigma}_j^2}{\hat{\tilde{\sigma}}_j^2} \xrightarrow{p} 1$. Together, these three results imply $\frac{\hat{\sigma}_j^2}{\sigma_j^2}\xrightarrow{p} 1$.
\begin{enumerate}[label=(\roman*)]
\item To show the closeness between the variances of $h_{j,N-1}(Z^*)$ and $h_j(Z^*)$, we write
\[
\left| \frac{\tilde{\sigma}_j^2}{\sigma_j^2} - 1 \right|
= \sigma_j^{-2} \left|\mathbb{E}[(h_{j,N-1}(Z^*) - \mathbb{E}[h_{j,N-1}(Z^*)])^2 - (h_j(Z^*) - \mathbb{E}[h_j(Z^*)])^2]\right|.
\]
Using the inequality~\eqref{eq:ineq_for_var}
\begin{equation}\label{eq:ineq_for_var}
    |\mathbb{E}(\xi_1^2 - \xi_2^2)| \le \mathbb{E}(\xi_1 - \xi_2)^2 + 2 \sqrt{\mathbb{E}\xi_2^2} \sqrt{\mathbb{E}(\xi_1 - \xi_2)^2}
\end{equation}
which holds for any random variables $\xi_1$ and $\xi_2$, we obtain
\[
\left| \frac{\tilde{\sigma}_j^2}{\sigma_j^2} - 1 \right|
\le \sigma_j^{-2} \mathbb{E}[h_{j,N-1}(Z^*) - h_j(Z^*)]^2
+ 2 \sigma_j^{-1} \left( \mathbb{E}[h_{j,N-1}(Z^*) - h_j(Z^*)]^2 \right)^{1/2}.
\]

According to the inequality \eqref{eq:h_jN_stability} in the proof of Lemma~\ref{lem:epsilon4_l2} and Jensen's inequality, we have $\mathbb{E}[h_{j,N-1}(Z^*) - h_{j,N}(Z^*)]^2
\le \frac{Cm^{8(T-1)}(n^2\stb^2(n)+1)}{N^2}$.
Therefore, by Assumption~\ref{ass:adp_prob_stb},
\[
\left| \frac{\tilde{\sigma}_j^2}{\sigma_j^2} - 1 \right|
\le \frac{Cm^{8(T-1)}(n^2\stb^2(n)+1)}{\sigma_j^2 N^2}
+ \frac{Cm^{4(T-1)}\sqrt{n^2\stb^2(n)+1}}{\sigma_j N}
= o(1).
\]

\item For this part, we apply the variance consistency theorem in \cite{bayle2020cross} (see Theorem 5 therein). The $h_n(Z_i,Z_{B_j})$ in their notation is equivalent to $\tilde{\Delta}_j(Z_i)$ in our setting. According to the definition of the mean-squared stability $\gamma_{ms}$ in \cite{bayle2020cross}, we have
\[
\gamma_{ms}(h_j)
= \frac{1}{N-1} \sum_{l \ne i} \mathbb{E}\Big[
(h_j(Z_i, \bZ_{\no i,:}) 
- h_j(Z_i, \bZ_{\no i,:}^{\del l})^2
\Big].
\]
By inequalities \eqref{eq:gamma_loss_bnd2}, \eqref{eq:mean_qstar_leave_one}, \eqref{eq:mean_h_qstar_stability} and Lemma~\ref{lem:q_loo_l2} in the proof of Lemma~\ref{lem:epsilon4_l2}, we have $\gamma_{loss}(h_j) \le \gamma_{ms}(h_j) \le \frac{Cm^{8(T-1)}(n^2\stb^2(n)+1)}{N^2}=o(\frac{\sigma_j^2}{N})$ under Assumption~\ref{ass:adp_prob_stb}.

We then validate the uniform integrability of $\frac{[h_j(Z_i) - \mathbb{E}(h_j(Z_i))]^2}{\sigma_j^2}$ based on the bounded third-moment condition described in Assumption~\ref{ass:3moment}.
Let $\xi_{N,i} = \frac{[h_j(Z_i) - \mathbb{E}(h_j(Z_i))]^2}{\sigma_j^2}$, we can write
\[
\sup_N \mathbb{E}[|\xi_{N,i}| \ind(|\xi_{N,i}| > t)] 
\le \sup_N \big(\mathbb{E}|\xi_{N,i}^{3/2}|\big)^{2/3} \big(\mathbb{P}(|\xi_{N,i}| > t)\big)^{1/3}
\le C \left( \frac{\mathbb{E}|\xi_{N,i}|}{t} \right)^{1/3}
= C t^{-1/3},
\]
which converges to zero as $t \to \infty$. Then Theorem 5 in \cite{bayle2020cross} implies $\frac{\hat{\tilde{\sigma}}_j^2}{\tilde{\sigma}_j^2} \xrightarrow{p} 1$.

\item Recall the definition of $\hat{\sigma}^2_j=\frac{1}{N-1}\sum_{i=1}^N(\hat{\Delta}_j(Z_i)-\bar{\Delta}_j)^2$, where $\bar{\Delta}_j = \frac{1}{N}\sum_{i=1}^N\hat{\Delta}_j(Z_i)$.

By the definition of $\hat{\sigma}^2_j$ and $\hat{\tilde{\sigma}}_j^2$ and applying \eqref{eq:ineq_for_var} to $\frac{N-1}{N}|\hat{\sigma}_j^2 - \hat{\tilde{\sigma}}_j^2|$, we have
\[
|\hat{\sigma}_j^2 - \hat{\tilde{\sigma}}_j^2|
\le \frac{1}{N-1} \sum_{i=1}^N (\hat{\Delta}_j(Z_i) - \tilde{\Delta}_j(Z_i))^2
+ 2\hat{\tilde{\sigma}}_j \sqrt{ \frac{1}{N-1} \sum_{i=1}^N (\hat{\Delta}_j(Z_i) - \tilde{\Delta}_j(Z_i))^2 }.
\]
Recall the definition of $\varepsilon_{i,1}$, $\varepsilon_{i,2}$, and $\varepsilon_{i,3}$ in \eqref{eq:main_decomp_def}. We can write $\hat{\Delta}_j(Z_i) - \tilde{\Delta}_j(Z_i) = \varepsilon_{i,1} + \varepsilon_{i,2} + \varepsilon_{i,3}$, and hence 
\begin{equation*}
    \begin{split}
       |\hat{\sigma}_j^2 - \hat{\tilde{\sigma}}_j^2| & \le \frac{1}{N-1} \sum_{i=1}^N \left(\sum_{l=1}^3 \varepsilon_{i,l}\right)^2 + 2\hat{\tilde{\sigma}}_j \sqrt{ \frac{1}{N-1} \sum_{i=1}^N \left(\sum_{l=1}^3 \varepsilon_{i,l}\right)^2 }\\
       & \le \frac{1}{N-1} \left(\sum_{i=1}^N \sum_{l=1}^3 |\varepsilon_{i,l}|\right)^2 + 2\hat{\tilde{\sigma}}_j \frac{1}{\sqrt{N-1}} \left(\sum_{i=1}^N\sum_{l=1}^3 |\varepsilon_{i,l}|\right) \\
       & \le \frac{\sigma_j^2 N}{N-1} \left( \sum_{l=1}^3 \frac{1}{\sigma_j \sqrt{N}} \sum_{i=1}^N |\varepsilon_{i,l}| \right)^2 +  \frac{2\hat{\tilde{\sigma}}_j\sigma_j \sqrt{N}}{\sqrt{N-1}} \left( \sum_{l=1}^3 \frac{1}{\sigma_j \sqrt{N}} \sum_{i=1}^N |\varepsilon_{i,l}| \right) 
    \end{split}
\end{equation*}

Lemma~\ref{lem:epsilon1_7} shows that $\frac{1}{\sigma_j \sqrt{N}} \sum_{i=1}^N |\varepsilon_{i,1}| \xrightarrow{p} 0$. We have also shown $\frac{1}{\sigma_j \sqrt{N}} \sum_{i=1}^N |\varepsilon_{i,k}| \xrightarrow{p} 0$ for $k=2,\,3$ in \eqref{eq:epsilon2_bound} and \eqref{eq:epsilon3_bound}, respectively. Furthermore, note that in (i) and (ii) we already established that $\frac{\sigma_j^2}{\hat{\tilde{\sigma}}_j^2} \xrightarrow{p} 1$. 
Therefore, we conclude $\frac{\hat{\sigma}_j^2}{\hat{\tilde{\sigma}}_j^2} \xrightarrow{p} 1$.
\end{enumerate}
Now, combining (i), (ii), and (iii), we establish that $\frac{\hat{\sigma}_j^2}{\sigma_j^2} \xrightarrow{p} 1$.
\subsubsection{Combining Normality and Consistency of Variance Estimate}
In Section~\ref{thm_part1}, we have shown $\frac{1}{\sigma_j\sqrt{N}}\sum_{i=1}^N(\hat{h}_j^{(T)}(Z_i,\bZ_{\backslash i,:})-\Delta_j^{(T)} ) \overset{d.}{\rightarrow} N(0,1)$, and
in Section~\ref{thm_part2}, we have shown $\hat{\sigma}_j^2 \overset{p.}{\rightarrow} \sigma_j^2$. Combining these two conclusions together and applying Slutsky's theorem, we have $\frac{1}{\hat{\sigma}_j\sqrt{N}}\sum_{i=1}^N(\hat{h}_j^{(T)}(Z_i,\bZ_{\backslash i,:})-\Delta_j^{(T)}) \overset{d.}{\rightarrow} N(0,1)$, and consequently,
\begin{equation*}
    \lim_{N\rightarrow \infty}\bbP(\Delta_j^{(T)}\in \hat{\bbC}_j)=\lim_{N\rightarrow \infty}\bbP\left(\frac{1}{\hat{\sigma}_j\sqrt{N}}\left|\sum_{i=1}^N(\hat{h}_j^{(T)}(Z_i,\bZ_{\backslash i,:})-\Delta_j^{(T)} )\right|\le z_{\alpha/2}\right)=1-\alpha.
\end{equation*}
Therefore, the asymptotic coverage in Theorem~\ref{thm:main} is valid.

\subsection{Proofs of Key Lemmas}\label{sec:proof_technical}
\subsubsection{Proof of Lemma~\ref{lem:epsilon1_7}}
We note that \eqref{eq:epsilon1} is a restatement of Lemma S8 in \cite{Liu2026model}, adapted to our notation. While for \eqref{eq:epsilon7}, we will show in the following that its L.H.S. takes a very similar form to that of \eqref{eq:epsilon1}, and hence can be shown following similar arguments to those in \cite{Liu2026model}. Note that
\begin{equation*}
    \begin{split}
        h_j^{(t)}(Z^*,\bZ) =&\big(\mu(X^*;\bZ)-\mu_{\no j}(X^*;\bZ)\big)\big(2Y^*-\mu_{\no j}(X^*;\bZ)-\mu(X^*;\bZ)\big),\\
        h_j(Z^*,\bZ,Q)=&\big(\mu^*(X^*;Q,\bZ)-\mu^*(X^*;Q_{\no j},\bZ)\big)\big(2Y^*-\mu^*(X^*;Q_{\no j},\bZ)-\mu^*(X^*;Q,\bZ)\big),
    \end{split}
\end{equation*}
where we have omitted the superscript in $\mu^{(t)}$ and $Q^{(t)}$ for notational simplicity. Here, $\mu(X^*;\bZ) = \frac{1}{K}\sum_{k=1}^K\mu_{I_k,F_k}(X^*;\bZ)$, $\mu_{\no j}(X^*;\bZ) = \frac{1}{K}\sum_{k=1}^K\mu_{\til{I}_k,\til{F}_k}(X^*;\bZ)$, $(I_k, F_k)\sim \cU^{[N]}_n \times Q$, and $(\til{I}_k,\til{F}_k)\sim \cU^{[N]}_n\times Q_{\no j}$; $\mu^*(X^*;Q,\bZ) = \bbE_{I\sim\cU^{[N]}_n}\bbE_{F\sim Q} \mu_{I,F}(X^*)$, $\mu^*(X^*;Q_{\no j},\bZ) = \bbE_{I\sim\cU^{[N]}_n}\bbE_{F\sim Q_{\no j}^{(t)}} \mu_{I,F}(X^*)$. Therefore, $h_j^{(t)}(Z^*,\bZ)-h_j(Z^*,\bZ,Q)$ takes a very similar form to $\hat{h}_j^{(t)}(Z_i,\bZ_{\no i,:}) - h_j(Z_i,\bZ_{\no i,:}, Q)$, where the only differences are: (i) we do not leave sample $i$ out when computing the ensembled predictor and we evaluate the predictors on a new test sample $Z^*=(X^*,Y^*)$; and (ii) when computing the Monte-Carlo approximation of the ensembled predictor, we directly average all $K$ minipatch predictors trained on $(I_k,F_k)\sim \cU^{[N]}_n\times Q$ or $(\til{I}_k,\til{F}_k)\sim \cU^{[N]}_n\times Q_{\no j}$. Therefore, the arguments in the proof of Lemma S8 in \cite{Liu2026model} for showing \eqref{eq:epsilon1} all still hold for $h_j^{(t)}(Z^*,\bZ)-h_j(Z^*,\bZ,Q)$. Hence, with probability at least $1-(MN)^{-c}$, for any $1\leq j\leq M$,
\begin{equation*}
    \Big|h_j^{(t)}(Z^*,\bZ)-h_j(Z^*,\bZ,Q)\Big|\leq C\Big(\sqrt{\frac{q_j(\log M+\log N)}{K}}+\frac{\log M+\log N}{K}\Big).
\end{equation*}
The proof of Lemma~\ref{lem:epsilon1_7} is now complete.

\subsubsection{Proof of Lemma~\ref{lem:Q_K_err}}

Recall that we update $\bq^{(t+1)}$ and $\bq^{*(t+1)}$ based on $\bar{\Delta}_j^{(t)}$ in Algorithm~\ref{algo:samp_prob} and $\bar{\Delta}_j^{*(t)}$ in Algorithm~\ref{algo:loco_adaptive_combinatorial}, $j = 1,\dots,M$, respectively. Lemma~\ref{lem:q_err_bnd} shows an upper bound for $\big|q^{(t+1)}_j-q^{*(t+1)}_j\big|$ as a function of $\bar{\Delta}^{(t)}-\Delta^{*(t)}$.
\begin{lem}\label{lem:q_err_bnd}
    Suppose that we apply Algorithm~\ref{algo:samp_prob} to $\{\Delta_j^{(1)}\}_{j=1}^M$ and $\{\Delta_j^{(2)}\}_{j=1}^M$ to obtain $\bq^{(1)},\,\bq^{(2)}\in (0,\delta]^{M}$, respectively. Let $\varepsilon = \Delta^{(1)}-\Delta^{(2)}\in \bbR^{M}$, and $\varepsilon_{\min} = \Big|\min_l \Delta_l^{(1)} - \min_l \Delta_l^{(2)}\Big|$. There exist constants $c,\,C>0$ depending on $c_0$ such that as long as $\|\varepsilon\|_1+M\varepsilon_{\min}\leq c$, then 
    \begin{equation*}
        \begin{split}
            |q_j^{(1)} - q_j^{(2)}|&\leq \begin{cases}
                Cm^2\Big(q^{(1)}_j\big(\|\varepsilon\|_1 + M\varepsilon_{\min}\big) + |\varepsilon_j| + \varepsilon_{\min}\Big), &\text{if }q_j^{(1)},\,q^{(2)}_j<\delta,\\
                Cm^2\big(\|\varepsilon\|_1 + M\varepsilon_{\min}\big), &\text{otherwise};
            \end{cases}\\
            \|\bq^{(1)} - \bq^{(2)}\|_1&\leq \frac{Cm^3}{\delta}\big(\|\varepsilon\|_1 + M\varepsilon_{\min}\big).
        \end{split}
    \end{equation*}
\end{lem}
Furthermore, in the following, we will show a probabilistic bound for $\bar{\Delta}_j^{(t)} - \bar{\Delta}_j^{*(t)}$ as a function of the difference between $\bq^{(t)}$ and $\bq^{*(t)}$. Note that
\begin{equation*}
\begin{split}
    \Big|\bar{\Delta}_j^{(t)} - \bar{\Delta}_j^{*(t)}\Big| &= \Big|\frac{1}{N}\sum_{i=1}^N(\hat{h}_j^{(t)}(Z_i,\bZ_{\backslash i,:}) - \frac{1}{N}\sum_{i=1}^Nh_j(Z_i,\bZ_{\backslash i,:},Q^{*(t)}))\Big|\\
    &\le \Big|\frac{1}{N}\sum_{i=1}^N(\hat{h}_j^{(t)}(Z_i,\bZ_{\backslash i,:}) - \frac{1}{N}\sum_{i=1}^Nh_j(Z_i,\bZ_{\backslash i,:},Q^{(t)}))\Big| \\ & \quad + \Big|\frac{1}{N}\sum_{i=1}^Nh_j(Z_i,\bZ_{\backslash i,:},Q^{(t)})) - \frac{1}{N}\sum_{i=1}^Nh_j(Z_i,\bZ_{\backslash i,:},Q^{*(t)}))\Big|
\end{split}   
\end{equation*}
By Lemma~\ref{lem:epsilon1_7}, we have that with probability at least $1-(MN)^{-c}$, 
\begin{equation*}
    \left|\frac{1}{N}\sum_{i=1}^N\hat{h}_j^{(t)}(Z_i,\bZ_{\backslash i,:}) - \frac{1}{N}\sum_{i=1}^Nh_j(Z_i,\bZ_{\backslash i,:},Q^{(t)})\right|\le C\Big(\sqrt{q_j^{(t)}\varepsilon_K} + \varepsilon_K\Big),
\end{equation*}
where we have denoted $\frac{\log M+\log N}{K}$ by $\varepsilon_K$. By Lemma~\ref{lem:diffQ_err}, we have
\begin{equation*}
    \Big|\frac{1}{N}\sum_{i=1}^Nh_j(Z_i,\bZ_{\backslash i,:},Q^{(t)})) - \frac{1}{N}\sum_{i=1}^Nh_j(Z_i,\bZ_{\backslash i,:},Q^{*(t)}))\Big|\le C\Big(\big(q_j^{(t)}+q_j^{*(t)}\big)\big\|\bq^{(t)}-\bq^{*(t)}\big\|_1 + \big|q_j^{(t)} - q_j^{*(t)}\big|\Big).
\end{equation*}
Therefore, combining the two inequalities above, we have
\begin{equation}\label{eq:lemma3_j}
    \Big|\bar{\Delta}_j^{(t)} - \bar{\Delta}_j^{*(t)}\Big| \le C\Big(C(\sqrt{q_j^{(t)}\varepsilon_K} + \varepsilon_K) + \Big(\big(q_j^{(t)}+q_j^{*(t)}\big)\big\|\bq^{(t)}-\bq^{*(t)}\big\|_1 + \big|q_j^{(t)} - q_j^{*(t)}\big|\Big)\Big).
\end{equation}
Since $\sum_{j=1}^M (\sqrt{q_j^{(t)}\varepsilon_K} + \varepsilon_K)\leq \sqrt{M(\sum_{j=1}^M q_j^{(t)})\varepsilon_K} + M\varepsilon_K \leq 2\sqrt{mM\varepsilon_K}$, where the last inequality is due to Assumption~\ref{ass:adp_n_minipatches}, we have
\begin{equation}\label{eq:lemma3_l1}
    \Big\|\bar{\Delta}^{(t)} - \bar{\Delta}^{*(t)}\Big\|_1 \le C\Big(\sqrt{mM\varepsilon_K} + m\big\|\bq^{(t)}-\bq^{*(t)}\big\|_1\Big),
\end{equation}
with probability at least $1-(MN)^{-c}$. Given Lemma~\ref{lem:q_err_bnd}, \eqref{eq:lemma3_j}, and \eqref{eq:lemma3_l1}, we are ready to prove the following claim via an induction argument.
\begin{claim}
    \label{claim:induction_Q_K}
Suppose that Assumptions~\ref{ass:adp_bounded}-\ref{ass:minDelta_q_bnd} hold. With probability at least $1-(MN)^{-c}$, we have
\begin{equation}\label{eq:induction_Q_K_l1}
        \Big\|\bar{\Delta}^{(t)}-\bar{\Delta}^{*(t)}\Big\|_1\leq Cm^{4t-4}\sqrt{\frac{mM(\log M+\log N)}{K}}.
\end{equation}
If there exists $t\le t' \le T$ such that $j=j_{\min}^{(t')}$ for some $t\le t' \le T$, then 
\begin{equation}\label{eq:induction_Q_K_j}
    \Big|\bar{\Delta}_j^{(t)}-\bar{\Delta}_j^{*(t)}\Big|\leq Cm^{4t-4}\sqrt{\frac{m(\log M+\log N)}{MK}}.
\end{equation}
Furthermore,
\begin{equation}\label{eq:induction_Q_K_q}
    \|\bq^{(t+1)}-\bq^{*(t+1)}\|_1 \le Cm^{4t-1}\sqrt{\frac{mM(\log M+\log N)}{K}}.
\end{equation}
\end{claim}
Given Claim~\ref{claim:induction_Q_K}, Lemma~\ref{lem:Q_K_err} is immediately implied by Eq.~\ref{eq:induction_Q_K_q}.
\begin{proof}[Proof of Claim~\ref{claim:induction_Q_K}]
In the first iteration, $\bq^{(1)} = \bq^{*(1)} = \frac{m}{M}\mathbf{1}_M$. We can then invoke \eqref{eq:lemma3_j} and \eqref{eq:lemma3_l1}, and since $\sqrt{\frac{m}{M}\varepsilon_K}+\varepsilon_K\leq C\sqrt{\frac{m}{M}\varepsilon_K}$ by Assumption~\ref{ass:adp_n_minipatches}, we immediately obtain \eqref{eq:induction_Q_K_j} and \eqref{eq:induction_Q_K_l1} with $t=1$.
Now suppose that there exists $2\leq t_0\leq T$ such that Claim~\ref{claim:induction_Q_K} holds for $1\leq t\leq t_0-1$. In the following, we will show that this implies that Claim~\ref{claim:induction_Q_K} also holds for $t=t_0$. For simplicity, we define the following notations that will be used only throughout this proof:
    \begin{equation*}
        \begin{split}            \varepsilon^{(t)}:=&\Big\|\bar{\Delta}^{(t)}-\bar{\Delta}^{*(t)}\Big\|_1,\\         \varepsilon_j^{(t)}:=&\Big|\bar{\Delta}_j^{(t)}-\bar{\Delta}_j^{*(t)}\Big|,\\       \varepsilon_{\min}^{(t)}:=&\Big|\min_l\bar{\Delta}_l^{(t)}-\min_l\bar{\Delta}_l^{*(t)}\Big|,\\
        \xi:=&\sqrt{mM\varepsilon_K} = \sqrt{\frac{mM(\log M+\log N)}{K}}.
        \end{split}
    \end{equation*}
    We can then write  \begin{equation}\label{eq:induction_Q_K_l1_t0}
        \begin{split}
           \varepsilon^{(t_0-1)}\leq Cm^{4t_0-8}\xi;
        \end{split}
    \end{equation}
    and $\forall t\in [t_0-1,T]$, if $j=j_{\min}^{(t)}$, then
    \begin{equation}\label{eq:induction_Q_K_j_t0}
         \varepsilon_j^{(t_0-1)}\leq \frac{Cm^{4t_0-8}}{M}\xi.
    \end{equation}
    Thus, we have 
    \begin{equation}\label{eq:induction_Q_K_min_t0}
        \begin{split}
            \varepsilon_{\min}^{(t_0-1)}\leq&\varepsilon^{(t_0-1)}_{j_{\min}^{(t_0-1)}} + \varepsilon^{(t_0-1)}_{j_{\min}^{*(t_0-1)}}\\
            \leq&\frac{Cm^{4t_0-8}}{M}\xi.
        \end{split}
    \end{equation}
    The bounds above and Lemma~\ref{lem:q_err_bnd} together imply that
    \begin{equation*}
        \begin{split}
            \Big\|\bq^{(t_0)}-\bq^{*(t_0)}\Big\|_1\leq &Cm^3\Big(\varepsilon^{(t_0-1)}+M\varepsilon_{\min}^{(t_0-1)}\Big)\\
            \leq&Cm^{4t_0-5}\xi.
        \end{split}
    \end{equation*}
    In addition, if $j=j_{\min}^{(t)}$ for some $t\in [t_0,T]$, then Assumption~\ref{ass:minDelta_q_bnd} implies that $q_j^{(t_0)}\leq \frac{Cm}{M}<\delta$. Lemma~\ref{lem:q_err_bnd} then yields
    \begin{equation*}
        \begin{split}
            \Big|q_j^{(t_0)}-q_j^{*(t_0)}\Big|\leq &Cm^2\left(q_j^{(t_0)}\big(\varepsilon^{(t_0-1)}+M\varepsilon_{\min}^{(t_0-1)}\big)+\varepsilon_j^{(t_0-1)}+\varepsilon_{\min}^{(t_0-1)}\right)\\
            \leq&\frac{Cm^{4t_0-5}}{M}\xi. 
        \end{split}
    \end{equation*}
By \eqref{eq:lemma3_j}, if $j = j_{\min}^{(t)}$ for some $t_0\le t \le T$, we have
\begin{equation*}
    |\bar{\Delta}_j^{(t_0)} - \bar{\Delta}_j^{*(t_0)}| \le \frac{C\xi}{M} + \frac{Cm^{4t_0-4}\xi}{M}+\frac{Cm^{4t_0-5}\xi}{M} \le \frac{Cm^{4t_0-4}\xi}{M},
\end{equation*}
where we use $q_j^{(t_0)}, q_j^{*(t_0)}\leq \frac{Cm}{M}$. By \eqref{eq:lemma3_l1}, we can get
\begin{equation*}
    \Big\|\bar{\Delta}^{(t_0)} - \bar{\Delta}^{*(t_0)}\Big\|_1 \le C\xi+ Cm^{4t_0-4}\xi \le Cm^{4t_0-4}\xi,
\end{equation*}
where we use $\sum_{j=1}^Mq_j^{(t_0)} = \sum_{j=1}^Mq_j^{*(t_0)} = m$. Invoking Lemma~\ref{lem:q_err_bnd} again, we arrive at \eqref{eq:induction_Q_K_q} with $t=t_0$.
Thus, Claim~\ref{claim:induction_Q_K} is true with $t = t_0$. By induction, the proof of Claim~\ref{claim:induction_Q_K} is now complete.
\end{proof}
 
\subsubsection{Proof of Lemma~\ref{lem:Q_loo_err}}
Recall that in Algorithm~\ref{algo:loco_adaptive_combinatorial}, we iteratively update $q^{*(t)}$ based on $\bar{\Delta}_j^{*(t-1)} = \frac{1}{N}\sum_{l=1}^Nh_j(Z_l,\bZ_{\del l,:},Q^{*(t-1)})$, $j=1,\dots,M$. The key is to show that the computation of $\bar{\Delta}_j^{*(t-1)}$ is stable against the leaving-one-observation-out operation through a proof by induction. 

Recall our definitions of $\bar{\Delta}_j^{*(\del S,t)}$, $\bq^{*(\del S,t)}$, and $Q^{*(\del S,t)}$ before Assumption~\ref{ass:minDelta_q_bnd}. For any subset $S\subset[N]$, suppose we assign an arbitrary order to the elements in $S$ so that we can write $S=(i_1,\dots,i_k,\dots,i_{|S|})$; also let $S_k=(i_1,\dots,i_k)$, a subset of $S$ consisting of the first $k$ elements. The following lemma shows a key bound for $\bar{\Delta}_j^{*(\del S,t)}-\bar{\Delta}_j^{*(t)}$ in each iteration $t\geq 1$.
\begin{lem}\label{lem:Delta_leaveSout}
    With probability at least $1-(MN)^{-c}$, for each iteration $1\leq t\leq T$ and for each observation subset $S\subset[N]$ with size $|S|\leq C$,
    \begin{equation*}
        \begin{split}
            &\Big|\bar{\Delta}_j^{*(\del S,t)}-\bar{\Delta}_j^{*(t)}\Big|\\
            &\quad \leq \frac{C(n\stb(n)\sqrt{\log N+\log M}+1)}{N}\sum_{k=1}^{|S|}\frac{1}{N-k}\sum_{i'\notin S_k}(q_j^{*(\del S,t)} + q_j^{*(\del \{i_k,i'\},t)})\\
            &\quad \quad + \frac{Cn}{N}\sum_{k=1}^{|S|}\frac{1}{N-k}\sum_{i'\notin S_k}\Big\{\big(q_j^{*(t)}+q_j^{*(\del S,t)}+q_j^{*(\del\{i_k,i'\},t)}\big)\big\|\bq^{*(t)}-\bq^{*(\del\{i_k,i'\},t)}\big\|_1\\
            &\qquad\qquad\qquad\qquad\qquad\qquad+\big|q_j^{*(t)} - q_j^{*(\del \{i_k,i'\},t)}\big|\Big\}\\
            &\quad\quad+C(q_j^{*(t)}+q_j^{*(\del S, t)})\big\|\bq^{*(t)}-\bq^{*(\del S,t)}\big\|_1 + C\big|q_j^{*(t)} - q_j^{*(\del S,t)}\big| + \frac{C|S|q_j^{*(t)}}{N}.
        \end{split}
    \end{equation*}
\end{lem}
Lemma~\ref{lem:Delta_leaveSout} shows that the error in the $t$th iteration due to leaving some observations out depends on the probability vectors' error in previous iterations. We are now ready to show the following claim about $\Big|\bar{\Delta}_j^{*(\del S,t)}-\bar{\Delta}_j^{*(t)}\Big|$ via a proof of induction:
\begin{claim}\label{claim:induction_Q_loo}
With probability at least $1-(MN)^{-c}$, $\forall 1\leq t\leq T$, $\forall\,S\subset[N]$ with size $|S|\leq 2$, we have
\begin{equation}\label{eq:induction_Q_loo_l1}
        \Big\|\bar{\Delta}^{*(\del S,t)}-\bar{\Delta}^{*(t)}\Big\|_1\leq Cm^{4t-3}\frac{n\stb(n)\sqrt{\log N+\log M}+1}{N}.
\end{equation}
 If there exist $S'\subset[N]\text{ with size }|S'|\leq 2$, and $j=j_{\min}^{*(\del S',t')}\text{ for some }t\leq t'\leq T$, then we have
\begin{equation}\label{eq:induction_Q_loo_j}
    \Big|\bar{\Delta}_j^{*(\del S,t)}-\bar{\Delta}_j^{*(t)}\Big|\leq \frac{Cm^{4t-3}}{M}\frac{n\stb(n)\sqrt{\log N+\log M}+1}{N}.
\end{equation}
Furthermore, 
\begin{equation}\label{eq:induction_Q_loo}
    \Big\|\bq^{*(\del S,t+1)}-\bq^{*(t+1)}\Big\|_1\leq Cm^{4t}\frac{n\stb(n)\sqrt{\log N+\log M}+1}{N}.
\end{equation}
\end{claim}
\begin{proof}[Proof of Claim~\ref{claim:induction_Q_loo}]
    We start by considering the first iteration. Since $\bq^{*(\del S, 1)} = \frac{m}{M}\mathbf{1}_M$ for any $S\subset[N]$, Lemma~\ref{lem:Delta_leaveSout} suggests that with probability at least $1-(MN)^{-c}$,
    \begin{equation}\label{eq:induction_Q_loo_j_1st}
        \begin{split}
            &\Big|\bar{\Delta}_j^{*(\del S,1)}-\bar{\Delta}_j^{*(1)}\Big|\\
            &\quad \leq \frac{Cn\stb(n)\sqrt{\log N+\log M}}{N}\sum_{k=1}^{|S|}\frac{1}{N-k}\sum_{i'\notin S_k}(q_j^{*(\del S,1)} + q_j^{*(\del \{i_k,i'\},1)})+\frac{Cq_j^{*(1)}}{N}\\
            &\quad \leq \frac{Cm}{M}\frac{n\stb(n)\sqrt{\log N+\log M}+1}{N},
        \end{split}
    \end{equation}
    for any $S\subset[N]$ with size $|S|\leq 2$. \eqref{eq:induction_Q_loo_j_1st} immediately implies that
    \begin{equation}\label{eq:induction_Q_loo_l1_1st}
        \Big\|\bar{\Delta}^{*(\del S,1)}-\bar{\Delta}^{*(1)}\Big\|_1\leq Cm\frac{n\stb(n)\sqrt{\log N+\log M}+1}{N}.
    \end{equation}
    \eqref{eq:induction_Q_loo_j_1st} and \eqref{eq:induction_Q_loo_l1_1st} together imply that \eqref{eq:induction_Q_loo_j} and \eqref{eq:induction_Q_loo_l1} in Claim~\ref{claim:induction_Q_loo} hold when $t=1$. 

    Now suppose that there exists $2\leq t_0\leq T$ such that \eqref{eq:induction_Q_loo_j} and \eqref{eq:induction_Q_loo_l1} in Claim~\ref{claim:induction_Q_loo} hold for $1\leq t\leq t_0-1$. In the following, we will show that this implies that \eqref{eq:induction_Q_loo_j} and \eqref{eq:induction_Q_loo_l1} in Claim~\ref{claim:induction_Q_loo} also hold for $t=t_0$. For simplicity, we define the following notations that will be used only throughout this proof:
    \begin{equation*}
        \begin{split}
            \varepsilon(N):=&\frac{n\stb(n)\sqrt{\log N+\log M}+1}{N},\\
            \varepsilon^{(S,t)}:=&\Big\|\bar{\Delta}^{*(\del S,t)}-\bar{\Delta}^{*(t)}\Big\|_1,\\
            \varepsilon_j^{(S,t)}:=&\Big|\bar{\Delta}_j^{*(\del S,t)}-\bar{\Delta}_j^{*(t)}\Big|,\\
            \varepsilon_{\min}^{(S,t)}:=&\Big|\min_l\bar{\Delta}_l^{*(\del S,t)}-\min_l\bar{\Delta}_l^{*(t)}\Big|.
        \end{split}
    \end{equation*}
    We can then write 
    \begin{equation}\label{eq:induction_Q_loo_l1_t0}
        \begin{split}
           \varepsilon^{(S,t_0-1)}\leq Cm^{4t_0-7}\varepsilon(N),\quad\forall S\subset[N]\text{ with size }|S|\leq 2;
        \end{split}
    \end{equation}
    and $\forall t\in [t_0-1,T]$, $S'\subset[N]$ with $|S'|\leq 2$, if $j=j_{\min}^{*(\del S',t)}$, then
    \begin{equation}\label{eq:induction_Q_loo_j_t0}
         \varepsilon_j^{(S,t_0-1)}\leq \frac{Cm^{4t_0-7}}{M}\varepsilon(N).
    \end{equation}
    Therefore, we further have 
    \begin{equation}\label{eq:induction_Q_loo_min_t0}
        \begin{split}
            \varepsilon_{\min}^{(S,t_0-1)}\leq&\varepsilon^{(S,t_0-1)}_{j_{\min}^{*(\del S,t_0-1)}} + \varepsilon^{(S,t_0-1)}_{j_{\min}^{*(t_0-1)}}\\
            \leq&\frac{Cm^{4t_0-7}}{M}\varepsilon(N).
        \end{split}
    \end{equation}
    The bounds above, Assumption~\ref{ass:adp_prob_stb}, and Lemma~\ref{lem:q_err_bnd} together imply that for any $S\subset [N]$ with $|S|\leq 2$,
    \begin{equation}\label{eq:q_l1_err}
        \begin{split}
            \Big\|\bq^{*(t_0)}-\bq^{*(\del S,t_0)}\Big\|_1\leq &Cm^3\Big(\varepsilon^{(S,t_0-1)}+M\varepsilon_{\min}^{(S,t_0-1)}\Big)\\
            \leq&Cm^{4t_0-4}\varepsilon(N).
        \end{split}
    \end{equation}
    In addition, if $j=j_{\min}^{*(\del S',t)}$ for some $t\in [t_0,T]$, $S'\subset[N]$ with $|S'|\leq 2$, then Assumption~\ref{ass:minDelta_q_bnd} implies that $q_j^{(\del S,t_0)}\leq \frac{Cm}{M}<\delta$. Lemma~\ref{lem:q_err_bnd} and \eqref{eq:induction_Q_loo_j_t0}-\eqref{eq:induction_Q_loo_min_t0} then yield
    \begin{equation*}
        \begin{split}
            \Big|q_j^{*(t_0)}-q_j^{*(\del S,t_0)}\Big|\leq &Cm^2\left(q_j^{*(\del S,t_0)}\big(\varepsilon^{(S,t_0-1)}+M\varepsilon_{\min}^{(S,t_0-1)}\big)+\varepsilon_j^{(S,t_0-1)}+\varepsilon_{\min}^{(S,t_0-1)}\right)\\
            \leq&\frac{Cm^{4t_0-4}}{M}\varepsilon(N). 
        \end{split}
    \end{equation*}
    Now we are ready to show that Claim~\ref{claim:induction_Q_loo} holds for $t=t_0$. By Lemma~\ref{lem:Delta_leaveSout}, for any $S\subset[N]$ with $|S|\leq 2$, we can write
    \begin{equation*}
        \begin{split}
            &\Big\|\bar{\Delta}^{*(\del S,t_0)}-\bar{\Delta}^{*(t_0)}\Big\|_1\\
            &\quad\leq Cm\varepsilon(N)+ \frac{Cmn}{N}\sum_{k=1}^{|S|}\frac{1}{N-k}\sum_{i'\notin S_k}\big\|\bq^{*(t_0)}-\bq^{*(\del\{i_k,i'\},t_0)}\big\|_1\\
            &\quad\quad+Cm\big\|\bq^{*(t_0)}-\bq^{*(\del S,t_0)}\big\|_1\\
            &\quad\leq Cm^{4t_0-3}\varepsilon(N),
        \end{split}
    \end{equation*}
    where we have utilized the fact that $\sum_{j=1}^Mq_j^{*(\del S,t)}=m$.
    If $j=j_{\min}^{*(S',t')}$ for some $t_0\leq t'\leq T$, $S'\subset[N]$ with $|S'|\leq 2$, then 
    \begin{equation*}
        \begin{split}
            &\Big|\bar{\Delta}_j^{*(\del S,t_0)}-\bar{\Delta}_j^{*(t_0)}\Big|\\
            &\quad \leq C\varepsilon(N)\sum_{k=1}^{|S|}\frac{1}{N-k}\sum_{i'\notin S_k}(q_j^{*(\del S,t_0)} + q_j^{*(\del \{i_k,i'\},t_0)})\\
            &\quad \quad + \frac{Cn}{N}\sum_{k=1}^{|S|}\frac{1}{N-k}\sum_{i'\notin S_k}\Big\{\big(q_j^{*(t_0)}+q_j^{*(\del S,t_0)}+q_j^{*(\del\{i_k,i'\},t_0)}\big)\big\|\bq^{*(t_0)}-\bq^{*(\del\{i_k,i'\},t_0)}\big\|_1\\
            &\qquad\qquad\qquad\qquad\qquad\qquad+\big|q_j^{*(t_0)} - q_j^{*(\del \{i_k,i'\},t_0)}\big|\Big\}\\
            &\quad\quad+C(q_j^{*(t_0)}+q_j^{*(\del S, t_0)})\big\|\bq^{*(t_0)}-\bq^{*(\del S,t_0)}\big\|_1 + C\big|q_j^{*(t_0)} - q_j^{*(\del S,t_0)}\big| + \frac{C|S|q_j^{*(t_0)}}{N}\\
            &\quad\leq \frac{Cm}{M}\varepsilon(N)+\frac{Cm^{4t_0-3}n}{MN}\varepsilon(N)+\frac{Cm^{4t_0-4}n}{MN}\varepsilon(N)+\frac{Cm^{4t_0-3}}{M}\varepsilon(N) + \frac{Cm^{4t_0-4}}{M}\varepsilon(N) + \frac{Cm}{MN}\\
            &\quad\leq \frac{Cm^{4t_0-3}}{M}\varepsilon(N).
        \end{split}
    \end{equation*}
    Therefore, \eqref{eq:induction_Q_loo_j} and \eqref{eq:induction_Q_loo_l1} in Claim~\ref{claim:induction_Q_loo} with $1\leq t\leq t_0-1$ implies that they also hold with $t=t_0$. By induction, the proof of \eqref{eq:induction_Q_loo_j} and \eqref{eq:induction_Q_loo_l1} in Claim~\ref{claim:induction_Q_loo} is now complete. In addition, reusing our arguments in \eqref{eq:q_l1_err}, we immediately have \eqref{eq:induction_Q_loo}.
\end{proof}
Applying our Claim~\ref{claim:induction_Q_loo} with $S=\{i\}$ and consider the $(t-1)$th iteration, \eqref{eq:induction_Q_loo} immediately implies Lemma~\ref{lem:Q_loo_err}.

\subsubsection{Proof of Lemma~\ref{lem:epsilon4_l2}}
Within this proof, we define 
\begin{align*}
    \bar{h}_j(Z_i,\bZ_{\no i,:}) =& h_j(Z_i,\bZ_{\no i,:}) - h_{j,N-1}(Z_i)\\
    &-\bbE\big[h_j(Z_i,\bZ_{\no i,:})|\bZ_{\no i,:}] + \bbE\big[h_{j,N-1}(Z_i)],
\end{align*}
and let $\delta_j(Z_i) = h_{j,N-1}(Z_i) - h_{j,N}(Z_i)$.
We can then write
\begin{equation*}
    \begin{split}
        \varepsilon_{i,4} = &\til{h}_j(Z_i,\bZ_{\no i,:}) - \bbE\big[\til{h}_j(Z_i,\bZ_{\no i,:})|\bZ_{\no i,:}\big]\\
        =&\bar{h}_j(Z_i,\bZ_{\no i,:}) + \delta_j(Z_i)-\bbE\big[\delta_j(Z_i)\big].
    \end{split}
\end{equation*}
Let $\varepsilon_4^{(1)} = \frac{1}{\sqrt{N}}\sum_{i=1}^N\bar{h}_j(Z_i,\bZ_{\no i,:})$, $\varepsilon_4^{(2)} = \frac{1}{\sqrt{N}}\sum_{i=1}^N\big(\delta_j(Z_i)-\bbE\big[\delta_j(Z_i)\big]\big)$. Hence, $\varepsilon_4=\big(\varepsilon_4^{(1)} + \varepsilon_4^{(2)}\big)/\sigma_j$, and $\bbE\big(\varepsilon_4^2\big)\leq 2\bbE\big(\varepsilon_4^{(1)2}\big)/\sigma_j^2 + 2\bbE\big(\varepsilon_4^{(2)2}\big)/\sigma_j^2$. In the following, we will bound $\bbE\big(\varepsilon_4^{(1)2}\big)$ and $\bbE\big(\varepsilon_4^{(2)2}\big)$ separately.
\paragraph{Bounding $\bbE\big(\varepsilon_4^{(1)2}\big)$.} We first note that $\varepsilon_4^{(1)}$ takes a very similar form to the double-centered $N$-fold cross-validation error in \cite{bayle2020cross}, except that we replace the error function typically used in cross-validation with a LOCO score. Hence, we can apply Theorem 2 in \cite{bayle2020cross}, and immediately arrive at the following:
    $$
    \bbE(\varepsilon_4^{(1)2}) \leq \frac{3(N-1)}{2}\gamma_{loss}(\bar{h}_j),$$
    where $\gamma_{loss}(\bar{h}_j): = \bbE[\bar{h}_j(Z_i,\bZ_{\no i,:}) - \bar{h}_j(Z_i,\bZ_{\no i,:}^{\no l})]^2$ measures the stability of the function $\bar{h}_j$ when one training sample is being replaced; $\bZ^{\no l}$ is the training data we can get after replacing $Z_l$ in $\bZ$ with an i.i.d. copy $Z'$. By the definition of $\bar{h}_j(Z_i,\bZ_{\no i,:})$ in the beginning of this proof, we can write
    \begin{equation}\label{eq:gamma_loss_bnd}
        \begin{split}
            \gamma_{loss}(\bar{h}_j) = &\bbE[\bar{h}_j(Z_i,\bZ_{\no i,:}) - \bar{h}_j(Z_i,\bZ_{\no i,:}^{\no l})]^2\\
            =&\bbE_{\bZ_{\no i,:}}{\rm Var}\big(h_j(Z_i,\bZ_{\no i,:}) - h_j(Z_i,\bZ_{\no i,:}^{\no l})|\bZ_{\no i,:}\big)\\
            \leq&\bbE_{\bZ_{\no i,:}}\bbE_{Z_i}\Big[\big(h_j(Z_i,\bZ_{\no i,:}) - h_j(Z_i,\bZ_{\no i,:}^{\no l})\big)^2|\bZ_{\no i,:}\Big]\\
            = &\bbE\big[h_j(Z_i,\bZ_{\no i,:}) - h_j(Z_i,\bZ_{\no i,:}^{\no l})\big]^2.
        \end{split}
    \end{equation}
    
Recall the definition $h_j(Z_i,\bZ_{\no i,:}) = h_j(Z_i,\bZ_{\no i,:},Q^{*(T)}(\bZ_{\no i,:})$ in Table~\ref{tab:notations_h}. Let $\bQ{i}{l}{t}$ denote the feature sampling probability function given by Algorithm~\ref{algo:loco_adaptive_combinatorial} when applied to $\bZ_{\no \{i,l\},:}$ with $t$ iterations. We then have
\begin{equation*}
    \begin{split}
        &\left|h_j(Z_i,\bZ_{\no i,:}) - h_j(Z_i,\bZ_{\no i,:}^{\no l})\right|\\
        &\quad\leq \left|h_j(Z_i,\bZ_{\no i,:},\bQ{i}{l}{T}) - h_j(Z_i,\bZ_{\no i,:}^{\no l}, \bQ{i}{l}{T})\right|\\
        &\quad\hspace{1em} +\left|h_j(Z_i,\bZ_{\no i,:},\bQ{i}{l}{T}) - h_j(Z_i,\bZ_{\no i,:},Q^{*(T)}(\bZ_{\no i,:})\right|\\
        &\quad\hspace{1em} +\left|h_j(Z_i,\bZ_{\no i,:}^{\no l}, \bQ{i}{l}{T}) - h_j(Z_i,\bZ_{\no i,:},Q^{*(T)}(\bZ^{\no l}_{\no i,:})\right|.
    \end{split}
\end{equation*}
Applying Lemma~\ref{lem:diffQ_err} to the last two terms above, we have 
\begin{equation}\label{eq:gamma_loss_bnd2}
    \begin{split}
        &\left|h_j(Z_i,\bZ_{\no i,:}) - h_j(Z_i,\bZ_{\no i,:}^{\no l})\right|\\
        &\quad\leq \left|h_j(Z_i,\bZ_{\no i,:},\bQ{i}{l}{T}) - h_j(Z_i,\bZ_{\no i,:}^{\no l}, \bQ{i}{l}{T})\right|\\
        &\quad\hspace{1em} +C\|\qvec{\{i,l\}}{T} - \qvec{i}{T}\|_1 + C\|\qvec{\{i,l\}}{T} - \til{\qvec{i}{T}}\|_1,
    \end{split}
\end{equation}
where $\qvec{\{i,l\}}{T},\,\qvec{i}{T},\,\til{\qvec{i}{T}}$ are the feature sampling probability vectors given by Algorithm~\ref{algo:loco_adaptive_combinatorial} when applied to $\bZ_{\no \{i,l\},:}$, $\bZ_{\no i,:}$, $\bZ_{\no i,:}^{\no l}$, respectively. Since $\bZ_{\no i,:}^{\no l}$ has the same distribution as $\bZ_{\no i,:}$ conditioning on $\bZ_{\no \{i,l\},:}$, we know that 
\begin{equation}\label{eq:mean_qstar_leave_one}
    \bbE\|\qvec{\{i,l\}}{T} - \til{\qvec{i}{T}}\|_1^2 = \bbE\|\qvec{\{i,l\}}{T} - \qvec{i}{T}\|_1^2.
\end{equation}

Therefore, \eqref{eq:gamma_loss_bnd} and \eqref{eq:gamma_loss_bnd2} imply that
\begin{equation}\label{eq:gamma_loss_bnd3}
        \gamma_{loss}(\bar{h}_j)\leq \bbE\left|h_j(Z_i,\bZ_{\no i,:},\bQ{i}{l}{T}) - h_j(Z_i,\bZ_{\no i,:}^{\no l}, \bQ{i}{l}{T})\right|^2 +C\bbE\big\|\qvec{\{i,l\}}{T} - \qvec{i}{T}\big\|_1^2.
\end{equation}
The following lemma gives an upper bound for $\bbE\big\|\qvec{\{i,l\}}{T} - \qvec{i}{T}\big\|_1^2$.
\begin{lem}\label{lem:q_loo_l2}
    Suppose that Assumptions~\ref{ass:stb}-\ref{ass:adp_prob_stb} and \ref{ass:minDelta_q_bnd} hold. Then
    $$
    \bbE\big\|\qvec{\{i,l\}}{T} - \qvec{i}{T}\big\|_1^2\leq \frac{Cm^{8(T-1)}(n^2\stb^2(n)+1)}{N^2}.
    $$
\end{lem}
Now, we focus on bounding the first term $\bbE\left|h_j(Z_i,\bZ_{\no i,:},\bQ{i}{l}{T}) - h_j(Z_i,\bZ_{\no i,:}^{\no l}, \bQ{i}{l}{T})\right|^2$ in \eqref{eq:gamma_loss_bnd3}. For simplicity, we abbreviate $\bQ{i}{l}{T}$ as $Q$ within this proof. Let $\mu_F(X^*;\bZ) = \frac{1}{\binom{N}{n}}\sum_{I\subset[N], |I|=n}\mu_{I,F}(X^*;\bZ)$, $\varepsilon(Z^*,\bZ, Q) = Y^* -\sum_{F\subset[M]}\mu_F(X^*;\bZ)Q(F)$, $\delta_j(F;Q) = Q(F)-Q_{\backslash j}(F) = Q(F)\left(\ind(j\in F) - \omega_j(Q)\ind(j\notin F)\right)$, $\omega_j(Q) = \frac{\bbP_{F\sim Q}(j\in F)}{\bbP_{F\sim Q}(j\notin F)} = \frac{q_j}{1-q_j-\prod_{k=1}^M(1-q_k)}$. We can then write
\begin{equation}\label{eq:h_j}
    h_j(Z^*,\bZ,Q) = 2\varepsilon(Z^*,\bZ,Q) \sum_F\mu_F(X^*;\bZ)\delta_j(F;Q) + \left(\sum_F\mu_F(X^*;\bZ)\delta_j(F;Q)\right)^2.
\end{equation}
Due to Assumption~\ref{ass:adp_bounded}, we have
\begin{equation*}
\begin{split}
    &\left|h_j(Z_i,\bZ_{\no i,:},Q) - h_j(Z_i,\bZ_{\no i,:}^{\no l}, Q)\right|\\
    &\quad\leq 2\Big|\sum_F\big(\mu_F(X_i;\bZ_{\no i,:}) - \mu_F(X_i;\bZ_{\no i,:}^{\no l}\big) Q(F)\Big|\Big|\sum_F\mu_F(X_i;\bZ_{\no i,:})\delta_j(F;Q)\Big|\\
    &\quad\hspace{1em}+2\Big|\varepsilon(Z_i,\bZ_{\no i,:}^{\no l},Q)\Big|\sum_F\big(\mu_F(X_i;\bZ_{\no i,:}) - \mu_F(X_i;\bZ_{\no i,:}^{\no l}\big) \delta_j(F;Q)\Big|\\
    &\quad\hspace{1em}+\Big|\sum_F\big(\mu_F(X_i;\bZ_{\no i,:}) - \mu_F(X_i;\bZ_{\no i,:}^{\no l}\big) \delta_j(F;Q)\Big|\Big|\sum_F\mu_F(X_i;\bZ_{\no i,:})\delta_j(F;Q)\Big|\\
    &\quad\hspace{1em}+\Big|\sum_F\big(\mu_F(X_i;\bZ_{\no i,:}) - \mu_F(X_i;\bZ_{\no i,:}^{\no l}\big) \delta_j(F;Q)\Big|\Big|\sum_F\mu_F(X_i;\bZ_{\no i,:}^{\no l})\delta_j(F;Q)\Big|\\
    &\quad \leq C\Big|\sum_F\big(\mu_F(X_i;\bZ_{\no i,:}) - \mu_F(X_i;\bZ_{\no i,:}^{\no l}\big) Q(F)\Big|\\
    &\quad\hspace{1em}+C\Big|\sum_F\big(\mu_F(X_i;\bZ_{\no i,:}) - \mu_F(X_i;\bZ_{\no i,:}^{\no l}\big) \delta_j(F;Q)\Big|\\
    &\quad \leq C\Big|\sum_F\big(\mu_F(X_i;\bZ_{\no i,:}) - \mu_F(X_i;\bZ_{\no i,:}^{\no l}\big) Q(F)\Big|+C\Big|\sum_F\big(\mu_F(X_i;\bZ_{\no i,:}) - \mu_F(X_i;\bZ_{\no i,:}^{\no l}\big) Q_{\no j}(F)\Big|,
\end{split}
\end{equation*}
where the last line is due to the definition of $\delta_j(Q;F)$ above. Furthermore, note that
\begin{equation*}
    \begin{split}
       \mu_F(X_i;\bZ_{\no i,:}) - \mu_F(X_i;\bZ_{\no i,:}^{\no l}) = &\frac{1}{\binom{N-1}{n}}\sum_{I\subset [N]\no i}\mu_{I,F}(X_i;\bZ) -  \frac{1}{\binom{N-1}{n}}\sum_{I\subset [N]\no i}\mu_{I,F}(X_i;\bZ^{\no l})\\
       = &\frac{1}{\binom{N-1}{n}}\sum_{I\subset [N]\no i: I\owns l}\left(\mu_{I,F}(X_i;\bZ) -  \mu_{I,F}(X_i;\bZ^{\no l})\right)\\
       =&\frac{n}{N-1}\bbE_{I_0\sim \cU^{[N]\no \{i,l\}}_{n-1}}\left(H([\bZ_{I_0,F};Z_{l,F}])(X_i) -  H([\bZ_{I_0,F};Z'_{F}])(X_i)\right).
    \end{split}
\end{equation*}
For $F$ with size $|F|\leq 4\max\{m,\log(n)\}$, Assumption~\ref{ass:stb} implies that, conditional on $\bZ_{\no l,:}$, $\mu_F(X_i;\bZ_{\no i,:}) - \mu_F(X_i;\bZ_{\no i,:}^{\no l})$ is sub-Gaussian with parameter $n\stb(n)/(N-1)$.  For every $F$, Assumption~\ref{ass:adp_bounded} bounds the same difference by $Cn/(N-1)$.  Because $Q^{*(\del\{i,l\},T)}$ is independent of $Z_l$ and $Z'$, Lemma~\ref{lem:good_bad_decomposition}, applied to the $Q^{*(\del\{i,l\},T)}$ and $Q^{*(\del\{i,l\},T)}_{\no j}$ averages in \eqref{eq:h_j}, gives the following second-moment bound: 
\begin{equation}\label{eq:mean_h_qstar_stability}
    \begin{split}
        &\bbE\left|h_j(Z_i,\bZ_{\no i,:},\bQ{i}{l}{T}) - h_j(Z_i,\bZ_{\no i,:}^{\no l}, \bQ{i}{l}{T})\right|^2\\
        &\quad\leq \bbE\left[\bbE\left(\left|h_j(Z_i,\bZ_{\no i,:},\bQ{i}{l}{T}) - h_j(Z_i,\bZ_{\no i,:}^{\no l}, \bQ{i}{l}{T})\right|^2|\bZ_{\no l,:}\right)\right]\\
        &\quad\leq \frac{C(n^2\stb^2(n)+1)}{N^2}.
    \end{split}
\end{equation}
Therefore, $\bbE\big(\varepsilon_4^{(1)2}\big)\leq \frac{Cm^{8(T-1)}(n^2\stb^2(n)+1)}{N}$.

\paragraph{Bounding $\bbE\big(\varepsilon_4^{(2)2}\big)$.}
While for $\varepsilon_4^{(2)}$, we note that $\bbE(\varepsilon_4^{(2)2})={\rm Var}(\delta_j(Z^*))$, where $Z^*\sim \cP$ has the same distribution as $Z_i$. By the definition of $\delta_j(Z^*)$, we have
\begin{equation*}
    \begin{split}
      {\rm Var}(\delta_j(Z^*))\leq &\bbE\big(\delta_j^2(Z^*)\big)\\
      \leq &\bbE\Big(h_j(Z^*,\bZ_{\del 1,:},Q^{*(T)}(\bZ_{\del 1,:}))-h_j(Z^*,\bZ,Q^{*(T)}(\bZ))\Big)^2.
    \end{split}
\end{equation*}
In Claim~\ref{claim:induction_Q_loo_expt} in the proof of Lemma~\ref{lem:q_loo_l2}, we have shown that
$$
\bbE(h_j(Z_i,\bZ_{\no i,:},Q^{*(T)}(\bZ))-h_j(Z_i,\bZ_{\del \{i,l\},:},Q^{*(T)}(\bZ_{\no l,:})))^2\leq \frac{Cm^{8(T-1)}(n^2\stb^2(n)+1)}{N^2}.
$$
Applying the same arguments for showing the bound above, we have that 
\begin{equation}\label{eq:h_jN_stability}
    \bbE\Big(h_j(Z^*,\bZ_{\del 1,:},Q^{*(T)}(\bZ_{\del 1,:}))-h_j(Z^*,\bZ,Q^{*(T)}(\bZ))\Big)^2\leq \frac{Cm^{8(T-1)}(n^2\stb^2(n)+1)}{N^2}
\end{equation}

Therefore, 
\begin{equation*}
    \begin{split}
        \bbE(\varepsilon_4^{2})\le&2\bbE(\varepsilon_4^{(1)2})/\sigma_j^2 + 2\bbE(\varepsilon_4^{(2)2})/\sigma_j^2\\
        \leq &\frac{Cm^{8(T-1)}(n^2\stb^2(n)+1)}{\sigma_j^2N}.
    \end{split}
\end{equation*}
The proof of Lemma~\ref{lem:epsilon4_l2} is now complete.
 
\subsubsection{Proof of Lemma~\ref{lem:epsilon5}}
\begin{equation}
\begin{split}
    |\varepsilon_{i,5}| &= |\bbE[h_j(Z_i,\bZ_{\backslash i,:},Q^{*(\no i, T)})|\bZ_{\backslash i,:}]-\bbE[h_j(Z^*,\bZ,Q^{*(T)})|\bZ]|\\
    &=|\bbE[h_j(Z^*,\bZ_{\backslash i,:},Q^{*(\no i, T)})|\bZ]-\bbE[h_j(Z^*,\bZ,Q^{*(T)})|\bZ]|\\
    &\le \bbE[|h_j(Z^*,\bZ_{\backslash i,:},Q^{*(\no i, T)})-h_j(Z^*,\bZ,Q^{*(T)})||\bZ]
\end{split}
\end{equation}
The following lemma shows a bound for $\bbE[|h_j(Z^*,\bZ_{\backslash i,:},Q^{*(\no i, T)})-h_j(Z^*,\bZ,Q^{*(T)})||\bZ]$ in terms of $Q^{*(\no i, T)}$ and $Q^{*(T)}$.
\begin{lem}\label{lem:cond_mean_infK}
With probability at least $1-(MN)^{-c}$, for each iteration $1 \le t \le T$ and for each observation $i$, we have
    \begin{equation}
    \begin{split}
        &\bbE\left[|h_j(Z^*,\bZ_{\backslash i,:},Q^{*(\no i, t)})-h_j(Z^*,\bZ,Q^{*(t)})||\bZ\right]\\
        &\hspace{2em}\leq \frac{C(n\mathrm{stb}(n)\sqrt{\log N + \log M}+1)}{N(N-1)}\sum_{i'\neq i}(q_j^{*(\no i,t)} + q_j^{*(\del \{i,i'\},t)})\\
        &\hspace{3em}+\frac{Cn}{N(N-1)}\sum_{i'\neq i}\left((q_j^{*(\del \{i,i'\}, t)} + q_j^{*(\no i,t)} + q_j^{*(t)})\|\bq^{*(\del \{i,i'\}, t)} - \bq^{*(t)}\|_1 +|q_j^{*(\del \{i,i'\}, t)} - q_j^{*(t)}|\right)\\
        &\hspace{3em}+C(q_j^{*(t)}+q_j^{*(\no i, t)})\|\bq^{*(t)}-\bq^{*(\no i,t)}\|_1 + C\big|q_j^{*(t)}-q_j^{*(\no i, t)}\big|.
    \end{split}
\end{equation}
\end{lem}
 Due to  \eqref{eq:induction_Q_loo} in Claim~\ref{claim:induction_Q_loo}, we have that with probability at least $1-C_0(MN)^{-c}$,
\begin{equation}
    \begin{split}
        &\bbE\left[|h_j(Z^*,\bZ_{\backslash i,:},Q^{*(\no i, t)})-h_j(Z^*,\bZ,Q^{*(t)})||\bZ\right]\\
        &\hspace{2em}\leq \frac{C(n\mathrm{stb}(n)\sqrt{\log N + \log M}+1)}{N}\\
        &\hspace{3em}+Cnm^{4(t-1)}\frac{n\mathrm{stb}(n)\sqrt{\log N + \log M}+1}{N^2}\\
        &\hspace{3em}+Cm^{4(t-1)}\frac{n\mathrm{stb}(n)\sqrt{\log N + \log M}+1}{N}\\
        &\hspace{2em}\le Cm^{4(t-1)}\frac{n\mathrm{stb}(n)\sqrt{\log N + \log M}+1}{N},
    \end{split}
\end{equation}
where we utilize the fact $\big|q_j^{*(t)}-q_j^{*(\no i, t)}\big| \le \|\bq^{*(t)}-\bq^{*(\no i,t)}\|_1$, $|q_j^{*(\del \{i,i'\}, t)} - q_j^{*(t)}|\le \|\bq^{*(\del \{i,i'\}, t)} - \bq^{*(t)}\|_1$, and $q_j^{*(t)}, q_j^{*(\no i, t)}, q_j^{*(\del \{i,i'\}} \le 1$. Therefore,
\begin{equation}
    \begin{split}
        |\varepsilon_5| &\le \frac{1}{\sigma_j\sqrt{N}}\sum_{i=1}^N|\varepsilon_{i,5}|\\
        &\le \frac{Cm^{4(T-1)}(n\mathrm{stb}(n)\sqrt{\log N + \log M}+1)}{\sigma_j\sqrt{N}},
    \end{split}
\end{equation}
with probability at least $1-C_0(MN)^{-c}$.
The proof of Lemma~\ref{lem:epsilon5} is now complete.

\subsection{Proof of Supporting Lemmas}
\begin{proof}[Proof of Lemma~\ref{lem:Delta_leaveSout}]
    Recall the definition of $\bar{\Delta}_j^{*(\del S, t)}$ and $\bar{\Delta}_j^{*(t)}$, we can write
    \begin{equation*}
        \begin{split}
            \bar{\Delta}_j^{*(\del S, t)} - \bar{\Delta}_j^{*(t)}=&\frac{1}{N-|S|}\sum_{l\notin S}h_j(Z_l,\bZ_{\del S\cup\{l\}}, Q^{*(\del S, t)}) - \frac{1}{N}\sum_{l=1}^Nh_j(Z_l,\bZ_{\del l}, Q^{*(t)})\\
            =&\frac{1}{N-|S|}\sum_{l\notin S}\left(h_j(Z_l,\bZ_{\del S\cup\{l\}}, Q^{*(\del S, t)})-h_j(Z_l,\bZ_{\del l}, Q^{*(t)})\right)\\
            &+\frac{|S|}{N(N-|S|)}\sum_{l\notin S}h_j(Z_l,\bZ_{\del l}, Q^{*(t)}) - \frac{1}{N}\sum_{l \in S}h_j(Z_l,\bZ_{\del l}, Q^{*(t)}).
        \end{split}
    \end{equation*}
    As has been shown in \eqref{eq:h_j}, the feature importance score $h_j(Z,\bZ,Q)$ has the following decomposition:
    $$
    h_j(Z^*,\bZ,Q) = 2\varepsilon(Z^*,\bZ,Q)\sum_{F}\mu_F(X^*;\bZ)\delta_j(F;Q) + \Big(\sum_{F}\mu_F(X^*;\bZ)\delta_j(F;Q)\Big)^2,
    $$
    where $\delta_j(F,Q) = Q(F)(\ind(j\in F) - \omega_j(Q)\ind(j\notin F))$,  and $\omega_j(Q) = \frac{q_j}{1-q_j-\prod_{k=1}^M(1-q_k)}\leq \frac{q_j}{1-\delta-e^{-m}}\leq Cq_j$. Here in the bound for $\omega_j(Q)$, we have applied the fact that $\sum_kq_k=m$ in the second inequality and we utilized Assumption~\ref{ass:adp_prob_stb} in the last inequality. Therefore, Assumption~\ref{ass:adp_bounded} implies that 
    $$
     \big|h_j(Z^*,\bZ,Q)\big| \leq C\sum_F\big|\delta_j(F,Q)\big| + \Big(\sum_F\big|\delta_j(F,Q)\big| \Big)^2\leq   Cq_j,$$
     where the last inequality is due to the fact that $\sum_F\big|\delta_j(F,Q)\big| \leq \frac{q_j}{1-e^{-m}} + \omega_j(Q)\leq Cq_j.$ Therefore, we can further bound $\bar{\Delta}_j^{*(\del S, t)} - \bar{\Delta}_j^{*(t)}$ as follows:
     \begin{equation}\label{eq:Delta_leaveSout}
         \begin{split}
             \Big|\bar{\Delta}_j^{*(\del S, t)} - \bar{\Delta}_j^{*(t)}\Big|\leq &\frac{1}{N-|S|}\sum_{l\notin S}\left|h_j(Z_l,\bZ_{\del S\cup\{l\}}, Q^{*(\del S, t)})-h_j(Z_l,\bZ_{\del l}, Q^{*(t)})\right|+ \frac{C|S|q_j^{*(t)}}{N}.
         \end{split}
     \end{equation}
     In the following, we will focus on bounding $\left|h_j(Z_l,\bZ_{\del S\cup\{l\}}, Q^{*(\del S, t)})-h_j(Z_l,\bZ_{\del l}, Q^{*(t)})\right|$ with high probability. Invoking \eqref{eq:h_j} again, we have
     \begin{equation}\label{eq:h_leaveSout_err}
         \begin{split}
             &\left|h_j(Z_l,\bZ_{\del S\cup\{l\}}, Q^{*(\del S, t)})-h_j(Z_l,\bZ_{\del l}, Q^{*(t)})\right|\\
             &\hspace{2em}\leq 2\Big|\sum_F\big(\mu_F(X_l;\bZ_{\del l}) Q^{*(t)}(F) - \mu_F(X_l;\bZ_{\del S\cup\{l\}})Q^{*(\del S,t)}(F))\Big| \Big|\sum_F\mu_F(X_l;\bZ_{\del S\cup\{l\}})\delta_j(F,Q^{*(\del S, t)})\Big|\\
             &\hspace{3em}+C(1+q_j^{*(t)}+q_j^{*(\no S,t)})\Big|\sum_F\big(\mu_F(X_l;\bZ_{\del l})\delta_j(F;Q^{*(t)}) - \mu_F(X_l;\bZ_{\del S\cup\{l\}})\delta_j(F;Q^{*(\del S,t)})\Big|\\
            &\hspace{2em} \leq Cq_j^{*(\del S,t)}\Big|\sum_F\big(\mu_F(X_l;\bZ_{\del l}) Q^{*(t)}(F) - \mu_F(X_l;\bZ_{\del S\cup\{l\}})Q^{*(\del S,t)}(F))\Big| \\
            &\hspace{3em} + C\Big|\sum_F\big(\mu_F(X_l;\bZ_{\del l}) \delta_j(F;Q^{*(t)}) - \mu_F(X_l;\bZ_{\del S\cup\{l\}})\delta_j(F;Q^{*(\del S,t)})\Big| \\
            &\hspace{2em} \leq Cq_j^{*(\del S,t)}\Big|\sum_F\mu_F(X_l;\bZ_{\del S\cup\{l\}}) \big(Q^{*(t)}(F) -Q^{*(\del S,t)}(F)\big)\Big| \\
            &\hspace{3em} + C\Big|\sum_F\mu_F(X_l;\bZ_{\del S\cup\{l\}})\big(\delta_j(F;Q^{*(t)}) - \delta_j(F;Q^{*(\del S,t)}\big)\Big|\\
             &\hspace{3em} + Cq_j^{*(\del S,t)}\Big|\sum_F\big(\mu_F(X_l;\bZ_{\del l})  - \mu_F(X_l;\bZ_{\del S\cup\{l\}})\big)Q^{*(t)}(F)\Big| \\
            &\hspace{3em} + C\Big|\sum_F\big(\mu_F(X_l;\bZ_{\del l}) - \mu_F(X_l;\bZ_{\del S\cup\{l\}})\big)\delta_j(F;Q^{*(t)})\Big|,
         \end{split}
     \end{equation}
     where the first and second inequalities are due to Assumption~\ref{ass:adp_bounded} and the fact that $\sum_F|\delta_j(F;Q^{*(\del S,t)})|\leq Cq_j^{*(\del S,t)}$. The following lemma helps us establish upper bounds for the first two terms above.
     \begin{lem}\label{lem:expct_diff_q}
         Consider a bounded function $g:2^{[M]}\rightarrow \bbR$ which satisfies $|g(F)|\leq 1$. In addition, suppose that we are given two feature sampling probability vectors $\bq^{(1)},\bq^{(2)}\in [0,1]^M$ satisfying $\sum_{j=1}^Mq_j^{(1)} = \sum_{j=1}^Mq_j^{(2)}=m $ and $\|\bq^{(1)}-\bq^{(2)}\|_{\infty}\leq \frac{1}{2}$. We then have
         \begin{equation*}
             \begin{split}
                 \Big|\sum_F g(F)\big(Q_{q^{(1)}}(F) - Q_{q^{(2)}}(F)\big)\Big|\leq &C\|\bq^{(1)}-\bq^{(2)}\|_1,\\
                 \Big|\sum_F g(F)\big(\delta_j(F;Q_{q^{(1)}}) - \delta_j(F;Q_{q^{(2)}})\big)\Big|\leq &C(q_j^{(1)}+q_j^{(2)})\|\bq^{(1)}-\bq^{(2)}\|_1 + C\big|q_j^{(1)}-q_j^{(2)}\big|,
             \end{split}
         \end{equation*}
     \end{lem}
     With Lemma~\ref{lem:expct_diff_q}, we can now bound the first two terms in \eqref{eq:h_leaveSout_err} via the difference between $\bq^{*(t)}$ and $\bq^{*(\del S,t)}$:
     \begin{equation}\label{eq:h_leaveSout_12}
         \begin{split}
            & q_j^{*(\del S,t)}\Big|\sum_F\mu_F(X_l;\bZ_{\del S\cup\{l\}}) \big(Q^{*(t)}(F) -Q^{*(\del S,t)}(F)\big)\Big|\leq Cq_j^{*(\del S,t)}\|\bq^{*(t)}-\bq^{*(\del S,t)}\|_1,\\
             &\Big|\sum_F\mu_F(X_l;\bZ_{\del S\cup\{l\}})\big(\delta_j(F;Q^{*(t)}) - \delta_j(F;Q^{*(\del S,t)}\big)\Big|\\
            &\hspace{2em} \leq C(q_j^{*(t)}+q_j^{*(\del S, t)})\|\bq^{*(t)}-\bq^{*(\del S,t)}\|_1 + C\big|q_j^{*(t)}-q_j^{*(\del S, t)}\big|.
         \end{split}
     \end{equation}
    Now, it remains to bound the last two terms in \eqref{eq:h_leaveSout_err}, which stem from the difference in the trained predictors due to leaving $S$ out. The following lemma characterizes $\mu_F(X_l;\bZ_{\del l}) - \mu_F(X_l;\bZ_{\del S\cup\{l\}})$.
\begin{lem}\label{lem:mu_leaveSout_err}
    For any training data $\bZ$, test feature $X^*$, and subset $S=\{i_1,\dots,i_{|S|}\}\subset[N]$, we have
    \begin{equation*}
    \begin{split}
        &\mu_F(X^*;\bZ)-\mu_F(X^*;\bZ_{\del S}) \\
        =&\sum_{k=1}^{|S|}\frac{n}{(N-k+1)(N-k)}\sum_{i'\notin S_k}\Big\{\bbE_{I_0\sim \cU_{n-1}^{[N]\del (S_k\cup\{i'\})}}\big[\mu_{I_0\cup \{i_k\},F}(X^*;\bZ) - \mu_{I_0\cup\{i'\},F}(X^*;\bZ)\big]\Big\},
    \end{split}
    \end{equation*}
    where $S_k = \{i_1,\dots,i_k\}$.
\end{lem}
    Denote $\bbE_{I_0\sim \cU_{n-1}^{[N]\del (S_k\cup\{l,i'\})}}\big[\mu_{I_0\cup \{i_k\},F}(X_l;\bZ) - \mu_{I_0\cup\{i'\},F}(X_l;\bZ)\big]$ by $g_F(X_l, \bZ_{\del(S_k\cup\{l,i'\})},Z_{i_k},Z_{i'})$, which is bounded by Assumption~\ref{ass:adp_bounded}. Applying Lemma~\ref{lem:mu_leaveSout_err} to $\mu_F(X_l;\bZ_{\del l}) - \mu_F(X_l;\bZ_{\del S\cup\{l\}})$, we have
    \begin{equation*}
        \begin{split}
            &\left|\sum_F\big(\mu_F(X_l;\bZ_{\del l}) - \mu_F(X_l;\bZ_{\del S\cup\{l\}})\big)Q^{*(t)}(F)\right|\\
            &\hspace{2em}\leq \sum_{k=1}^{|S|}\frac{n}{(N-k)(N-k-1)}\sum_{i'\notin S_k\cup\{l\}}\left|\bbE_{F\sim Q^{*(t)}}g_F(X_l, \bZ_{\del(S_k\cup\{l,i'\})},Z_{i_k},Z_{i'})\right|\\
            &\hspace{2em}\leq \frac{Cn}{N^2}\sum_{k=1}^{|S|}\sum_{i'\notin S_k\cup\{l\}}\left|\bbE_{F\sim Q^{*(\del \{i_k,i'\},t)}}g_F(X_l, \bZ_{\del(S_k\cup\{l,i'\})},Z_{i_k},Z_{i'})\right|\\
            &\hspace{3em}+\frac{Cn}{N}\sum_{k=1}^{|S|}\frac{1}{N-k}\sum_{i'\notin S_k\cup\{l\}}\left|\sum_Fg_F(X_l, \bZ_{\del(S_k\cup\{l,i'\})},Z_{i_k},Z_{i'})\left(Q^{*(\del \{i_k,i'\},t)}(F)- Q^{*(t)}(F)\right)\right|\\
            &\hspace{2em}\leq \frac{Cn}{N^2}\sum_{k=1}^{|S|}\sum_{i'\notin S_k\cup\{l\}}\left|\bbE_{F\sim Q^{*(\del \{i_k,i'\},t)}}g_F(X_l, \bZ_{\del(S_k\cup\{l,i'\})},Z_{i_k},Z_{i'})\right|\\
            &\hspace{3em}+\frac{Cn}{N}\sum_{k=1}^{|S|}\frac{1}{N-k}\sum_{i'\notin S_k\cup\{l\}}\Big\|\bq^{*(\del \{i_k,i'\},t)} - \bq^{*(t)}\Big\|_1,
        \end{split}
    \end{equation*}
    where the second inequality is due to that $|S|\leq C$, the third inequality is due to Lemma~\ref{lem:expct_diff_q}. The following lemma is helpful for bounding the first term in the inequality.
    \begin{lem}\label{lem:good_bad_decomposition}
    	Suppose Assumption~\ref{ass:adp_prob_stb} holds.  Let $\mathcal H$ be a $\sigma$-field and let $\bq=(q_1,\ldots,q_M)$ be $\mathcal H$-measurable with $\sum_{j=1}^M q_j=m$ and $0\leq q_j\leq\delta$ for every $j$, almost surely.  Let $R$ be any of $Q_{\bq}$, $Q_{\bq,j}$, or $Q_{\bq,\no j}$, whenever the corresponding conditional law is well defined. Suppose that, conditionally on $\mathcal H$, for some $s,b>0$,
    	\begin{equation*}
    			\|V_F\|_{\psi_2}\leq s
			\quad\text{for every }F\text{ such that }|F|\leq4\max\{m,\log(n)\},
    			\qquad
    			|V_F|\leq b \quad\text{for every }F\subset[M].
    	\end{equation*}
    	Then there exist constants $c,C>0$ such that, for every $x>0$, almost surely,
    	\begin{equation*}
    			\bbP\left(\left|\bbE_{F\sim R}V_F\right|>
    			C\left\{s\sqrt{x}+\frac{b}{n}\right\}\,\middle|\,\mathcal H\right)
    			\leq2e^{-cx},
    			\qquad
    			\bbE\left(\left|\bbE_{F\sim R}V_F\right|^2\,\middle|\,\mathcal H\right)
    			\leq C\left\{s^2+\frac{b^2}{n^2}\right\}.
    	\end{equation*}
    \end{lem}
    We then apply Lemma~\ref{lem:good_bad_decomposition} with $V_F$ being $g_F(X_l, \bZ_{\del(S_k\cup\{l,i'\})},Z_{i_k},Z_{i'})$ and $\mathcal{H}$ being the $\sigma$-field associated with $\bZ_{\del(i',i_k)}$. Assumptions~\ref{ass:stb} and \ref{ass:adp_bounded} imply the conditions of Lemma~\ref{lem:good_bad_decomposition} with $b\leq C$. Therefore, with probability at least $1-\exp\{-c(\log M+\log N)\}$, the following holds:
    \begin{equation}\label{eq:h_leaveSout_err3}
        \begin{split}
            &\left|\sum_F\big(\mu_F(X_l;\bZ_{\del l}) - \mu_F(X_l;\bZ_{\del S\cup\{l\}})\big)Q^{*(t)}(F)\right|\\
            &\hspace{2em}\leq \frac{Cn\stb(n)}{N}\sqrt{\log M+\log N}+\frac{C}{N}+\frac{Cn}{N}\sum_{k=1}^{|S|}\frac{1}{N-k}\sum_{i'\notin S_k\cup\{l\}}\Big\|\bq^{*(\del \{i_k,i'\},t)} - \bq^{*(t)}\Big\|_1.
        \end{split}
    \end{equation}
    While for the last term in \eqref{eq:h_leaveSout_err}, we use a similar argument to bounding the third term. Formally, we can write
    \begin{equation*}
        \begin{split}
            &\left|\sum_F\big(\mu_F(X_l;\bZ_{\del l}) - \mu_F(X_l;\bZ_{\del S\cup\{l\}})\big)\delta_j(F;Q^{*(t)})\right|\\
            &\hspace{2em}\leq \frac{Cn}{N}\sum_{k=1}^{|S|}\frac{1}{N-k}\sum_{i'\notin S_k\cup\{l\}}\left|\sum_F \delta_j(F;Q^{*(\del\{i_k,i'\},t)})g_F(X_l, \bZ_{\del(S_k\cup\{l,i'\})},Z_{i_k},Z_{i'})\right|\\
           &\hspace{3em}+\frac{Cn}{N}\sum_{k=1}^{|S|}\frac{1}{N-k}\sum_{i'\notin S_k\cup\{l\}}\left|\sum_Fg_F(X_l, \bZ_{\del(S_k\cup\{l,i'\})},Z_{i_k},Z_{i'})\left(\delta_j(F;Q^{*(\del\{i_k,i'\},t)})- \delta_j(F;Q^{*(t)})\right)\right|\\
            &\hspace{2em}\leq \frac{Cn}{N}\sum_{k=1}^{|S|}\frac{1}{N-k}\sum_{i'\notin S_k\cup\{l\}}\left|\sum_F \delta_j(F;Q^{*(\del\{i_k,i'\},t)})g_F(X_l, \bZ_{\del(S_k\cup\{l,i'\})},Z_{i_k},Z_{i'})\right|\\
           &\hspace{3em}+\frac{Cn}{N}\sum_{k=1}^{|S|}\frac{1}{N-k}\sum_{i'\notin S_k\cup\{l\}}\left((q_j^{*(\del \{i_k,i'\}, t)} + q_j^{*(t)})\|\bq^{*(\del \{i_k,i'\}, t)} - \bq^{*(t)}\|_1 +|q_j^{*(\del \{i_k,i'\}, t)} - q_j^{*(t)}|\right),
        \end{split}
    \end{equation*}
where the last line is due to Lemma~\ref{lem:expct_diff_q}. Furthermore, note that for any feature sampling distribution $Q$ and any function $g(F)$ of the feature subset $F$, we can write 
\begin{equation}\label{eq:delta_expt}
    \begin{split}
        \left|\sum_F \delta_j(F;Q)g(F)\right| \leq &\left|\sum_Fg(F)\ind(j\in F)Q(F)\right| + \omega_j(Q)\left|\sum_Fg(F)\ind(j\notin F)Q(F)\right|\\
        \leq &Cq_j\Big|\bbE_{F\sim Q_j}g(F)\Big| + Cq_j\Big|\bbE_{F\sim Q_{\no j}}g(F)\Big|,
    \end{split}
\end{equation}
where $Q_j$ and $Q_{\no j}$ denote the distributions of $F\sim Q$ conditional on $j\in F$ and $j\notin F$, respectively. Therefore, we can still bound $\left|\sum_F \delta_j(F;Q^{*(\del\{i_k,i'\},t)})g_F(X_l, \bZ_{\del(S_k\cup\{l,i'\})},Z_{i_k},Z_{i'})\right|$ by viewing it as $q_j^{*(\del \{i_k,i'\},t)}$ times a sum of two expectations of $g_F$, taken over distributions $Q^{*(\del\{i_k,i'\},t)}_j$ and $Q^{*(\del\{i_k,i'\},t)}_{\no j}$. Therefore, the arguments above for bounding the third term in \eqref{eq:h_leaveSout_err} is still applicable here, and we have
 \begin{equation}\label{eq:h_leaveSout_err4}
        \begin{split}
            &\left|\sum_F\big(\mu_F(X_l;\bZ_{\del l}) - \mu_F(X_l;\bZ_{\del S\cup\{l\}})\big)\delta_j(F;Q^{*(t)})\right|\\
            &\hspace{2em}\leq \frac{Cn\stb(n)}{N}\sqrt{\log M+\log N}\sum_{k=1}^{|S|}\frac{1}{N-k}\sum_{i'\notin S_k\cup\{l\}} q_j^{*(\del \{i_k,i'\},t)}\\
           &\hspace{3em}+\frac{Cn}{N}\sum_{k=1}^{|S|}\frac{1}{N-k}\sum_{i'\notin S_k\cup\{l\}}\left((q_j^{*(\del \{i_k,i'\}, t)} + q_j^{*(t)})\|\bq^{*(\del \{i_k,i'\}, t)} - \bq^{*(t)}\|_1 +|q_j^{*(\del \{i_k,i'\}, t)} - q_j^{*(t)}|\right).
        \end{split}
    \end{equation}
Therefore, combining \eqref{eq:Delta_leaveSout},  \eqref{eq:h_leaveSout_err},  \eqref{eq:h_leaveSout_12}, \eqref{eq:h_leaveSout_err3} and \eqref{eq:h_leaveSout_err4}, we arrive at the results stated in Lemma~\ref{lem:Delta_leaveSout}. Note that we have taken a union bound over all subsets $S$ of bounded size and $1\leq t\leq T\leq MN$.
    \end{proof}
\begin{proof}[Proof of Lemma~\ref{lem:q_loo_l2}]
    We follow similar arguments to the proof of Lemma~\ref{lem:q_err_bnd}, except that we consider a second-moment bound instead of a probabilistic bound. 

Similar to Claim~\ref{claim:induction_Q_loo}, we can show a moment bound for $\bar{\Delta}_j^{*(\del S,t)} - \bar{\Delta}_j^{*(t)}$ and $\|\bq^{*(\del S,t)}-\bq^{*(t)}\Big\|_1$ for all $t$ via induction. Here, $\bar{\Delta}_j^{*(\del S,t)}$ and $q^{*(\del S,t)}$ were defined before Assumption~\ref{ass:minDelta_q_bnd}.

\begin{claim}\label{claim:induction_Q_loo_expt}
Suppose that Assumptions~\ref{ass:stb}-\ref{ass:adp_prob_stb} and \ref{ass:minDelta_q_bnd} hold. $\forall 1\leq t\leq T$, $\forall\,S\subset[N]$ with size $|S|\leq 2$, we have
\begin{equation}\label{eq:induction_Dsum_expt}
        \bbE\Big\|\bar{\Delta}^{*(\del S,t)}-\bar{\Delta}^{*(t)}\Big\|_1^2\leq Cm^{8t-6}\frac{n^2\stb^2(n)+1}{N^2}.
\end{equation}
For any $1\leq j\leq M$, 
\begin{equation}\label{eq:induction_D_expt}
    \bbE\Big|\bar{\Delta}_j^{*(\del S,t)}-\bar{\Delta}_j^{*(t)}\Big|^2\leq Cm^{8t-8}\frac{n^2\stb^2(n)+1}{N^2}.
\end{equation}
 If there exist $S'\subset[N]\text{ with size }|S'|\leq 2$, and $j=j_{\min}^{*(\del S',t')}\text{ for some }t\leq t'\leq T$, then we have
\begin{equation}\label{eq:induction_Dmin_expt}
    \bbE\Big|\bar{\Delta}_j^{*(\del S,t)}-\bar{\Delta}_j^{*(t)}\Big|^2\leq \frac{Cm^{8t-6}}{M^2}\frac{n^2\stb^2(n)+1}{N^2}.
\end{equation}
Furthermore, 
\begin{equation}\label{eq:induction_q_expt}
    \bbE\Big\|\bq^{*(\del S,t+1)}-\bq^{*(t+1)}\Big\|_1^2\leq Cm^{8t}\frac{n^2\stb^2(n)+1}{N^2}.
\end{equation}
\end{claim}

\begin{proof}[Proof of Claim~\ref{claim:induction_Q_loo_expt}]
    We first note that, in the proof of Lemma~\ref{lem:Delta_leaveSout}, we already established a deterministic bound for $|\bar{\Delta}_j^{*(\del S,t)} - \bar{\Delta}_j^{*(t)}|$: $\forall S\subset[N]$ such that $|S|\leq C$,
    \begin{equation*}
        \begin{split}
            &\Big|\bar{\Delta}_j^{*(\del S,t)}-\bar{\Delta}_j^{*(t)}\Big|\\
            &\quad \leq \frac{Cn}{N^3}\sum_{l\notin S}\sum_{k=1}^{|S|}\sum_{i'\notin S_k,i'\neq l} \Big|q_j^{*(\del S, t)}\bbE_{F\sim Q^{*(\del \{i_k,i'\},t)}}g_F(X_l,\bZ_{\no (S_k\cup \{l,i'\}),:},Z_{i_k},Z_{i'})\Big|\\
            &\quad \hspace{1em} + \frac{Cn}{N^3}\sum_{l\notin S}\sum_{k=1}^{|S|}\sum_{i'\notin S_k,i'\neq l} \Big|q_j^{*(\no \{i_k,i'\}, t)}\bbE_{F\sim Q_j^{*(\del \{i_k,i'\},t)}}g_F(X_l,\bZ_{\no (S_k\cup \{l,i'\}),:},Z_{i_k},Z_{i'})\Big|\\
             &\quad \hspace{1em} + \frac{Cn}{N^3}\sum_{l\notin S}\sum_{k=1}^{|S|}\sum_{i'\notin S_k,i'\neq l} \Big|q_j^{*(\del \{i_k,i'\}, t)}\bbE_{F\sim Q_{\no j}^{*(\del \{i_k,i'\},t)}}g_F(X_l,\bZ_{\no (S_k\cup \{l,i'\}),:},Z_{i_k},Z_{i'})\Big|\\
            &\quad \quad + \frac{Cn}{N}\sum_{k=1}^{|S|}\frac{1}{N-k}\sum_{i'\notin S_k}\Big\{\big(q_j^{*(t)}+q_j^{*(\del S,t)}+q_j^{*(\del\{i_k,i'\},t)}\big)\big\|\bq^{*(t)}-\bq^{*(\del\{i_k,i'\},t)}\big\|_1\\
            &\qquad\qquad\qquad\qquad\qquad\qquad+\big|q_j^{*(t)} - q_j^{*(\del \{i_k,i'\},t)}\big|\Big\}\\
            &\quad\quad+C(q_j^{*(t)}+q_j^{*(\del S, t)})\big\|\bq^{*(t)}-\bq^{*(\del S,t)}\big\|_1 + C\big|q_j^{*(t)} - q_j^{*(\del S,t)}\big| + \frac{Cq_j^{*(t)}}{N},
        \end{split}
    \end{equation*}
    where $g_F(X_l,\bZ_{\no (S_k\cup \{l,i'\}),:},Z_{i_k},Z_{i'})$ was as defined in the proof of Lemma~\ref{lem:Delta_leaveSout}. Therefore, we can further write
    \begin{equation}\label{eq:Dsum_bnd}
        \begin{split}
            &\Big\|\bar{\Delta}^{*(\del S,t)}-\bar{\Delta}^{*(t)}\Big\|_1\\
            &\quad \leq \frac{Cmn}{N^3}\sum_{l\notin S}\sum_{k=1}^{|S|}\sum_{i'\notin S_k,i'\neq l} \Big|\bbE_{F\sim Q^{*(\del \{i_k,i'\},t)}}g_F(X_l,\bZ_{\no (S_k\cup \{l,i'\}),:},Z_{i_k},Z_{i'})\Big|\\
            &\quad \hspace{1em} + \frac{Cmn}{N^3}\sum_{l\notin S}\sum_{k=1}^{|S|}\sum_{i'\notin S_k,i'\neq l} \Big|\bbE_{F\sim Q_j^{*(\del \{i_k,i'\},t)}}g_F(X_l,\bZ_{\no (S_k\cup \{l,i'\}),:},Z_{i_k},Z_{i'})\Big|\\
             &\quad \hspace{1em} + \frac{Cmn}{N^3}\sum_{l\notin S}\sum_{k=1}^{|S|}\sum_{i'\notin S_k,i'\neq l} \Big|\bbE_{F\sim Q_{\no j}^{*(\del \{i_k,i'\},t)}}g_F(X_l,\bZ_{\no (S_k\cup \{l,i'\}),:},Z_{i_k},Z_{i'})\Big|\\
            &\quad \quad + \frac{Cmn}{N}\sum_{k=1}^{|S|}\frac{1}{N-k}\sum_{i'\notin S_k}\big\|\bq^{*(t)}-\bq^{*(\del\{i_k,i'\},t)}\big\|_1+Cm\big\|\bq^{*(t)}-\bq^{*(\del S,t)}\big\|_1 + \frac{Cm}{N},
        \end{split}
    \end{equation}
    and for any $1\leq j\leq M$, we have
    \begin{equation}\label{eq:Dj_bnd}
        \begin{split}
            &\Big|\bar{\Delta}_j^{*(\del S,t)}-\bar{\Delta}_j^{*(t)}\Big|\\
            &\quad \leq \frac{Cn}{N^3}\sum_{l\notin S}\sum_{k=1}^{|S|}\sum_{i'\notin S_k,i'\neq l} \Big|\bbE_{F\sim Q^{*(\del \{i_k,i'\},t)}}g_F(X_l,\bZ_{\no (S_k\cup \{l,i'\}),:},Z_{i_k},Z_{i'})\Big|\\
            &\quad \hspace{1em} + \frac{Cn}{N^3}\sum_{l\notin S}\sum_{k=1}^{|S|}\sum_{i'\notin S_k,i'\neq l} \Big|\bbE_{F\sim Q_j^{*(\del \{i_k,i'\},t)}}g_F(X_l,\bZ_{\no (S_k\cup \{l,i'\}),:},Z_{i_k},Z_{i'})\Big|\\
             &\quad \hspace{1em} + \frac{Cn}{N^3}\sum_{l\notin S}\sum_{k=1}^{|S|}\sum_{i'\notin S_k,i'\neq l} \Big|\bbE_{F\sim Q_{\no j}^{*(\del \{i_k,i'\},t)}}g_F(X_l,\bZ_{\no (S_k\cup \{l,i'\}),:},Z_{i_k},Z_{i'})\Big|\\
            &\quad \quad + \frac{Cn}{N}\sum_{k=1}^{|S|}\frac{1}{N-k}\sum_{i'\notin S_k}\big\|\bq^{*(t)}-\bq^{*(\del\{i_k,i'\},t)}\big\|_1+C\big\|\bq^{*(t)}-\bq^{*(\del S,t)}\big\|_1 + \frac{C}{N}.
        \end{split}
    \end{equation}
    If there exist $S'\subset[N]\text{ with size }|S'|\leq 2$, and $j=j_{\min}^{*(\del S',t')}\text{ for some }t\leq t'\leq T$, then by Assumption~\ref{ass:minDelta_q_bnd}, we have
   \begin{equation}\label{eq:Dmin_bnd}
        \begin{split}
            &\Big|\bar{\Delta}_j^{*(\del S,t)}-\bar{\Delta}_j^{*(t)}\Big|\\
            &\quad \leq \frac{Cmn}{MN^3}\sum_{l\notin S}\sum_{k=1}^{|S|}\sum_{i'\notin S_k,i'\neq l} \Big|\bbE_{F\sim Q^{*(\del \{i_k,i'\},t)}}g_F(X_l,\bZ_{\no (S_k\cup \{l,i'\}),:},Z_{i_k},Z_{i'})\Big|\\
            &\quad \hspace{1em} + \frac{Cmn}{MN^3}\sum_{l\notin S}\sum_{k=1}^{|S|}\sum_{i'\notin S_k,i'\neq l} \Big|\bbE_{F\sim Q_j^{*(\del \{i_k,i'\},t)}}g_F(X_l,\bZ_{\no (S_k\cup \{l,i'\}),:},Z_{i_k},Z_{i'})\Big|\\
             &\quad \hspace{1em} + \frac{Cmn}{MN^3}\sum_{l\notin S}\sum_{k=1}^{|S|}\sum_{i'\notin S_k,i'\neq l} \Big|\bbE_{F\sim Q_{\no j}^{*(\del \{i_k,i'\},t)}}g_F(X_l,\bZ_{\no (S_k\cup \{l,i'\}),:},Z_{i_k},Z_{i'})\Big|\\
            &\quad \quad + \frac{Cn}{N}\sum_{k=1}^{|S|}\frac{1}{N-k}\sum_{i'\notin S_k}\Big\{\frac{m}{M}\big\|\bq^{*(t)}-\bq^{*(\del\{i_k,i'\},t)}\big\|_1+\big|q_j^{*(t)} - q_j^{*(\del \{i_k,i'\},t)}\big|\Big\}\\
            &\quad\quad+\frac{Cm}{M}\big\|\bq^{*(t)}-\bq^{*(\del S,t)}\big\|_1 + C\big|q_j^{*(t)} - q_j^{*(\del S,t)}\big| + \frac{Cm}{MN}.
        \end{split}
    \end{equation}

    We invoke Lemma~\ref{lem:good_bad_decomposition} again with $V_F$ being $g_F(X_l,\bZ_{\no (S_k\cup \{l,i'\}),:},Z_{i_k},Z_{i'})$ and $\mathcal{H}$ being the $\sigma$-field associated with $\bZ_{\no\{i_k,i'\},:}$. Assumptions~\ref{ass:stb}, \ref{ass:adp_bounded} and Lemma~\ref{lem:good_bad_decomposition} together gives:
    \begin{equation*}
        \begin{split}
            \bbE\Big|\bbE_{F\sim Q^{*(\del \{i_k,i'\},t)}}g_F(X_l,\bZ_{\no (S_k\cup \{l,i'\}),:},Z_{i_k},Z_{i'})\Big|^2\leq \stb^2(n)+\frac{C}{n^2},\\
            \bbE\Big|\bbE_{F\sim Q_j^{*(\del \{i_k,i'\},t)}}g_F(X_l,\bZ_{\no (S_k\cup \{l,i'\}),:},Z_{i_k},Z_{i'})\Big|^2\leq \stb^2(n)+\frac{C}{n^2},\\
            \bbE\Big|\bbE_{F\sim Q_{\no j}^{*(\del \{i_k,i'\},t)}}g_F(X_l,\bZ_{\no (S_k\cup \{l,i'\}),:},Z_{i_k},Z_{i'})\Big|^2\leq \stb^2(n)+\frac{C}{n^2}.
        \end{split}
    \end{equation*}
    Therefore, \eqref{eq:Dsum_bnd}-\eqref{eq:Dmin_bnd} imply that
    \begin{equation}\label{eq:Dsum_bnd2}
        \begin{split}
            \bbE\Big\|\bar{\Delta}^{*(\del S,t)}-\bar{\Delta}^{*(t)}\Big\|_1^2\leq &\frac{Cm^2(n^2\stb^2(n)+1)}{N^2} +Cm^2\bbE\big\|\bq^{*(t)}-\bq^{*(\del S,t)}\big\|_1^2\\
            &+ \frac{Cm^2n^2}{N^2}\sum_{k=1}^{|S|}\frac{1}{N-k}\sum_{i'\notin S_k}\bbE\big\|\bq^{*(t)}-\bq^{*(\del\{i_k,i'\},t)}\big\|_1^2,
        \end{split}
    \end{equation}
    \begin{equation}\label{eq:Dj_bnd2}
        \begin{split}
            \bbE\Big|\bar{\Delta}_j^{*(\del S,t)}-\bar{\Delta}_j^{*(t)}\Big|^2\leq &\frac{C(n^2\stb^2(n)+1)}{N^2} +C\bbE\big\|\bq^{*(t)}-\bq^{*(\del S,t)}\big\|_1^2\\
            &+ \frac{Cn^2}{N^2}\sum_{k=1}^{|S|}\frac{1}{N-k}\sum_{i'\notin S_k}\bbE\big\|\bq^{*(t)}-\bq^{*(\del\{i_k,i'\},t)}\big\|_1^2,
        \end{split}
    \end{equation}
    and if there exist $S'\subset[N]\text{ with size }|S'|\leq 2$, and $j=j_{\min}^{*(\del S',t')}\text{ for some }t\leq t'\leq T$, we have
    \begin{equation}\label{eq:Dmin_bnd2}
        \begin{split}
            &\bbE\Big|\bar{\Delta}_j^{*(\del S,t)}-\bar{\Delta}_j^{*(t)}\Big|^2\\
            &\quad \leq \frac{Cm^2(n^2\stb^2(n)+1)}{M^2N^2}+\frac{Cm^2}{M^2}\bbE\big\|\bq^{*(t)}-\bq^{*(\del S,t)}\big\|_1^2 + C\bbE\big|q_j^{*(t)} - q_j^{*(\del S,t)}\big|^2\\
            &\quad \quad + \frac{Cn^2}{N^2}\sum_{k=1}^{|S|}\frac{1}{N-k}\sum_{i'\notin S_k}\Big\{\frac{m^2}{M^2}\bbE\big\|\bq^{*(t)}-\bq^{*(\del\{i_k,i'\},t)}\big\|_1^2+\bbE\big|q_j^{*(t)} - q_j^{*(\del \{i_k,i'\},t)}\big|^2\Big\}.
        \end{split}
    \end{equation}
Therefore, setting $t=1$ and noticing that the initial feature sampling probability is always set as $\frac{m}{M}$, \eqref{eq:Dsum_bnd2}-\eqref{eq:Dmin_bnd2} immediately imply \eqref{eq:induction_Dsum_expt}-\eqref{eq:induction_Dmin_expt} in Claim~\ref{claim:induction_Q_loo_expt} with $t=1$. Invoking Lemma~\ref{lem:q_err_bnd}, we have
$$
\bbE\|\bq^{*(\no S, t+1)}-\bq^{*(t+1)}\|_1^2\leq Cm^6\big(\bbE\|\bar{\Delta}^{*(t)}-\bar{\Delta}^{*(\no S,t)}\|_1^2 + M^2\bbE|\bar{\Delta}_{\min}^{*(t)} - \bar{\Delta}_{\min}^{*(\no S, t)}|^2\big),
$$
and hence \eqref{eq:induction_q_expt} is a direct consequence of \eqref{eq:induction_Dsum_expt}-\eqref{eq:induction_Dmin_expt}. Following the same induction arguments as in the proof of Claim~\ref{claim:induction_Q_loo}, we arrive at the conclusion stated in Claim~\ref{claim:induction_Q_loo_expt}.
\end{proof}
Applying our Claim~\ref{claim:induction_Q_loo_expt} with $S=\{i\}$ and consider the $(t-1)$th iteration, \eqref{eq:induction_q_expt} immediately implies Lemma~\ref{lem:q_loo_l2}.
\end{proof}
\begin{proof}[Proof of Lemma~\ref{lem:cond_mean_infK}]
Following the same argument as in~\eqref{eq:h_leaveSout_err}, we can write 
\begin{equation}\label{eq:diff_infK_leavei}
    \begin{split}
        |h_j(Z^*,\bZ_{\backslash i,:},Q^{*(\no i, t)})-h_j(Z^*,\bZ,Q^{*(t)})| 
        &\leq Cq_j^{*(\no i,t)}\Big|\sum_F\mu_F(X^*;\bZ_{\no i,:}) \big(Q^{*(t)}(F) -Q^{*(\no i,t)}(F)\big)\Big| \\
        &\hspace{1em} + C\Big|\sum_F\mu_F(X^*;\bZ_{\no i,:})\big(\delta_j(F;Q^{*(t)}) - \delta_j(F;Q^{*(\no i,t)}\big)\Big|\\
        &\hspace{1em} + Cq_j^{*(\no i,t)}\Big|\sum_F\big(\mu_F(X^*;\bZ)  - \mu_F(X^*;\bZ_{\no i,:})\big)Q^{*(t)}(F)\Big| \\
        &\hspace{1em} + C\Big|\sum_F\big(\mu_F(X^*;\bZ) - \mu_F(X^*;\bZ_{\no i,:})\big)\delta_j(F;Q^{*(t)})\Big|
    \end{split}
\end{equation}
Applying Lemma~\ref{lem:expct_diff_q}, we can bound the first two terms in \eqref{eq:diff_infK_leavei} via the difference between $\bq^{*(t)}$ and $\bq^{*(\no i,t)}$:
     \begin{equation}\label{eq:diff_infK_eq1}
         \begin{split}
            & q_j^{*(\no i,t)}\Big|\sum_F\mu_F(X^*;\bZ_{\no i,:}) \big(Q^{*(t)}(F) -Q^{*(\no i,t)}(F)\big)\Big|\leq Cq_j^{*(\no i,t)}\|\bq^{*(t)}-\bq^{*(\no i,t)}\|_1,\\
             &\Big|\sum_F\mu_F(X^*;\bZ_{\no i,:})\big(\delta_j(F;Q^{*(t)}) - \delta_j(F;Q^{*(\no i,t)}\big)\Big|\\
            &\hspace{2em} \leq C(q_j^{*(t)}+q_j^{*(\no i, t)})\|\bq^{*(t)}-\bq^{*(\no i,t)}\|_1 + C\big|q_j^{*(t)}-q_j^{*(\no i, t)}\big|.
         \end{split}
     \end{equation}
As in the proof of Lemma~\ref{lem:Delta_leaveSout},  we denote $\bbE_{I_0\sim \cU_{n-1}^{[N]\del (\{i,i'\})}}\big[\mu_{I_0\cup \{i\},F}(X^*;\bZ) - \mu_{I_0\cup\{i'\},F}(X^*;\bZ)\big]$ by $g_F(X^*, \bZ_{\del(\{i,i'\})},Z_{i},Z_{i'})$, which is bounded by Assumption~\ref{ass:adp_bounded}. Applying Lemma~\ref{lem:mu_leaveSout_err} to $\mu_F(X^*;\bZ) - \mu_F(X^*;\bZ_{\no i,:})$, we have
    \begin{equation}\label{eq:diff_infK_eq2}
        \begin{split}
            &\left|\sum_F\big(\mu_F(X^*;\bZ) - \mu_F(X^*;\bZ_{\no i,:})\big)Q^{*(t)}(F)\right|\\
            &\hspace{2em}\leq \frac{n}{N(N-1)}\sum_{i'\neq i}\left|\bbE_{F\sim Q^{*(t)}}g_F(X^*, \bZ_{\del\{i,i'\}},Z_{i},Z_{i'})\right|\\
            &\hspace{2em}\leq \frac{n}{N(N-1)}\sum_{i'\neq i}\left|\bbE_{F\sim Q^{*(\del \{i,i'\},t)}}g_F(X^*, \bZ_{\del\{i,i'\}},Z_{i},Z_{i'})\right|\\
            &\hspace{3em}+\frac{n}{N(N-1)}\sum_{i'\neq i}\left|\sum_Fg_F(X^*, \bZ_{\del\{i,i'\}},Z_{i},Z_{i'})\left(Q^{*(\del \{i,i'\},t)}(F)- Q^{*(t)}(F)\right)\right|\\
            &\hspace{2em}\leq \frac{n}{N(N-1)}\sum_{i'\neq i}\left|\bbE_{F\sim Q^{*(\del \{i,i'\},t)}}g_F(X^*, \bZ_{\del\{i,i'\}},Z_{i},Z_{i'})\right|\\
            &\hspace{3em}+\frac{Cn}{N(N-1)}\sum_{i'\neq i}\Big\|\bq^{*(\del \{i,i'\},t)} - \bq^{*(t)}\Big\|_1,
        \end{split}
    \end{equation}
where the third inequality is due to Lemma~\ref{lem:expct_diff_q}.
While for the last term in \eqref{eq:diff_infK_leavei}, we use a similar argument to bounding the third term. Formally, we can write
    \begin{equation*}
        \begin{split}
            &\left|\sum_F\big(\mu_F(X^*;\bZ) - \mu_F(X^*;\bZ_{\no i,:})\big)\delta_j(F;Q^{*(t)})\right|\\
            &\hspace{2em}\leq \frac{n}{N(N-1)}\sum_{i'\neq i}\left|\sum_F \delta_j(F;Q^{*(\del\{i,i'\},t)})g_F(X^*, \bZ_{\del\{i,i'\}},Z_{i},Z_{i'})\right|\\
           &\hspace{3em}+\frac{n}{N(N-1)}\sum_{i'\neq i}\left|\sum_Fg_F(X^*, \bZ_{\del\{i,i'\}},Z_{i},Z_{i'})\left(\delta_j(F;Q^{*(\del\{i,i'\},t)})- \delta_j(F;Q^{*(t)})\right)\right|\\
            &\hspace{2em}\leq \frac{n}{N(N-1)}\sum_{i'\neq i}\left|\sum_F \delta_j(F;Q^{*(\del\{i,i'\},t)})g_F(X^*, \bZ_{\del\{i,i'\}},Z_{i},Z_{i'})\right|\\
           &\hspace{3em}+\frac{Cn}{N(N-1)}\sum_{i'\neq i}\left((q_j^{*(\del \{i,i'\}, t)} + q_j^{*(t)})\|\bq^{*(\del \{i,i'\}, t)} - \bq^{*(t)}\|_1 +|q_j^{*(\del \{i,i'\}, t)} - q_j^{*(t)}|\right),
        \end{split}
    \end{equation*}
where the last line is due to Lemma~\ref{lem:expct_diff_q}. Furthermore, as in \eqref{eq:delta_expt}, we can write 
\begin{equation*}
    \begin{split}
        \left|\sum_F \delta_j(F;Q)g(F)\right| \leq Cq_j\Big|\bbE_{F\sim Q_j}g(F)\Big| + Cq_j\Big|\bbE_{F\sim Q_{\no j}}g(F)\Big|,
    \end{split}
\end{equation*}
where $Q_j$ and $Q_{\no j}$ denote the distributions of $F\sim Q$ conditional on $j\in F$ and $j\notin F$, respectively. Therefore, following the same arguments as those in the proof of Lemma~\ref{lem:Delta_leaveSout}, we have
\begin{equation}\label{eq:diff_infK_eq3}
        \begin{split}
            &\left|\sum_F\big(\mu_F(X^*;\bZ) - \mu_F(X^*;\bZ_{\no i,:})\big)\delta_j(F;Q^{*(t)})\right|\\
            &\hspace{2em}\leq \frac{n}{N(N-1)}\sum_{i'\neq i} q_j^{*(\del \{i,i'\},t)}\Big(\left|\bbE_{F\sim Q_{j}^{*(\del \{i,i'\},t)}}g_F(X^*, \bZ_{\del\{i,i'\}},Z_{i},Z_{i'})\right|\\
            &\hspace{13em}+\left|\bbE_{F\sim Q_{\no j}^{*(\del \{i,i'\},t)}}g_F(X^*, \bZ_{\del\{i,i'\}},Z_{i},Z_{i'})\right|\Big)\\
           &\hspace{3em}+\frac{Cn}{N(N-1)}\sum_{i'\neq i}\left((q_j^{*(\del \{i,i'\}, t)} + q_j^{*(t)})\|\bq^{*(\del \{i,i'\}, t)} - \bq^{*(t)}\|_1 +|q_j^{*(\del \{i,i'\}, t)} - q_j^{*(t)}|\right).
        \end{split}
    \end{equation}
Therefore, by \eqref{eq:diff_infK_eq1}, \eqref{eq:diff_infK_eq2}, and \eqref{eq:diff_infK_eq3}, we obtain
\begin{equation}\label{eq:diff_infK_cond_mean}
    \begin{split}
        &\bbE\left[|h_j(Z^*,\bZ_{\backslash i,:},Q^{*(\no i, t)})-h_j(Z^*,\bZ,Q^{*(t)})||\bZ\right]\\
        &\hspace{2em}\leq \frac{Cnq_j^{*(\no i,t)}}{N(N-1)}\sum_{i'\neq i}\bbE\left[\left|\bbE_{F\sim Q^{*(\del \{i,i'\},t)}}g_F(X^*, \bZ_{\del\{i,i'\}},Z_{i},Z_{i'})\right||\bZ\right]\\
        &\hspace{3em}+\frac{Cn}{N(N-1)}\sum_{i'\neq i} q_j^{*(\del \{i,i'\},t)}\Big(\bbE\left[\big|\bbE_{F\sim Q_{j}^{*(\del \{i,i'\},t)}}g_F(X^*, \bZ_{\del\{i,i'\}},Z_{i},Z_{i'})\big||\bZ\right]\\
        &\hspace{13em}+\bbE\left[\big|\bbE_{F\sim Q_{\no j}^{*(\del \{i,i'\},t)}}g_F(X^*, \bZ_{\del\{i,i'\}},Z_{i},Z_{i'})\big||\bZ\right]\Big)\\
        &\hspace{3em}+\frac{Cn}{N(N-1)}\sum_{i'\neq i}\left((q_j^{*(\del \{i,i'\}, t)} + q_j^{*(\no i,t)} + q_j^{*(t)})\|\bq^{*(\del \{i,i'\}, t)} - \bq^{*(t)}\|_1 +|q_j^{*(\del \{i,i'\}, t)} - q_j^{*(t)}|\right)\\
        &\hspace{3em}+C(q_j^{*(t)}+q_j^{*(\no i, t)})\|\bq^{*(t)}-\bq^{*(\no i,t)}\|_1 + C\big|q_j^{*(t)}-q_j^{*(\no i, t)}\big|
    \end{split}
\end{equation}
For $F$ with size $|F|\leq 4\max\{m,\log(n)\}$, Assumption~\ref{ass:stb} implies that, conditional on $X^*$ and $\bZ_{\del \{i,i'\}}$, $g_F(X^*, \bZ_{\del\{i,i'\}},Z_{i},Z_{i'})$ is sub-Gaussian with parameter bounded by $\mathrm{stb}(n)$ as a function of $Z_i,Z_{i'}$; for every $F$, it is bounded by Assumption~\ref{ass:adp_bounded}.  Each of $Q^{*(\del \{i,i'\},t)}$, $Q_{j}^{*(\del \{i,i'\},t)}$, and $Q_{\no j}^{*(\del \{i,i'\},t)}$ depends only on $\bZ_{\del \{i,i'\}}$.  Thus, Lemma~\ref{lem:good_bad_decomposition} shows that the corresponding $F$-average is bounded with high probability by $C\{\mathrm{stb}(n)\sqrt{\log N+\log M}+C/n\}$. Applying Jensen's inequality on the Banach space defined by the sub-Gaussian norm, we can then bound the first two terms of \eqref{eq:diff_infK_cond_mean} and conclude the proof of Lemma~\ref{lem:cond_mean_infK}.
\end{proof}
\begin{proof}[Proof of Lemma~\ref{lem:expct_diff_q}]
      To show Lemma~\ref{lem:expct_diff_q}, we first note that $  \Big|\sum_F g(F)\big(Q_{\bq^{(1)}}(F) - Q_{\bq^{(2)}}(F)\big)\Big|$ can be viewed as the expectation difference of $g(F)$ under two probability measures $Q_{\bq^{(1)}}$ and $Q_{\bq^{(2)}}$:
      $$
      \sum_F g(F)\big(Q_{\bq^{(1)}}(F) - Q_{\bq^{(2)}}(F)\big) = \bbE_{F\sim Q_{\bq^{(1)}}}g(F) - \bbE_{F\sim Q_{\bq^{(2)}}}g(F).
      $$
      The following lemma from \cite{Liu2026model} shows that we can construct a joint distribution of two random feature subsets $(F^{(1)},F^{(2)})$ such that $F^{(1)}\sim Q_{\bq^{(1)}}$,  $F^{(2)}\sim Q_{\bq^{(2)}}$, and $F^{(1)}$ is close to $F^{(2)}$ if $\bq^{(1)}$ is close to $\bq^{(2)}$. For simplicity, throughout this proof, we denote $Q_{\bq^{(1)}}$ and $Q_{\bq^{(2)}}$ by $Q^{(1)}$ and $Q^{(2)}$, respectively, and use these two sets of notations interchangeably.
      \begin{lem}[Lemma S12 in \cite{Liu2026model}]\label{lem:F1_F2_diff}
        If $\|\bq^{(1)}-\bq^{(2)}\|_{\infty}\leq \frac{1}{2}$, $m\geq\log(1+\frac{1}{1-\delta})$, then there exists a joint distribution $Q^{(1,2)}$ on $2^{[M]}\times 2^{[M]}$ such that when $(F^{(1)},\,F^{(2)})\sim Q^{(1,2)}$, their marginal distributions are $Q^{(1)}$ and $Q^{(2)}$, i.e., $F^{(1)}\sim Q^{(1)}$, $F^{(2)}\sim Q^{(2)}$, and 
        \begin{align}
            \bbP(F^{(1)}\neq F^{(2)})\leq &\frac{2\|\bq^{(1)}-\bq^{(2)}\|_1}{1-e^{-m}},\label{eq:F1F2_diff1}\\
                \bbE\left||F^{(1)}|-|F^{(2)}|\right|\leq &\frac{2m+1}{(1-e^{-m})^2}\|\bq^{(1)} - \bq^{(2)}\|_1,\label{eq:F1F2_diff2}\\
                \bbP\left(j\in F^{(1)},\,F^{(1)}\neq F^{(2)}\right)\leq &\frac{3q_j^{(1)}\|\bq^{(1)}-\bq^{(2)}\|_1 + |q_j^{(2)} - q_j^{(1)}|}{1-e^{-m}},\label{eq:F1F2_diff3}\\
                \bbP\left(j\in F^{(2)},\,F^{(1)}\neq F^{(2)}\right)\leq &\frac{3q_j^{(2)}\|\bq^{(1)}-\bq^{(2)}\|_1 + |q_j^{(2)} - q_j^{(1)}|}{1-e^{-m}}.\label{eq:F1F2_diff4}
        \end{align}
    \end{lem}
Let $F^{(1)}$ and $F^{(2)}$ be the random feature subsets that satisfy \eqref{eq:F1F2_diff1}-\eqref{eq:F1F2_diff4} in Lemma \ref{lem:F1_F2_diff}. We can then write
    \begin{equation*}
        \begin{split}
           \left|\sum_F g(F)\big(Q_{\bq^{(1)}}(F) - Q_{\bq^{(2)}}(F)\big)\right| = &\Big|\bbE_{F\sim Q_{\bq^{(1)}}}g(F) - \bbE_{F\sim Q_{\bq^{(2)}}}g(F)\Big|\\
            = &\left|\bbE_{F^{(1)}, F^{(2)}}\left(g(F^{(1)})-g(F^{(2)})\right)\right|\\
            \leq&\bbE_{F^{(1)}, F^{(2)}}\big(2\ind(F^{(1)}\neq F^{(2)})\big)\\
            \leq &2\bbP(F^{(1)}\neq F^{(2)})\\
            \leq& \frac{4\|\bq^{(1)}-\bq^{(2)}\|_1}{1-e^{-m}}\\
            \leq&C\|\bq^{(1)} - \bq^{(2)}\|_1,
        \end{split}
    \end{equation*}
    where the last line is due to Assumption~\ref{ass:adp_prob_stb}.    

Furthermore, by the definition $\delta_j(F;Q) = Q(F)\big(\ind(j\in F) - \omega_j(Q)\ind(j\notin F)\big)$, we have 
\begin{equation}\label{eq:eta2_bnd1}
    \begin{split}
        &\sum_F g(F)\left(\delta_j(F;Q^{(1)}) - \delta_j(F;Q^{(2)})\right)\\
        =&\sum_Fg(F)\ind(j\in F)\left(Q^{(1)}(F) - Q^{(2)}(F))\right) \\&+\sum_Fg(F)\ind(j\notin F)\left(\omega_j(Q^{(2)})Q^{(2)}(F) - \omega_j(Q^{(1)})Q^{(1)}(F))\right)\\
        =&\bbE_{F^{(1)},F^{(2)}}\left(g(F^{(1)})\ind(j\in F^{(1)}) - g(F^{(2)})\ind(j\in F^{(2)})\right)\\
        &-\omega_j(Q^{(1)})\bbE_{F^{(1)},F^{(2)}}\left(g(F^{(1)})\ind(j\notin F^{(1)}) - g(F^{(2)})\ind(j\notin F^{(2)})\right)\\
        &+\left(\omega_j(Q^{(2)})-\omega_j(Q^{(1)})\right)\bbE_{F^{(2)}}\left(g(F^{(2)})\ind(j\notin F^{(2)})\right).
    \end{split}
\end{equation}
Here, $\omega_j(Q) = \frac{q_j}{1-q_j-\prod_{k=1}^M(1-q_k)}$. The following lemma shows an upper bound for $\Big|\omega_j(Q^{(2)})-\omega_j(Q^{(1)})\Big|$. 
\begin{lem}[Lemma S13 in \cite{Liu2026model}]\label{lem:omega_j_diff}
    $$
    |\omega_j(Q^{(1)}) - \omega_j(Q^{(2)})|\leq C\big|q_j^{(1)} - q_j^{(2)}\big| + C(q_j^{(1)}+ q_j^{(2)})\|\bq^{(1)}-\bq^{(2)}\|_1.
    $$
\end{lem}
Therefore, following similar arguments to bounding $\left|\sum_F g(F)\big(Q_{q^{(1)}}(F) - Q_{q^{(2)}}(F)\big)\right|$ and leveraging Lemma~\ref{lem:omega_j_diff}, we can bound $\sum_F g(F)\left(\delta_j(F;Q^{(1)}) - \delta_j(F;Q^{(2)})\right)$ as follows:
\begin{equation*}
    \begin{split}
        &\Big|\sum_F g(F)\left(\delta_j(F;Q^{(1)}) - \delta_j(F;Q^{(2)})\right)\Big|\\
        &\hspace{2em}\leq 2\bbP\Big(g(F^{(1)})\ind(j\in F^{(1)})\neq g(F^{(2)})\ind(j\in F^{(2)})\Big) \\
        &\hspace{3em}+2\omega_j(Q^{(1)})\bbP(F^{(1)}\neq F^{(2)}) + \Big|\omega_j(Q^{(2)})-\omega_j(Q^{(1)})\Big|\\
        &\hspace{2em}\leq 2\left(\bbP\Big(j\in F^{(1)},F^{(1)}\neq F^{(2)}\Big)+\bbP\Big(j\in F^{(2)},F^{(1)}\neq F^{(2)}\Big)\right)\\
        &\hspace{3em}+Cq_j^{(1)}\bbP(F^{(1)}\neq F^{(2)}) + \Big|\omega_j(Q^{(2)})-\omega_j(Q^{(1)})\Big|\\
        &\hspace{2em}\leq C\big|q_j^{(1)} - q_j^{(2)}\big| + C(q_j^{(1)}+ q_j^{(2)})\|\bq^{(1)}-\bq^{(2)}\|_1.
    \end{split}
\end{equation*}
The proof is now complete.
\end{proof}
\begin{proof}[Proof of Lemma~\ref{lem:mu_leaveSout_err}]
    We prove Lemma~\ref{lem:mu_leaveSout_err} via induction. First, we consider the scenario where $|S|=1$. W.L.O.G., let $S=\{i\}$. We can then write
    \begin{equation*}
        \begin{split}
            \mu_F(X^*;\bZ) =& \frac{1}{\binom{N}{n}}\sum_{I\subset[N],|I|=n}\mu_{I,F}(X^*;\bZ)\\
            =& \frac{\binom{N-1}{n}}{\binom{N}{n}}\frac{1}{\binom{N-1}{n}}\sum_{I\subset[N],|I|=n, I\not\owns i}\mu_{I,F}(X^*; \bZ) + \frac{1}{\binom{N}{n}}\sum_{I\subset[N],|I|=n, I\owns i}\mu_{I,F}(X^*; \bZ) \\
            =&\frac{N-n}{N}\mu_F(X^*;\bZ_{\no i}) + \frac{n}{N}\frac{1}{\binom{N-1}{n-1}}\sum_{I\subset[N],|I|=n, I\owns i}\mu_{I,F}(X^*; \bZ).
        \end{split}
    \end{equation*}
    Hence, 
    \begin{equation}\label{eq:mu_loo_err}
        \begin{split}
            \mu_F(X^*;\bZ) - \mu_F(X^*;\bZ_{\no i}) = &\frac{n}{N}\left(\frac{1}{\binom{N-1}{n-1}}\sum_{I\subset[N],|I|=n, I\owns i}\mu_{I,F}(X^*; \bZ) - \mu_F(X^*;\bZ_{\no i})\right)\\
            =&\frac{n}{N}\left(\bbE_{I\sim\cU^{[N]}_n|I\owns i}\mu_{I,F}(X^*; \bZ) - \bbE_{I\sim\cU^{[N]\no i}_n}\mu_{I,F}(X^*; \bZ)\right)\\
            =&\frac{n}{N}\left(\bbE_{I_0\sim\cU^{[N]\no i}_{n-1}}\mu_{I_0\cup \{i\},F}(X^*; \bZ) - \bbE_{I\sim\cU^{[N]\no i}_n}\mu_{I,F}(X^*; \bZ)\right)\\
            =&\frac{n}{N}\Big(\frac{1}{N-1}\sum_{i'\neq i}\bbE_{I_0\sim\cU^{[N]\del\{i,i'\}}_{n-1}}\mu_{I_0\cup \{i\},F}(X^*; \bZ) \\
            &- \frac{1}{N-1}\sum_{i'\neq i}\bbE_{I_0\sim\cU^{[N]\del \{i,i'\}}_{n-1}}\mu_{I_0\cup\{i'\},F}(X^*; \bZ)\Big)\\
            =&\frac{n}{N(N-1)}\sum_{i'\neq i}\Big(\bbE_{I_0\sim\cU^{[N]\del\{i,i'\}}_{n-1}}\mu_{I_0\cup \{i\},F}(X^*; \bZ) \\
            &\hspace{7em}- \bbE_{I_0\sim\cU^{[N]\del \{i,i'\}}_{n-1}}\mu_{I_0\cup\{i'\},F}(X^*; \bZ)\Big),
        \end{split}
    \end{equation}
    which verifies Lemma~\ref{lem:mu_leaveSout_err} when $|S|=1$. Now suppose that Lemma~\ref{lem:mu_leaveSout_err} holds for $|S|\leq l$. Then for any $S$ of size $l+1$, we have
    \begin{equation*}
        \begin{split}
            &\mu_F(X^*;\bZ) - \mu_F(X^*;\bZ_{\del S})\\
            &\quad=\mu_F(X^*;\bZ) - \mu_F(X^*;\bZ_{\del S_{l}}) + \mu_F(X^*;\bZ_{\del S_l}) - \mu_F(X^*;\bZ_{\del S})\\
            &\quad=\sum_{k=1}^l\frac{n}{(N-k+1)(N-k)}\sum_{i'\notin S_k}\Big\{\bbE_{I_0\sim \cU_{n-1}^{[N]\del (S_k\cup\{i'\})}}\big[\mu_{I_0\cup \{i_k\},F}(X^*;\bZ) - \mu_{I_0\cup\{i'\},F}(X^*;\bZ)\big]\Big\}\\
            &\quad\hspace{1em}+\mu_F(X^*;\bZ_{\del S_l}) - \mu_F(X^*;\bZ_{\del S}).
        \end{split}
    \end{equation*}
    While for $\mu_F(X^*;\bZ_{\del S_l}) - \mu_F(X^*;\bZ_{\del S})$, there is only one sample $i_{l+1}$ being left out, so we can apply \eqref{eq:mu_loo_err} again to obtain that
    \begin{equation*}
        \begin{split}
            &\mu_F(X^*;\bZ) - \mu_F(X^*;\bZ_{\del S})\\
            &\hspace{2em}=\sum_{k=1}^l\frac{n}{(N-k+1)(N-k)}\sum_{i'\notin S_k}\Big\{\bbE_{I_0\sim \cU_{n-1}^{[N]\del (S_k\cup\{i'\})}}\big[\mu_{I_0\cup \{i_k\},F}(X^*;\bZ) - \mu_{I_0\cup\{i'\},F}(X^*;\bZ)\big]\Big\}\\
            &\hspace{3em}+\frac{n}{(N-l)(N-l-1)}\sum_{i'\notin S_l\cup i_{l+1}}\Big(\bbE_{I_0\sim\cU^{[N]\del(S_l\cup\{i_{l+1},i'\})}_{n-1}}\mu_{I_0\cup \{i_{l+1}\},F}(X^*; \bZ)\\
            &\hspace{17em}- \bbE_{I_0\sim\cU^{[N]\del (S_l\cup\{i_{l+1},i'\})}_{n-1}}\mu_{I_0\cup\{i'\},F}(X^*; \bZ)\Big)\\
            &\hspace{2em}=\sum_{k=1}^{l+1}\frac{n}{(N-k+1)(N-k)}\sum_{i'\notin S_k}\Big\{\bbE_{I_0\sim \cU_{n-1}^{[N]\del (S_k\cup\{i'\})}}\big[\mu_{I_0\cup \{i_k\},F}(X^*;\bZ) - \mu_{I_0\cup\{i'\},F}(X^*;\bZ)\big]\Big\}.
        \end{split}
    \end{equation*}
    The proof is now complete.
\end{proof}
\begin{proof}[Proof of Lemma~\ref{lem:good_bad_decomposition}]
Put $a:=\max\{m,\log(n)\}$ and $L:=4a$.  Conditional on $\mathcal H$, the vector $\bq$ is fixed; all auxiliary probabilities and expectations below are understood to be conditional on $\mathcal H$.  Let $B_1,\ldots,B_M$ be independent Bernoulli variables with success probabilities $q_1,\ldots,q_M$, let $S:=\sum_{k=1}^M B_k$, and let $p_0:=\bbP(S=0)$.  Then $Q_{\bq}$ is the law of $\{k:B_k=1\}$ conditional on $S>0$, and
\begin{equation*}
    p_0=\prod_{k=1}^M(1-q_k)\leq e^{-m}.
\end{equation*}
Let $\phi(u):=u\log u-u+1$.  For every $\lambda>0$, independence and the inequality $1+v\leq e^v$ give
\begin{equation*}
    \bbE e^{\lambda S}
    =\prod_{k=1}^M\{1+q_k(e^\lambda-1)\}
    \leq \exp\left\{(e^\lambda-1)\sum_{k=1}^M q_k\right\}
    =\exp\{m(e^\lambda-1)\}
    \leq\exp\{a(e^\lambda-1)\}.
\end{equation*}
Taking $\lambda=\log 4$ and applying Markov's inequality therefore yields
\begin{equation*}
    \bbP(S>L)
    =\bbP(S>4a)
    \leq \exp\{-4a\log 4+a(4-1)\}
    =\exp\{-a(4\log 4-3)\}
    =\exp\{-\phi(4)a\}.
\end{equation*}
Consequently,
\begin{equation*}
    Q_{\bq}(|F|>L)
    =\frac{\bbP(S>L)}{1-p_0}
    \leq\frac{\exp\{-\phi(4)a\}}{1-e^{-m}}
    \leq C\exp\{-\phi(4)a\}
    \leq\frac{C}{n}.
\end{equation*}
Here Assumption~\ref{ass:adp_prob_stb} implies $(1-e^{-m})^{-1}\leq c_1/c_2$, and the last inequality only uses $a\geq\log n$ and $\phi(4)>1$. 

For $Q_{\bq,j}$, conditioning on $j\in F$ fixes $B_j=1$ and leaves the remaining indicators independent.  With $S_{-j}:=\sum_{k\neq j}B_k$, we have $\bbE S_{-j}=m-q_j\leq a$ and, since $a\geq1$, $L-1=4a-1\geq3a$.  Repeating the moment-generating-function calculation above for $S_{-j}$ and taking $\lambda=\log 3$ gives
\begin{equation*}
    Q_{\bq,j}(|F|>L)
    \leq\bbP(S_{-j}>3a)
    \leq\exp\{-3a\log 3+a(3-1)\}
    =\exp\{-\phi(3)a\}
    \leq\frac{C}{n},
\end{equation*}
where the last inequality uses $\phi(3)=3\log3-2>1$ and $a\geq\log n$.  Finally,
\begin{equation*}
    Q_{\bq}(j\notin F)
    =1-\frac{q_j}{1-p_0}
    \geq1-\frac{\delta}{1-e^{-m}}
    \geq1-c_1.
\end{equation*}
Therefore, whenever $Q_{\bq,\no j}$ is used,
\begin{equation*}
    Q_{\bq,\no j}(|F|>L)
    \leq\frac{Q_{\bq}(|F|>L)}{Q_{\bq}(j\notin F)}
    \leq\frac{C}{n}.
\end{equation*}
Thus every admissible choice of $R$ satisfies $R(|F|>L)\leq C/n$.

Write
\begin{equation*}
    A:=\bbE_{F\sim R}\big[V_F\ind\{|F|\leq L\}\big],
    \qquad
    B:=\bbE_{F\sim R}\big[V_F\ind\{|F|>L\}\big].
\end{equation*}
Conditional on $\mathcal H$, the triangle inequality for the sub-Gaussian norm and the $\mathcal H$-measurability of $R$ give
\begin{equation*}
    \|A\|_{\psi_2}
    \leq \sum_{F:\,|F|\leq L}R(F)\|V_F\|_{\psi_2}
    \leq s,
    \qquad
    |B|\leq bR(|F|>L)\leq \frac{Cb}{n}.
\end{equation*}
The bound $\|A\|_{\psi_2}\leq s$ implies that, conditionally on $\mathcal H$, there exist constants $c,C>0$ such that, for every $x>0$,
\begin{equation*}
    \bbP\big(|A|>Cs\sqrt{x}\mid\mathcal H\big)\leq2e^{-cx}.
\end{equation*}
Since $\bbE_{F\sim R}V_F=A+B$ and $|B|\leq Cb/n$, increasing $C$ if necessary gives
\begin{equation*}
    \bbP\left(\left|\bbE_{F\sim R}V_F\right|>
    C\left\{s\sqrt{x}+\frac{b}{n}\right\}\,\middle|\,\mathcal H\right)
    \leq2e^{-cx}.
\end{equation*}
The same sub-Gaussian norm bound also implies $\bbE(A^2\mid\mathcal H)\leq Cs^2$.  Together with $B^2\leq Cb^2/n^2$ and $|A+B|^2\leq2A^2+2B^2$, this yields
\begin{equation*}
    \bbE\left(\left|\bbE_{F\sim R}V_F\right|^2\,\middle|\,\mathcal H\right)
    \leq C\left\{s^2+\frac{b^2}{n^2}\right\},
\end{equation*}
which completes the proof.
\end{proof}

\begin{proof}[Proof of Lemma~\ref{lem:q_err_bnd}]
Let
\begin{equation}\label{eq:til_Delta12_def}
\til{\Delta}^{(1)}_j
= \Delta^{(1)}_j-\min_{\ell}\Delta^{(1)}_\ell+\frac{c_0}{M},
\qquad
\til{\Delta}^{(2)}_j
= \Delta^{(2)}_j-\min_{\ell}\Delta^{(2)}_\ell+\frac{c_0}{M}.    
\end{equation}

Define
\begin{equation}\label{eq:t1_t2_def}
\begin{split}
    t_1^\star
&= \sup\Bigl\{0<t\le \max_{l}\til{\Delta}^{(1)}_l:\;
\frac{m\max_{l}\bigl(\til{\Delta}^{(1)}_l\wedge t\bigr)}
{\sum_{l=1}^M\bigl(\til{\Delta}^{(1)}_l\wedge t\bigr)}
\le \delta \Bigr\},\\
t_2^\star
&= \sup\Bigl\{0<t\le \max_{l}\til{\Delta}^{(2)}_l:\;
\frac{m\max_{l}\bigl(\til{\Delta}^{(2)}_l\wedge t\bigr)}
{\sum_{l=1}^M\bigl(\til{\Delta}^{(2)}_l\wedge t\bigr)}
\le \delta \Bigr\}.
\end{split}
\end{equation}
Let
\[
w^{(1)}_j=\til{\Delta}^{(1)}_j\wedge t_1^\star,
\qquad
w^{(2)}_j=\til{\Delta}^{(2)}_j\wedge t_2^\star.
\]
Then the output sampling probabilities can be written as
\[
q^{(1)}_j=\frac{m w^{(1)}_j}{\sum_{k}w^{(1)}_k},
\qquad
q^{(2)}_j=\frac{m w^{(2)}_j}{\sum_{k}w^{(2)}_k}.
\]
Hence,
\begin{align*}
\bigl|q^{(1)}_j-q^{(2)}_j\bigr|
&=\left|\frac{m w^{(1)}_j}{\sum_k w^{(1)}_k}-\frac{m w^{(2)}_j}{\sum_k w^{(2)}_k}\right|\\
&=\left|
\frac{m w^{(1)}_j \sum_k\bigl(w^{(2)}_k-w^{(1)}_k\bigr)}
{\bigl(\sum_k w^{(1)}_k\bigr)\bigl(\sum_k w^{(2)}_k\bigr)}
+\frac{m\bigl(w^{(1)}_j-w^{(2)}_j\bigr)}{\sum_k w^{(2)}_k}
\right|\\
&\le
\frac{1}{\sum_k w^{(2)}_k}\cdot
\frac{m w^{(1)}_j}{\sum_k w^{(1)}_k}\,
\bigl\|w^{(1)}-w^{(2)}\bigr\|_1
+\frac{m}{\sum_k w^{(2)}_k}\,\bigl|w^{(1)}_j-w^{(2)}_j\bigr|.
\end{align*}

By the definition of $w^{(1)}_j$ and $w^{(2)}_j$,
\[
\bigl|w^{(1)}_j-w^{(2)}_j\bigr|
\le
\begin{cases}
\bigl|\til{\Delta}^{(1)}_j-\til{\Delta}^{(2)}_j\bigr|,
& \til{\Delta}^{(1)}_j\le t_1^\star,\ \til{\Delta}^{(2)}_j\le t_2^\star,\\[3pt]
\bigl|t_1^\star-t_2^\star\bigr|\vee \bigl|\til{\Delta}^{(1)}_j-\til{\Delta}^{(2)}_j\bigr|,
& \mathbf{1}\!\left(\til{\Delta}^{(1)}_j\le t_1^\star\right)
+\mathbf{1}\!\left(\til{\Delta}^{(2)}_j\le t_2^\star\right)=1,\\[3pt]
\bigl|t_1^\star-t_2^\star\bigr|,
& \til{\Delta}^{(1)}_j> t_1^\star,\ \til{\Delta}^{(2)}_j> t_2^\star.
\end{cases}
\]
The following claim gives an upper bound for $\bigl|t_1^\star-t_2^\star\bigr|$:
\begin{claim}\label{clm:t1_t2_err}
    \[
\bigl|t_1^\star-t_2^\star\bigr|\le Cm\|\til{\varepsilon}\|_1,
\qquad \text{where }\ \til{\varepsilon}:=\til{\Delta}^{(1)}-\til{\Delta}^{(2)}.
\]
\end{claim}
We will assume Claim~\ref{clm:t1_t2_err} holds in the following proof, and defer the proof of Claim~\ref{clm:t1_t2_err} till the end.
Consequently,
\[
\bigl|w^{(1)}_j-w^{(2)}_j\bigr|
\le
\begin{cases}
|\til{\varepsilon}_j|, & \til{\Delta}^{(1)}_j\le t_1^\star,\ \til{\Delta}^{(2)}_j\le t_2^\star,\\[3pt]
Cm\|\til{\varepsilon}\|_1, & \til{\Delta}^{(1)}_j> t_1^\star\ \text{or}\ \til{\Delta}^{(2)}_j> t_2^\star,
\end{cases}
\]
and therefore
\[
\|w^{(1)}-w^{(2)}\|_1
\le \frac{Cm^2}{\delta}\,\|\til{\varepsilon}\|_1,
\]
where we have utilized the fact that 
\begin{equation}\label{eq:thresholded_size}
    |\{j:\,\til{\Delta}^{(1)}_j>t_1^*\}|,\,|\{j:\,\til{\Delta}^{(2)}_j>t_2^*\}|\leq \frac{m}{\delta}.
\end{equation}
If $\til{\Delta}^{(1)}_j\le t_1^\star$ and $\til{\Delta}^{(2)}_j\le t_2^\star$, then
\begin{align*}
\bigl|q^{(1)}_j-q^{(2)}_j\bigr|
&\le
\frac{1}{\sum_k w^{(2)}_k}\,q^{(1)}_j\,\|w^{(1)}-w^{(2)}\|_1
+\frac{m}{\sum_k w^{(2)}_k}\,|\til{\varepsilon}_j|\\
&\le
\frac{q^{(1)}_j}{\sum_k w^{(2)}_k}\cdot\frac{Cm^2}{\delta}\,\|\til{\varepsilon}\|_1 
+\frac{m}{\sum_k w^{(2)}_k}\,|\til{\varepsilon}_j|.
\end{align*}
Noting that when $t = c_0/M$, $\frac{m\max_l(\til{\Delta}_l^{(1)}\wedge t)}{\sum_{l=1}^M(\til{\Delta}_l^{(1)}\wedge t)} = \frac{m}{M}<\delta$. Hence, by the definition of $t_1^*$ and $\til{\Delta}_j^{(1)}$ in \eqref{eq:t1_t2_def} and \eqref{eq:til_Delta12_def}, we have $\til{\Delta}_j^{(1)},\,t_1^*\geq c_0/M$ for any $1\leq j\leq M$. Similarly, $\til{\Delta}_j^{(2)},\,t_2^*\geq c_0/M$ also holds. Therefore, we have $w^{(2)}_k\ge c_0/M$ and thus $\sum_{k=1}^M w^{(2)}_k\ge c_0$; hence
\[
\bigl|q^{(1)}_j-q^{(2)}_j\bigr|
\le \frac{Cm^2}{\delta}\Bigl(q^{(1)}_j\|\til{\varepsilon}\|_1+|\til{\varepsilon}_j|\Bigr).
\]

Otherwise, if
$\mathbf{1}(\til{\Delta}^{(1)}_j\le t_1^\star)+\mathbf{1}(\til{\Delta}^{(2)}_j\le t_2^\star)=1$, then
\[
\bigl|q^{(1)}_j-q^{(2)}_j\bigr|
\le
\frac{Cm^2}{\delta}\,\|\til{\varepsilon}\|_1\cdot \frac{q^{(1)}_j}{\sum_k w^{(2)}_k}
+\frac{Cm^2}{\sum_k w^{(2)}_k}\,\|\til{\varepsilon}\|_1
\le Cm^2\,\|\til{\varepsilon}\|_1,
\]
where the last inequality is due to the fact that $q_j^{(1)}\leq \delta$.
On the other hand, if $\til{\Delta}^{(1)}_j>t_1^\star$ and $\til{\Delta}^{(2)}_j>t_2^\star$, then $q^{(1)}_j=q^{(2)}_j=\delta$, implying that $\bigl|q^{(1)}_j-q^{(2)}_j\bigr|=0$.
Therefore,
\[
\|\bq^{(1)}-\bq^{(2)}\|_1 \le \frac{Cm^3}{\delta}\|\til{\varepsilon}\|_1 + \frac{Cm^2}{\delta}\|\til{\varepsilon}\|_1\leq  \frac{Cm^3}{\delta}\|\til{\varepsilon}\|_1,
\]
where we have applied \eqref{eq:thresholded_size} and the fact that $\sum_{j=1}^M q_j^{(1)}=m$.

\paragraph{Proof of claim~\ref{clm:t1_t2_err}.} Let
\[
f^{(1)}_j(t)=\frac{\til{\Delta}^{(1)}_j\wedge t}{\bigl(\max_l \til{\Delta}^{(1)}_l\bigr)\wedge t},
\qquad
f^{(2)}_j(t)=\frac{\til{\Delta}^{(2)}_j\wedge t}{\bigl(\max_l \til{\Delta}^{(2)}_l\bigr)\wedge t},
\]
and
\[
f^{(1)}(t)=\sum_{j=1}^M f^{(1)}_j(t),
\qquad
f^{(2)}(t)=\sum_{j=1}^M f^{(2)}_j(t).
\]
We can write
\[
t_1^\star=\sup\Bigl\{0<t\le \max_l\til{\Delta}^{(1)}_l:\ f^{(1)}(t)\ge \frac{m}{\delta}\Bigr\},
\qquad
t_2^\star=\sup\Bigl\{0<t\le \max_l\til{\Delta}^{(2)}_l:\ f^{(2)}(t)\ge \frac{m}{\delta}\Bigr\}.
\]
Since $f^{(1)}(t)$ and $f^{(2)}(t)$ are both continuous non-increasing functions, we know that
either $f^{(1)}(t_1^\star)=m/\delta$ (resp.\ $f^{(2)}(t_2^\star)=m/\delta$), or
$t_1^\star=\max_l\til{\Delta}^{(1)}_l$ with $f^{(1)}(t_1^\star)\ge m/\delta$
(resp.\ $t_2^\star=\max_l\til{\Delta}^{(2)}_l$ with $f^{(2)}(t_2^\star)\ge m/\delta$).

We now consider three distinct cases:
\begin{enumerate}
\item[(i)]
$\mathbf{1}\!\left(t_1^\star=\max_l\til{\Delta}^{(1)}_l\right)
+\mathbf{1}\!\left(t_2^\star=\max_l\til{\Delta}^{(2)}_l\right)=0$.
\item[(ii)]
$\mathbf{1}\!\left(t_1^\star=\max_l\til{\Delta}^{(1)}_l\right)
+\mathbf{1}\!\left(t_2^\star=\max_l\til{\Delta}^{(2)}_l\right)=2$.
\item[(iii)]
$\mathbf{1}\!\left(t_1^\star=\max_l\til{\Delta}^{(1)}_l\right)
+\mathbf{1}\!\left(t_2^\star=\max_l\til{\Delta}^{(2)}_l\right)=1$.
\end{enumerate}

\noindent\textbf{Case (i).}
$t_1^\star<\max_l\widetilde\Delta_l^{(1)}$ and $t_2^\star<\max_l\widetilde\Delta_l^{(2)}$.
Hence $f^{(1)}(t_1^\star)=f^{(2)}(t_2^\star)=\frac{m}{\delta}.$
W.L.O.G., suppose that $t_1^\star\le t_2^\star\le \max_l\widetilde\Delta_l^{(2)}$.
We note that
\begin{align*}
\bigl|f^{(1)}(t_1^\star)-f^{(2)}(t_1^\star)\bigr|
&\le \sum_{j=1}^M \bigl|f_j^{(1)}(t_1^\star)-f_j^{(2)}(t_1^\star)\bigr|\\
&= \sum_{j=1}^M \frac{\bigl|(\widetilde\Delta_j^{(1)}\wedge t_1^\star)-(\widetilde\Delta_j^{(2)}\wedge t_1^\star)\bigr|}{t_1^\star}\\
&\le \frac{\|\widetilde\Delta^{(1)}-\widetilde\Delta^{(2)}\|_1}{t_1^\star}.
\end{align*}
In addition, for any $\tau\ge 0$,
\begin{align*}
\bigl|f^{(2)}(t_1^\star)-f^{(2)}(t_2^\star)\bigr|
&= \sum_{j=1}^M\bigl(f_j^{(2)}(t_1^\star)-f_j^{(2)}(t_2^\star)\bigr)\\
&\ge
\sum_{j:\ \widetilde\Delta_j^{(2)}\le (t_1^\star+\tau)\wedge t_2^\star}
\Bigl(t_1^{\star-1}\wedge \widetilde\Delta_j^{(2)-1}-t_2^{\star-1}\Bigr)\widetilde\Delta_j^{(2)}\\
&\ge
\sum_{j:\ \widetilde\Delta_j^{(2)}\le (t_1^\star+\tau)\wedge t_2^\star}
\frac{t_2^\star-(t_1^\star\vee \widetilde\Delta_j^{(2)})}{t_2^{\star 2}}\,
\widetilde\Delta_j^{(2)}\\
&\ge
\frac{t_2^\star-t_1^\star-\tau}{t_2^{\star 2}}
\sum_{j:\ \widetilde\Delta_j^{(2)}\le (t_1^\star+\tau)\wedge t_2^\star}
\widetilde\Delta_j^{(2)}.
\end{align*}
Let $\tau=\|\widetilde\varepsilon\|_\infty$. Then
$\mathbf{1}\{\widetilde\Delta_j^{(2)}\le t_1^\star+\tau\}\ge \mathbf{1}\{\widetilde\Delta_j^{(1)}\le t_1^\star\}.$
If $t_1^\star+\tau\le t_2^\star$, then
\begin{align*}
\sum_{j:\ \widetilde\Delta_j^{(2)}\le (t_1^\star+\tau)\wedge t_2^\star}\widetilde\Delta_j^{(2)}
&=\sum_{j:\ \widetilde\Delta_j^{(2)}\le t_1^\star+\tau}\widetilde\Delta_j^{(2)}\\
&\ge \sum_{j:\ \widetilde\Delta_j^{(1)}\le t_1^\star}\widetilde\Delta_j^{(1)}
-\|\widetilde\Delta^{(1)}-\widetilde\Delta^{(2)}\|_1\\
&\ge \frac{c_0}{M}\Bigl(M-\frac{m}{\delta}\Bigr)-\|\widetilde\Delta^{(1)}-\widetilde\Delta^{(2)}\|_1\\
&= c_0\Bigl(1-\frac{m}{\delta M}\Bigr)-\|\widetilde\varepsilon\|_1\\
&\ge \frac{c_0(1-c_1)}{2},
\end{align*}
where the last line is due to that $\|\widetilde\varepsilon\|_1\leq \|\varepsilon\|_1 + M\varepsilon_{\min}$, and the latter is assumed to be bounded by $\frac{c_0(1-c_1)}{2}$.
Otherwise, 
\[
\sum_{j:\ \widetilde\Delta_j^{(2)}\le (t_1^\star+\tau)\wedge t_2^\star}\widetilde\Delta_j^{(2)}
=\sum_{j:\ \widetilde\Delta_j^{(2)}\le t_2^\star}\widetilde\Delta_j^{(2)}
\ge c_0(1-c_1),
\]
where the last line is follows from the same argument as above when lower bounding $\sum_{j:\ \widetilde\Delta_j^{(1)}\le t_1^\star}\widetilde\Delta_j^{(1)}$.  
Since $f^{(2)}(t_2^\star)-f^{(2)}(t_1^\star)=f^{(1)}(t_1^\star)-f^{(2)}(t_1^\star)$, we have
\[
\frac{t_2^\star-t_1^\star-\|\widetilde\varepsilon\|_\infty}{t_2^{\star 2}}\cdot \frac{c_0(1-c_1)}{2}
\le
\frac{\|\widetilde\Delta^{(1)}-\widetilde\Delta^{(2)}\|_1}{t_1^\star}.
\]
Therefore,
\[
|t_1^\star-t_2^\star|
\le \|\widetilde\varepsilon\|_\infty+\frac{2t_2^{\star 2}}{c_0(1-c_1)t_1^\star}\,\|\widetilde\Delta^{(1)}-\widetilde\Delta^{(2)}\|_1.
\le \|\widetilde\varepsilon\|_\infty+\frac{C}{c_0^2c_1}\,\|\widetilde\xi\|_1.
\]
Since $t_2^*\leq \max_l\til{\Delta}^{(2)}_l \leq C$, and $t_1^*$ satisfies $f^{(1)}(t_1^*) = \frac{\sum_{j=1}^M(\til{\Delta}_j\wedge t_1^*)}{t_1^*}=\frac{m}{\delta}$ which implies $t_1^*\geq c_0\frac{\delta}{m}$, we have
\[
|t_1^\star-t_2^\star|
\le Cm\|\til{\varepsilon}\|_1.
\]
Here, we omit the factor $\frac{1}{\delta}$ since it is bounded by a universal constant $2$.

\noindent\textbf{Case (ii).}
\[
|t_1^\star-t_2^\star|
=\left|\max_l\widetilde\Delta_l^{(1)}-\max_l\widetilde\Delta_l^{(2)}\right|
\le \|\widetilde\varepsilon\|_\infty.
\]

\smallskip
\noindent\textbf{Case (iii).}
As before, W.L.O.G.\ we assume $t_1^\star\le t_2^\star$.
If $t_1^\star=\max_l\widetilde\Delta_l^{(1)}$, then
\[
|t_1^\star-t_2^\star|
\le \left|\max_l\widetilde\Delta_l^{(1)}-\max_l\widetilde\Delta_l^{(2)}\right|
\le \|\widetilde\varepsilon\|_\infty.
\]
Otherwise, $t_2^\star=\max_l\widetilde\Delta_l^{(2)}$ and $f^{(2)}(t_2^\star)\ge \frac{m}{\delta}$, while
$t_1^\star<\max_l\widetilde\Delta_l^{(1)}$ and $f^{(1)}(t_1^\star)=\frac{m}{\delta}$.
Then
\[
0\ge f^{(2)}(t_2^\star)-f^{(2)}(t_1^\star)\ge f^{(1)}(t_1^\star)-f^{(2)}(t_1^\star),
\]
which implies
\[
\bigl|f^{(2)}(t_2^\star)-f^{(2)}(t_1^\star)\bigr|
\le \bigl|f^{(1)}(t_1^\star)-f^{(2)}(t_1^\star)\bigr|.
\]
Following the same argument as in Case (i), we obtain $|t_1^\star-t_2^\star|\le Cm\|\widetilde\varepsilon\|_1$. This completes the proof of Claim~\ref{clm:t1_t2_err}, and hence the proof.
\end{proof}

\section{Additional Empirical Results}
\subsection{Implementation Details}\label{add_sim_set}
In the simulation study, all machine learning models are implemented using scikit-learn. For all LOCO methods, ridge regression uses fit\_intercept $=$ False, and kernel ridge regression uses the RBF kernel. LOCO-AdaMP and LOCO-MP use DecisionTreeRegressor with its default hyperparameters, whereas LOCO-Split uses RandomForestRegressor with n\_estimators set to $100$, max\_features set to $1/3$ and min\_samples\_leaf set to $5$, while retaining the default values for all other hyperparameters.

In the case study, for LOCO-AdaMP, we set the number of minipatches within each iteration as $K_1=\cdots=K_4=1,000, K_5 = 10,000$, and keep the probability upper bound and the constant buffer the same as in the simulation study. For LOCO-MP, the number of minipatches is set as $14, 000$, same as the total number of minipatches sampled by LOCO-AdaMP across all iterations. For Floodgate, we follow Section 4.5 of \cite{zhang2020floodgate} and estimate the conditional feature distributions using graphical lasso (GLASSO) with $3$-fold cross-validation. We implemented GCM using version 0.2.0 of the R package GeneralisedCovarianceMeasure, obtained from the CRAN archive. For each method, we rank the features by their inference $p$-values and select the top $K$. We then fit a separate random forest using those $K$ features, with the same hyperparameters as the random forest used by LOCO-Split.

\subsection{Validation of Assumption~\ref{ass:minDelta_q_bnd}}
To assess the plausibility of Assumption~\ref{ass:minDelta_q_bnd}, we empirically examine the probabilities of adaptive sampling in our simulation studies. We focus on the high-dimensional setting, since Assumption~\ref{ass:minDelta_q_bnd} can be easily satisfied in the low-dimensional setting where $m/M=0.3$. In particular, we use the same simulation setup as in Section~\ref{sec:sim_pred}, and for each iteration $t>=2$, we identify the least important feature and record the largest sampling probability assigned to that feature over the current and all previous iterations except for the initial probability $m/M$. Figure~\ref{fig:ass500} displays the histograms of the recorded maximum probabilities for least important features across $100$ independent replicates and iterations. The results show that these maximal sampling probabilities for least important features are consistently bounded by approximately $Cm/M$ with $C$ being roughly bounded by $3$. This provides empirical evidence that Assumption~\ref{ass:minDelta_q_bnd} is mild and likely satisfied in practice. 
\begin{figure}[!htb]
    \centering    
    \includegraphics[width=\linewidth]{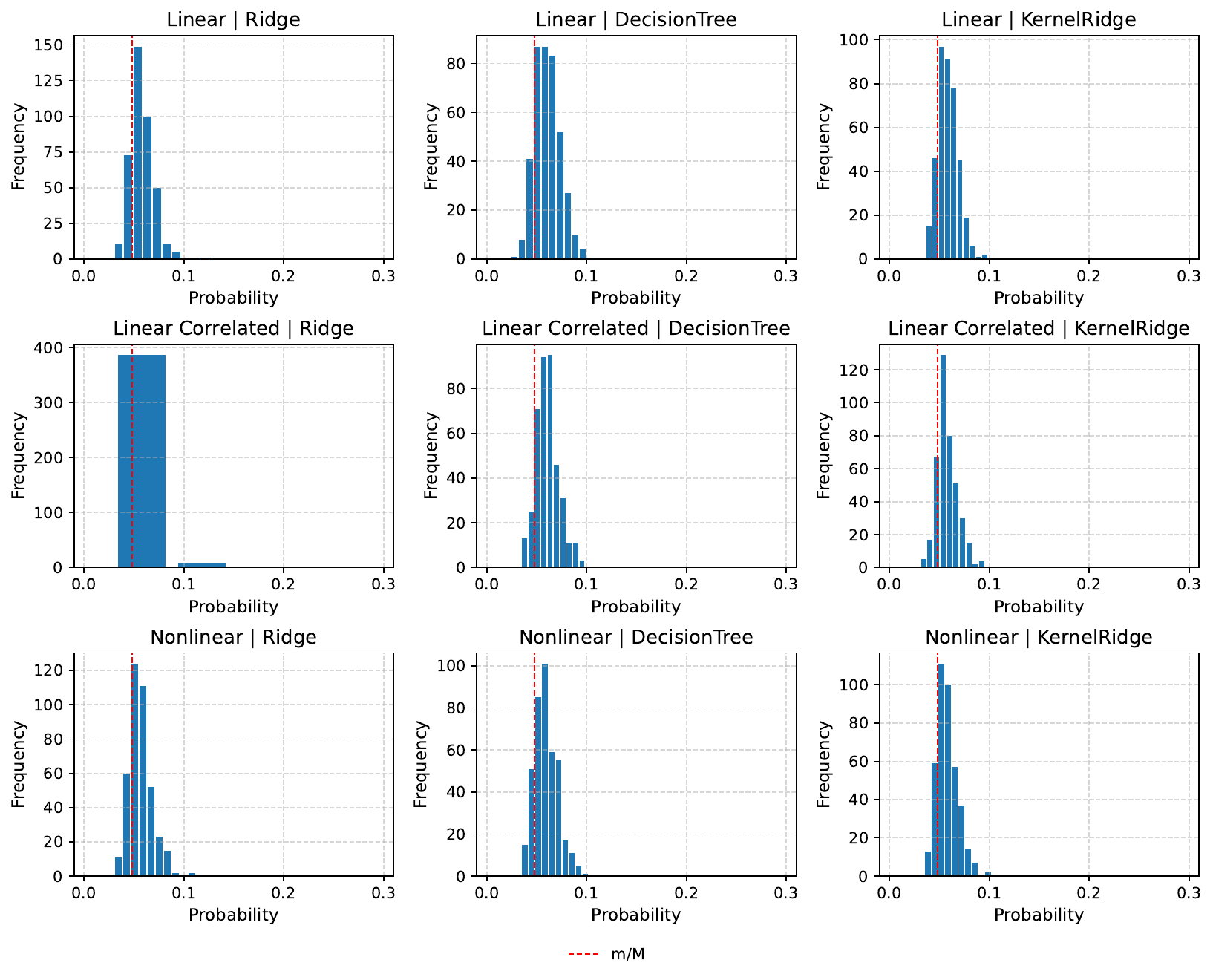}
    \caption{Empirical assessment for Assumption~\ref{ass:minDelta_q_bnd} in the high-dimensional setting $M=500$. The x-axis represents the largest sampling probability ever assigned to a feature identified as least important in one iteration $t>=2$. The y-axis represents the frequency across 100 simulation replicates and all iterations $t>=2$.}
    \label{fig:ass500}
\end{figure}

\subsection{Additional Simulation Results}
\subsubsection{High-dimensional Results}\label{add_high}
We report the coverage validation, confidence intervals, and identification rates for signal features under the high-dimensional setting ($M=500$) in Figures~\ref{fig:cover500}-\ref{fig:power500}.
\begin{figure}[!htb]
    \centering
    \includegraphics[width=\linewidth]{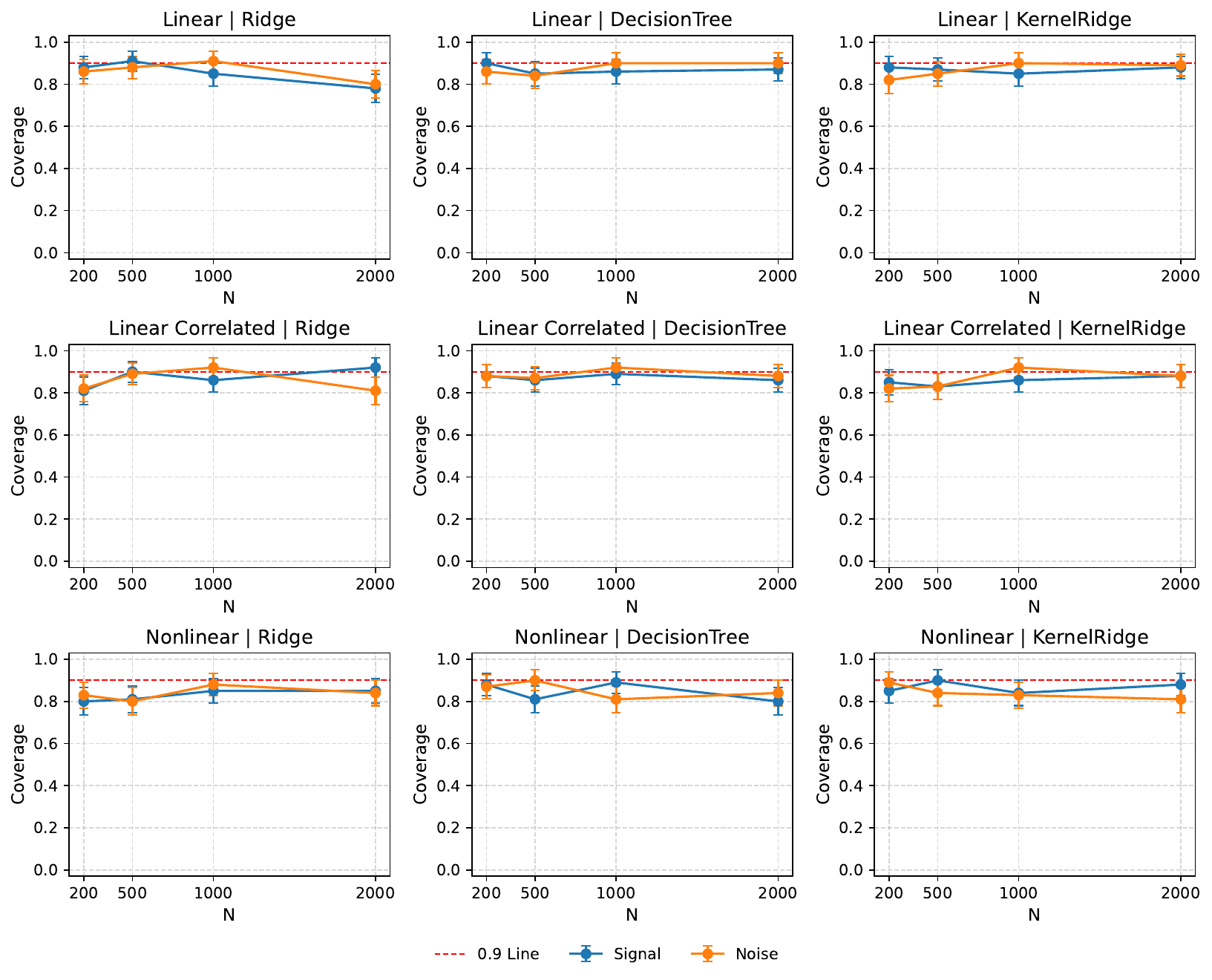}
    \caption{\small Coverage rates of $90\%$ confidence intervals for the targets in the high-dimensional setting $(M=500)$ over $100$ replicates. The first and third rows use the default setting with independent features, whereas the second uses the setting with one correlated feature pair. The blue and orange lines represent a signal feature (feature 5) and a noise feature (feature 6) respectively. The true target value is approximated via Monte Carlo method with $10,000$ test observations; hence the reported coverage rates slightly underestimate the true coverage. To address the underestimation effect due to Monte Carlo randomness, adjusted coverage rates are reported in Figure~\ref{fig:adjcover500}.}
    \label{fig:cover500}
\end{figure}
\begin{figure}[!htb]
    \centering
    \includegraphics[width=\linewidth]{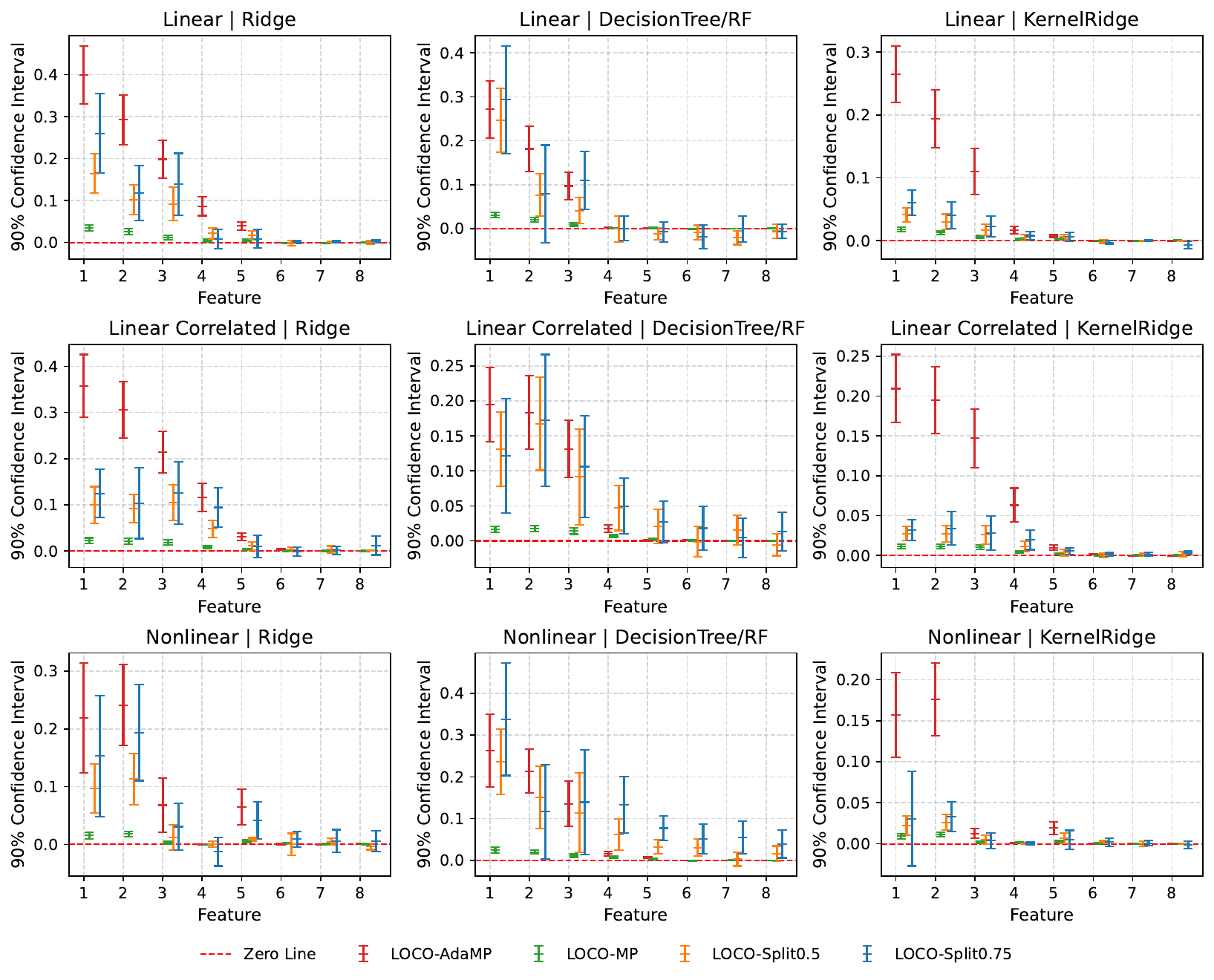}
    \caption{\small 90\% confidence intervals for the first 8 features in the high-dimensional setting ($M=500, N=200$). The first $5$ features are signal features with decreasing strengths. The first and third rows use the default setting with independent features, whereas the second uses the setting with one correlated feature pair. For the second column, decision trees are used for LOCO-MP and LOCO-AdaMP, and random forests are used for LOCO-Split. Unlike the low-dimensional results (Figure~\ref{fig:ci50}), LOCO-AdaMP's interval magnitudes are larger than LOCO-Split when the base model is ridge/kernel ridge, perhaps due to that overly large shrinkage of ridge penalty in high-dimensional settings.}
    \label{fig:ci500}
\end{figure}
\begin{figure}[!htb]
    \centering
    \includegraphics[width=\linewidth]{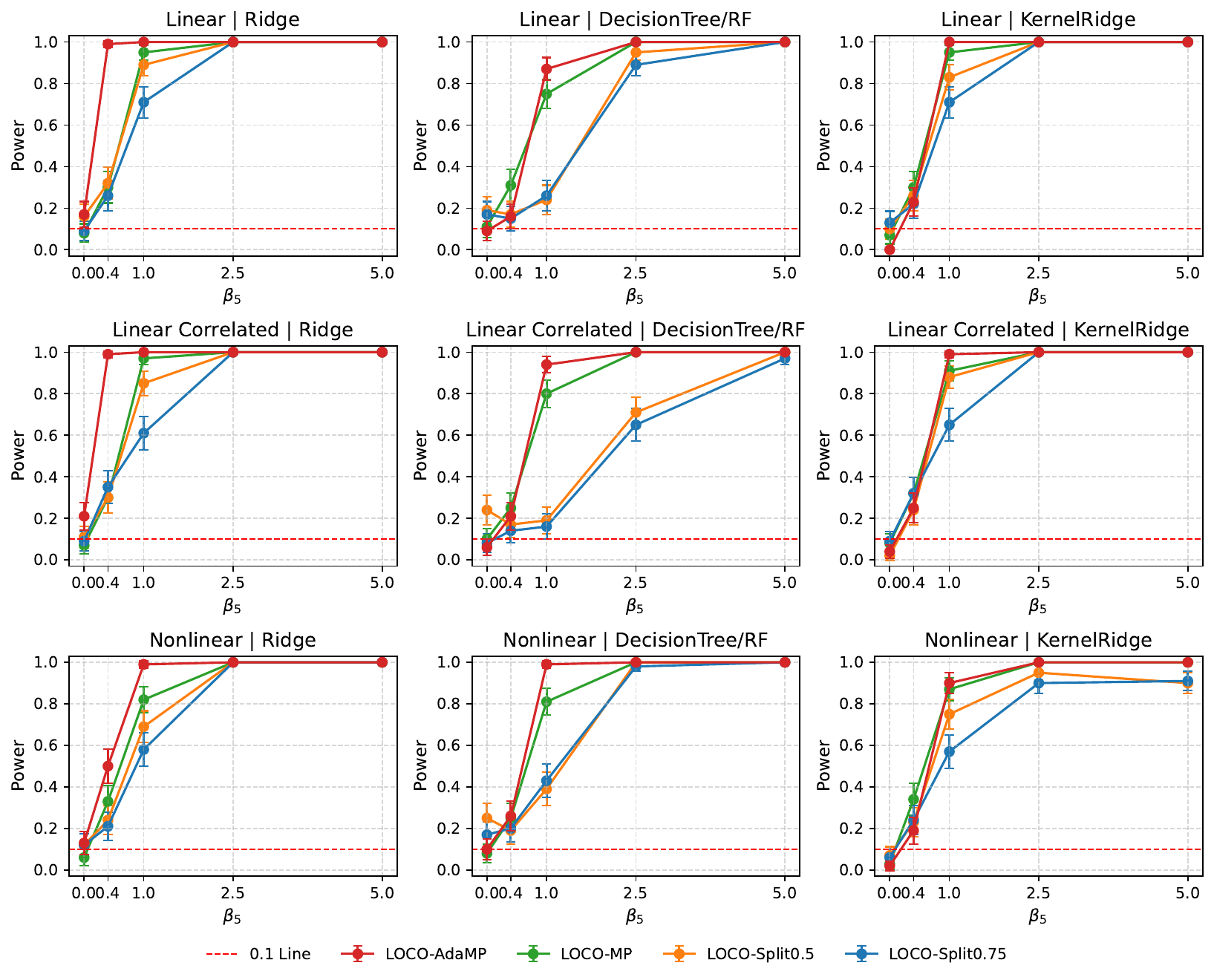}
    \caption{\small Rejection rates of one-sided hypothesis tests ($\alpha = 0.1$) for feature 5 in the high-dimensional setting $(M = 500, N = 200)$, averaged over 100 replicates. The x-axis represents values for $\beta_5$. LOCO-AdaMP quickly attains higher rejection rates when $\beta_5$ increases above zero. The first and third rows use the default setting with independent features, whereas the second uses the setting with one correlated feature pair. LOCO-AdaMP still attains the highest rejection rate as $\beta_5$ increases beyond zero.}
    \label{fig:power500}
\end{figure}

\subsubsection{Inference Targets over Iterations}\label{add_inf_tar}
We present additional figures tracking the inference targets over LOCO-AdaMP iterations, complementing Figure~\ref{fig:target50_ind} in Section~\ref{sec:disc_target}. We use the same simulation setup as in Section~\ref{sec:sim_pred}. Figure~\ref{fig:target50_cn} demonstrates the result for the correlated linear and nonlinear models in the low-dimensional setting, while Figure~\ref{fig:target500} presents the result in the high-dimensional setting, showing qualitatively similar trends to Figure~\ref{fig:target50_ind}. Moreover, we show the inference targets with y-axis on a symmetric logarithmic scale in Figures~\ref{fig:log_target50} and \ref{fig:log_target500}.
\begin{figure}[!htb]
    \centering
    \includegraphics[width=\linewidth]{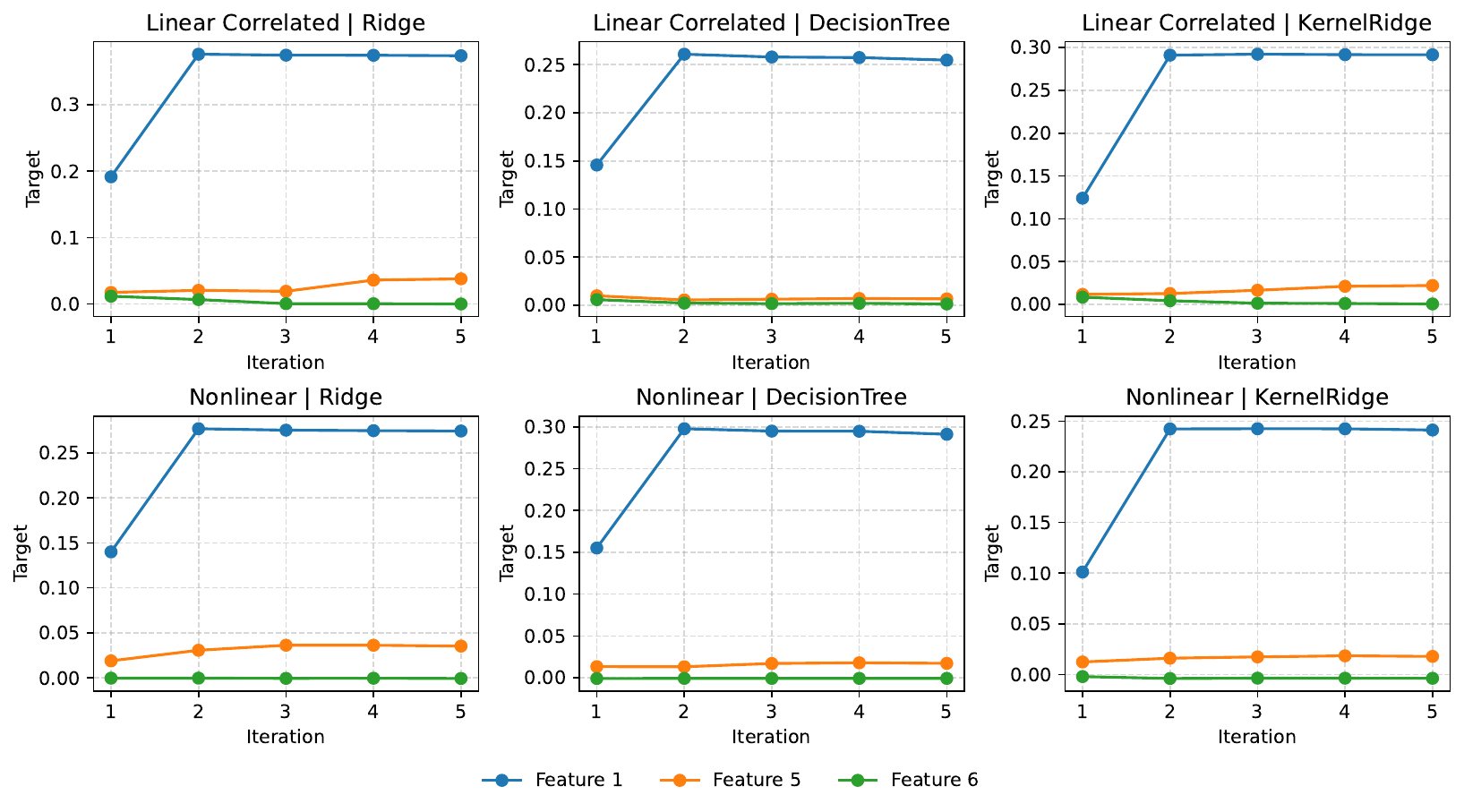}
    \caption{\small Inference targets $\Delta_j^{(t)}$ of LOCO-AdaMP over iterations $t$ in the low-dimensional setting ($M = 50$) for correlated linear model and nonlinear model. Feature 1 and feature 5 are signal features with SNR$=2.5$ and SNR$=1$ respectively, while feature 6 is a noise feature. Results are averaged over $100$ simulation replicates.}
    \label{fig:target50_cn}
\end{figure}

\begin{figure}[!htb]
    \centering
    \includegraphics[width=\linewidth]{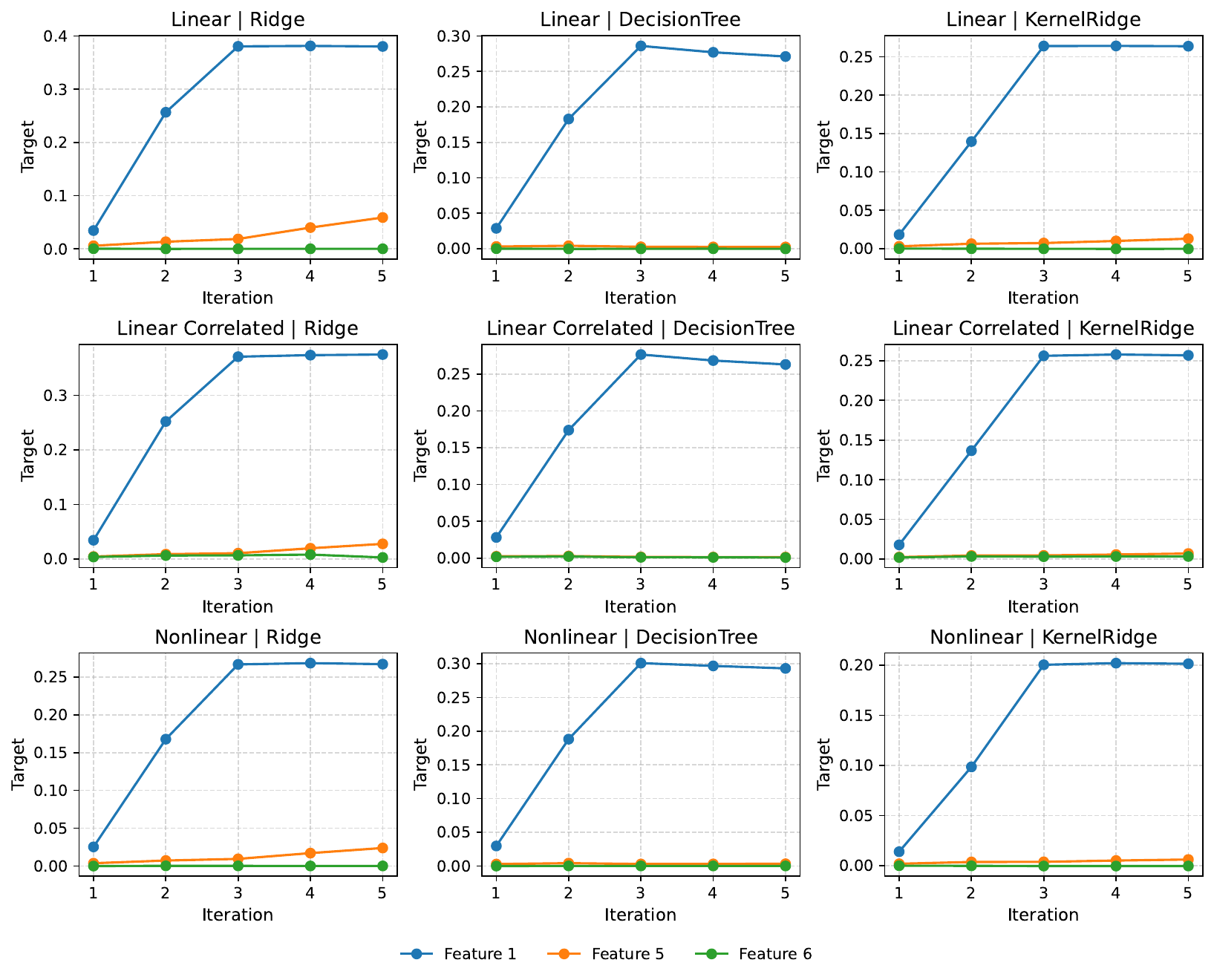}
    \caption{\small Inference targets $\Delta_j^{(t)}$ of LOCO-AdaMP over iterations $t$ in the high-dimensional setting ($M = 500$). Feature 1 and feature 5 are signal features with SNR$=2.5$ and SNR$=1$ respectively, while feature 6 is a noise feature. Results are averaged over $100$ simulation replicates.}
    \label{fig:target500}
\end{figure}

\begin{figure}[!htb]
    \centering
    \includegraphics[width=\linewidth]{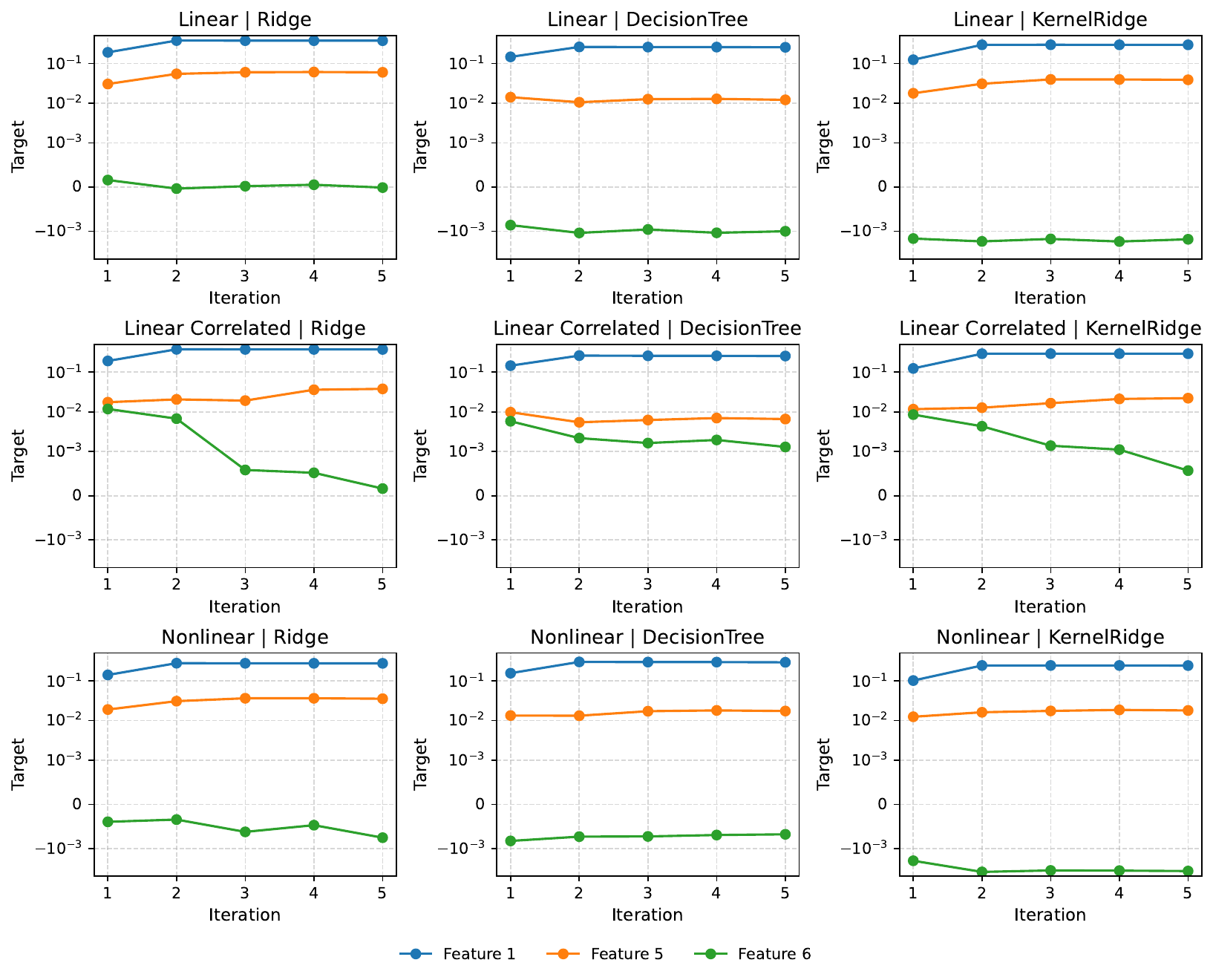}
    \caption{\small Log-scaled inference targets $\log\Delta_j^{(t)}$ of LOCO-AdaMP over iterations $t$ in the low-dimensional setting ($M = 50$). Feature 1 and feature 5 are signal features with SNR$=2.5$ and SNR$=1$ respectively, while feature 6 is a noise feature. Results are averaged over $100$ simulation replicates. The y-axis uses a symmetric logarithmic scale with a linear region within $[-10^{-3}, 10^{-3}]$.}
    \label{fig:log_target50}
\end{figure}

\begin{figure}[!htb]
    \centering
    \includegraphics[width=\linewidth]{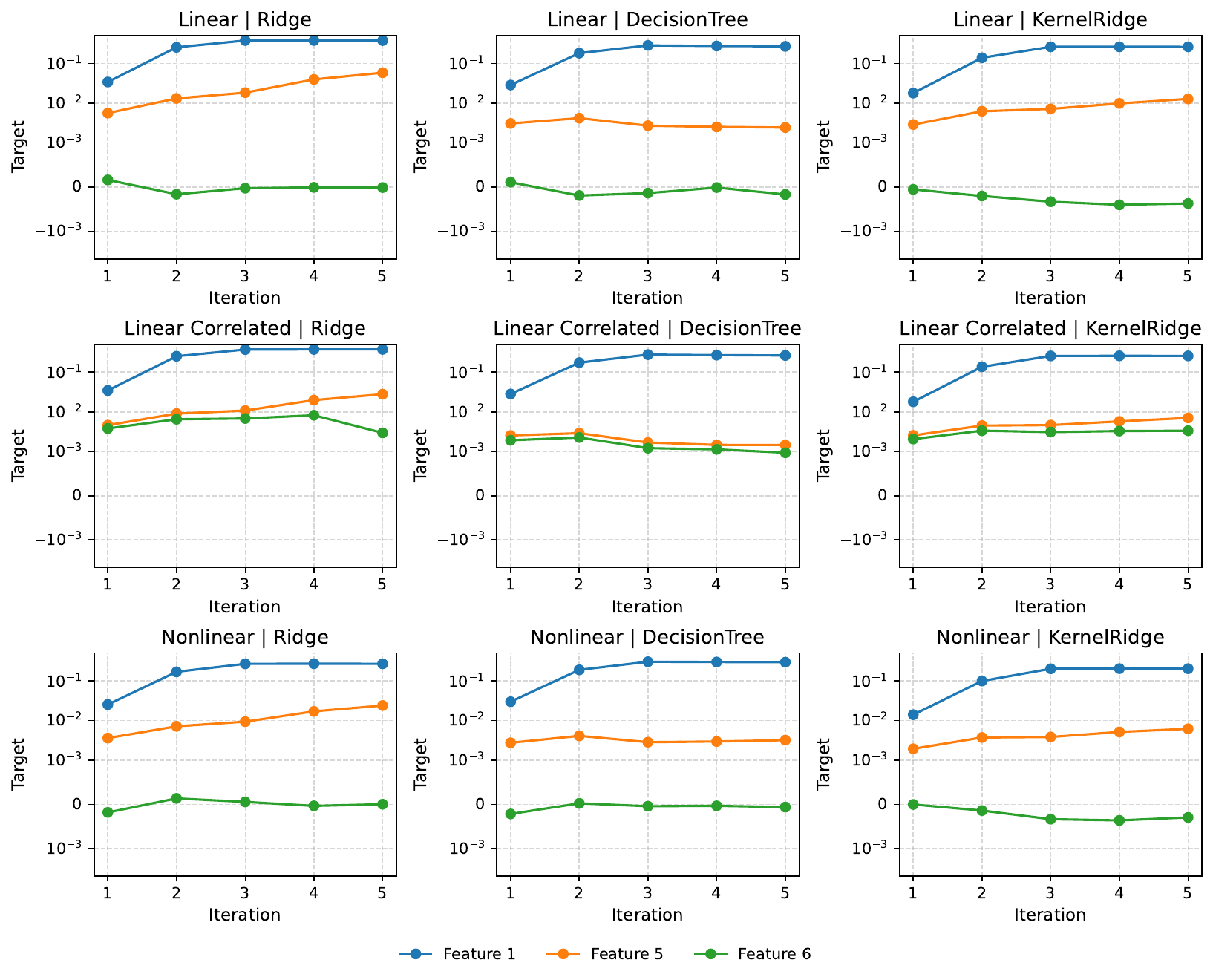}
    \caption{\small Log-scaled inference targets $\log\Delta_j^{(t)}$ of LOCO-AdaMP over iterations $t$ in the high-dimensional setting ($M = 500$). Feature 1 and feature 5 are signal features with SNR$=2.5$ and SNR$=1$ respectively, while feature 6 is a noise feature. Results are averaged over $100$ simulation replicates. The y-axis uses a symmetric logarithmic scale with a linear region within $[-10^{-3}, 10^{-3}]$.}
    \label{fig:log_target500}
\end{figure}

\subsubsection{Adjusted Coverage Results}\label{adj_cover}
The true inference target is difficult to evaluate because it involves an expectation over unseen data. We therefore approximate it using a Monte Carlo estimator based on an independent test sample of size $N_1$. Directly treating this surrogate as the true target may underestimate the empirical coverage rate because the surrogate itself contains random error. Throughout this section, we consider the inference target of the final iteration $T$ and suppress the iteration index $T$ in the notation. Denote the surrogate by $\tilde{\Delta}_j$:
$$
\tilde{\Delta}_j=\frac{1}{N_1}\sum_{l=1}^{N_1}(\Delta_j(\bZ))_l,\quad (\Delta_j(\bZ))_l=\ell(Z_l,\mu_{\no j}(\cdot;\bZ_{:,\no j})) - \ell(Z_l,\mu(\cdot;\bZ)).
$$
Then under the asymptotic normality result (Theorem~\ref{thm:main}), we have $$\bar{\Delta}_{j} \overset{d.}{\rightarrow}  \cN\left(\Delta_j,\frac{\sigma_{j}^2}{N}\right),$$
and $$\tilde{\Delta}_{j} \overset{d.}{\rightarrow}  \cN\left(\Delta_j,\frac{\tilde\sigma_{j}^2}{N_1}\right),$$ where $\tilde\sigma_{j}^2$ is the conditional variance of the LOCO importance over random test data. 
Since the training and test data are independent, Theorem~\ref{thm:main} implies that $$\bar{\Delta}_{j} - \tilde{\Delta}_{j} \overset{d.}{\rightarrow}\cN\left(0,\frac{\sigma_{j}^2}{N}+\frac{\tilde\sigma_{j}^2}{N_1}\right).$$ Therefore, to address the coverage underestimation effect of Monte Carlo approximation, we construct the adjusted confidence interval $$\hat{\mathbb{C}}^{\text{adj}}_j= \left[\bar{\Delta}_j - z_{\alpha/2}\sqrt{\frac{\hat{\sigma}^2_j}{N}+\frac{\hat{\tilde{\sigma}}^2_j}{N_1}},\bar{\Delta}_j +z_{\alpha/2}\sqrt{\frac{\hat{\sigma}^2_j}{N}+\frac{\hat{\tilde{\sigma}}^2_j}{N_1}} \right],$$ where $\hat{\sigma}_j^2$ is our LOCO-AdaMP variance estimate returned by Algorithm~\ref{algo:loco_adaptive}, and $\hat{\tilde{\sigma}}_j^2$ is the sample variance of $(\Delta_j(\bZ))_l$. We note here that this adjusted confidence interval is not what one would use in practice; it is only an instrumental tool that we construct to validate Theorem~\ref{thm:main}, as Theorem~\ref{thm:main} implies $$\lim_{N\rightarrow\infty}\bbP(\tilde{\Delta}_j\in \hat{\mathbb{C}}^{\text{adj}}_j)=1-\alpha.$$ In the experiment, we set $N_1=10,000$. Figures~\ref{fig:adjcover50} and \ref{fig:adjcover500} show higher and mostly valid coverage rates than Figure~\ref{fig:cover50} and \ref{fig:cover500} respectively for both noise feature and signal feature.

\begin{figure}[!htb]
    \centering
    \includegraphics[width=\linewidth]{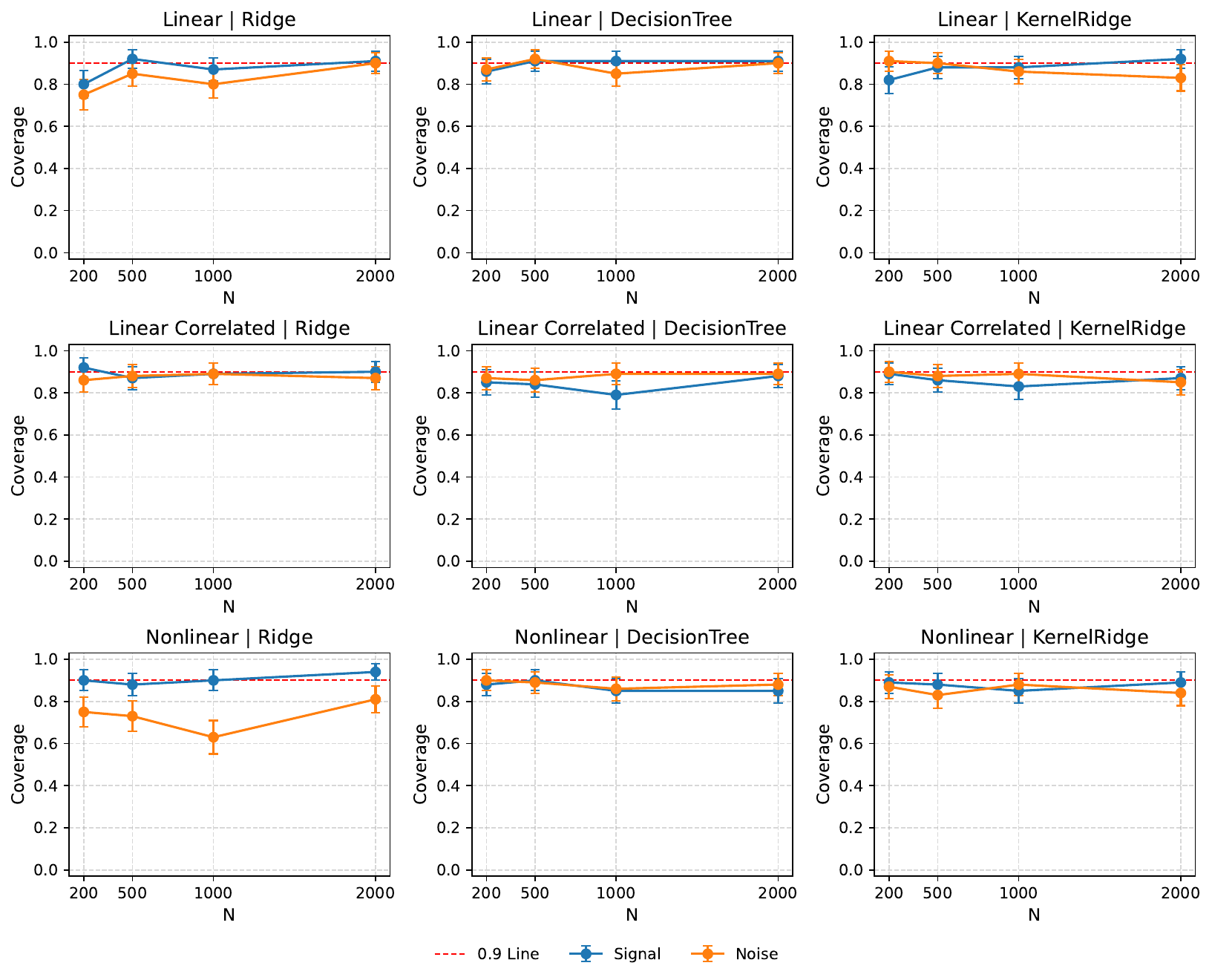}
    \caption{\small Adjusted coverage rates of $90\%$ confidence intervals for the targets in the low-dimensional setting $(M=50)$ over $100$ replicates. The first and third rows use the default setting with independent features, whereas the second uses the setting with one correlated feature pair. The blue and orange lines represent a signal feature (feature 5) and a noise feature (feature 6) respectively.}
    \label{fig:adjcover50}
\end{figure}
\begin{figure}[!htb]
    \centering
    \includegraphics[width=\linewidth]{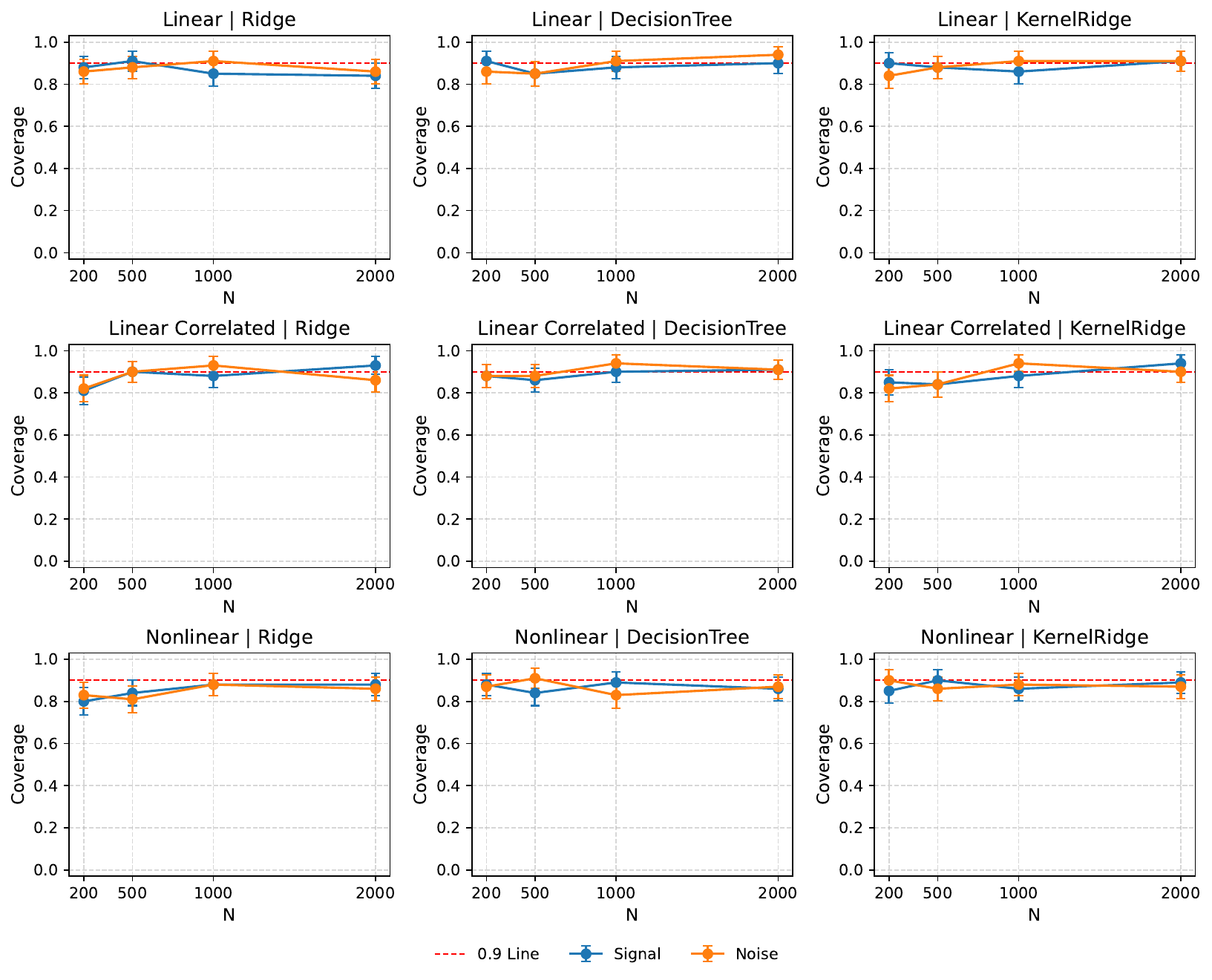}
    \caption{\small Adjusted coverage rates of $90\%$ confidence intervals for the targets in the high-dimensional setting $(M=500)$ over $100$ replicates. The first and third rows use the default setting with independent features, whereas the second uses the setting with one correlated feature pair. The blue and orange lines represent a signal feature (feature 5) and a noise feature (feature 6) respectively.} 
    \label{fig:adjcover500}
\end{figure}

\subsubsection{Experiments with Absolute Error Function in LOCO Metric}
\label{abs_loss}
To examine the robustness of our method to the choice of loss function, we conduct additional simulation studies using the absolute loss function. All simulation settings are kept the same as in Section~\ref{sec:sims}, with the only difference being the choice of loss function. Figure~\ref{fig:cover50_abs}-\ref{fig:power500_abs} show that the results are quantitatively consistent with those under the squared loss function, indicating that the performance of the proposed method is not sensitive to this choice.

\begin{figure}[!htb]
    \centering
    \includegraphics[width=\linewidth]{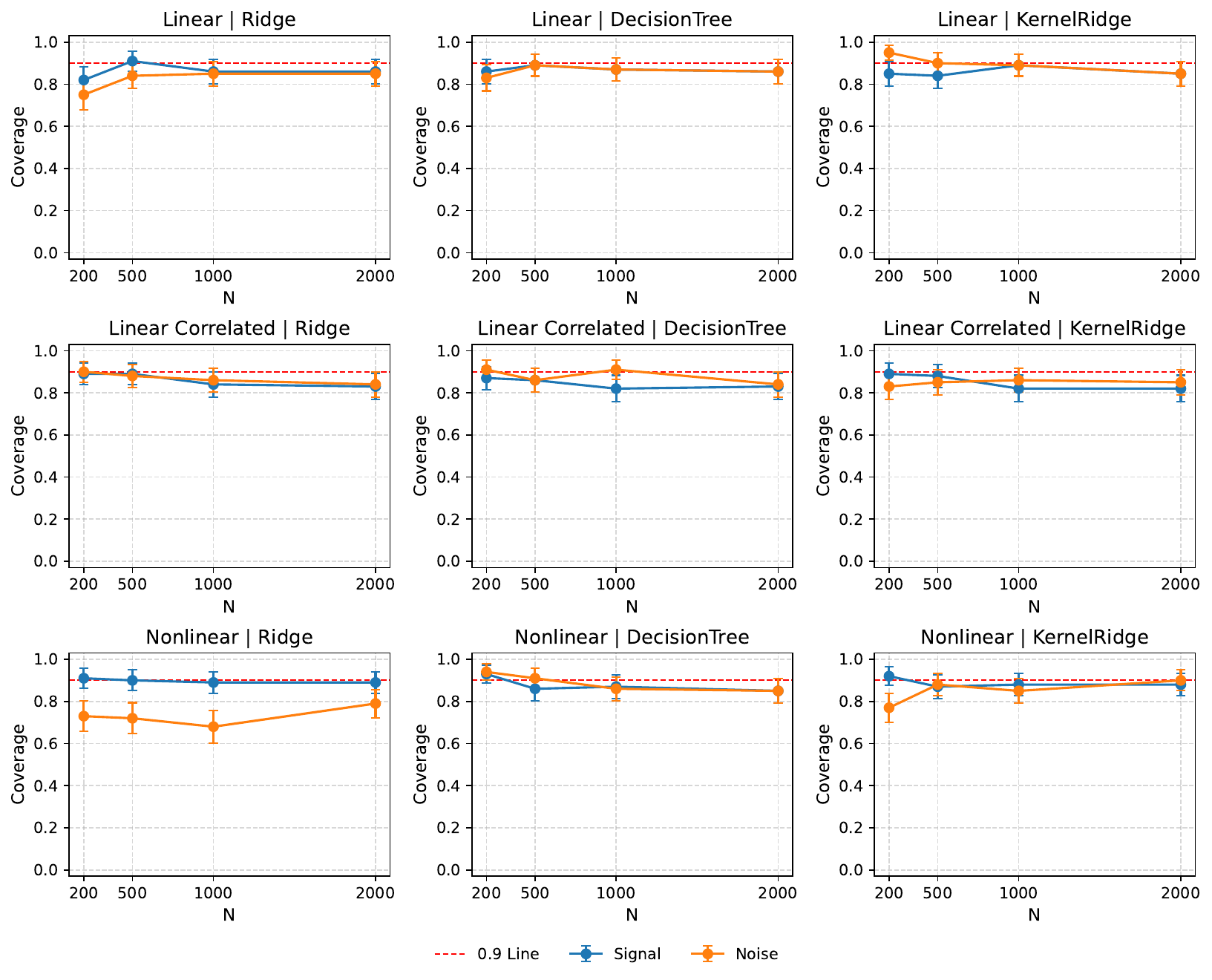}
    \caption{\small Coverage rates of $90\%$ confidence intervals for the targets in the low-dimensional setting $(M=50)$ over $100$ replicates. The first and third rows use the default setting with independent features, whereas the second uses the setting with one correlated feature pair. The blue and orange lines represent a signal feature (feature 5) and a noise feature (feature 6) respectively. The true target value is approximated via Monte Carlo method with $10,000$ test observations; hence the reported coverage rates slightly underestimate the true coverage.}
    \label{fig:cover50_abs}
\end{figure}
\begin{figure}[!htb]
    \centering
    \includegraphics[width=\linewidth]{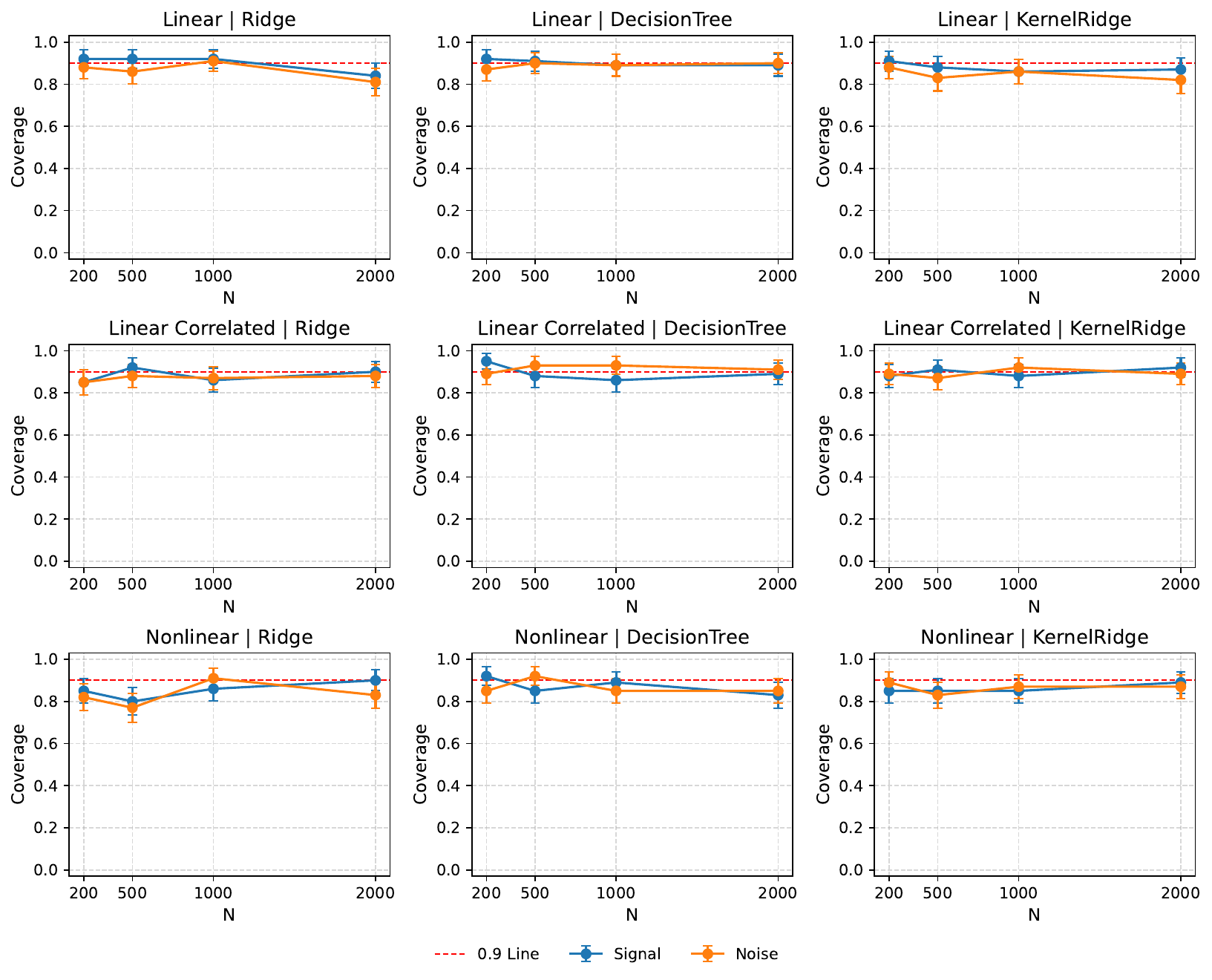}
    \caption{\small Coverage rates of $90\%$ confidence intervals for the targets in the high-dimensional setting $(M=500)$ over $100$ replicates. The first and third rows use the default setting with independent features, whereas the second uses the setting with one correlated feature pair. The blue and orange lines represent a signal feature (feature 5) and a noise feature (feature 6) respectively. The true target value is approximated via Monte Carlo method with $10,000$ test observations; hence the reported coverage rates slightly underestimate the true coverage.}
    \label{fig:cover500_abs}
\end{figure}

\begin{figure}[!htb]
    \centering
    \includegraphics[width=1\linewidth]{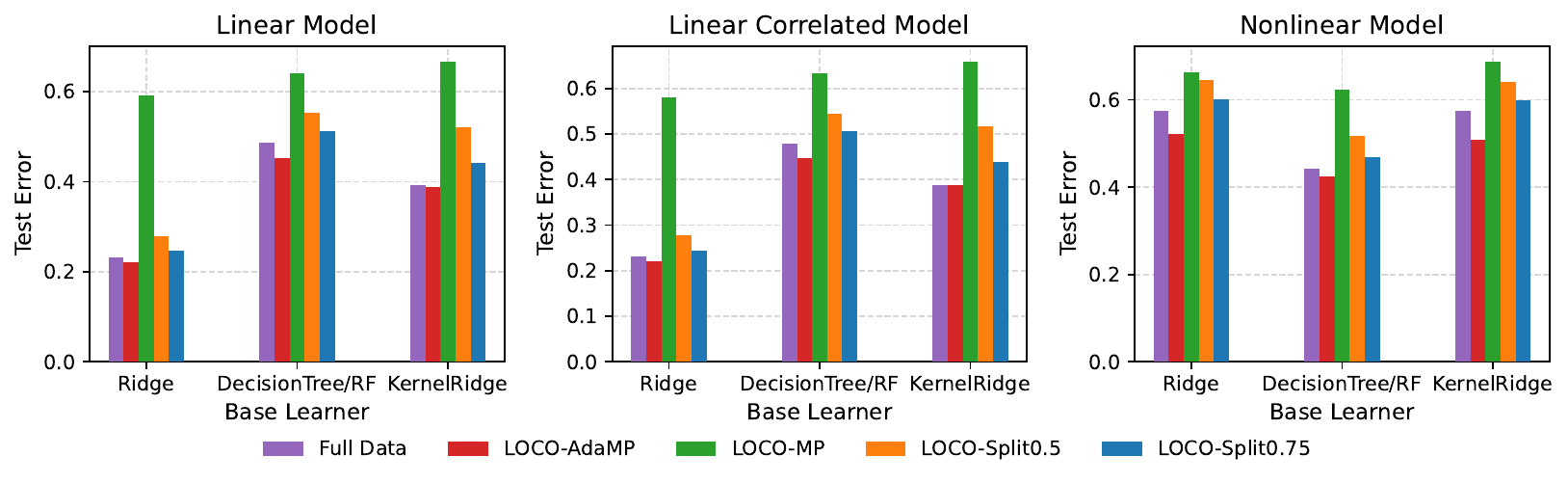}
    \includegraphics[width=1\linewidth]{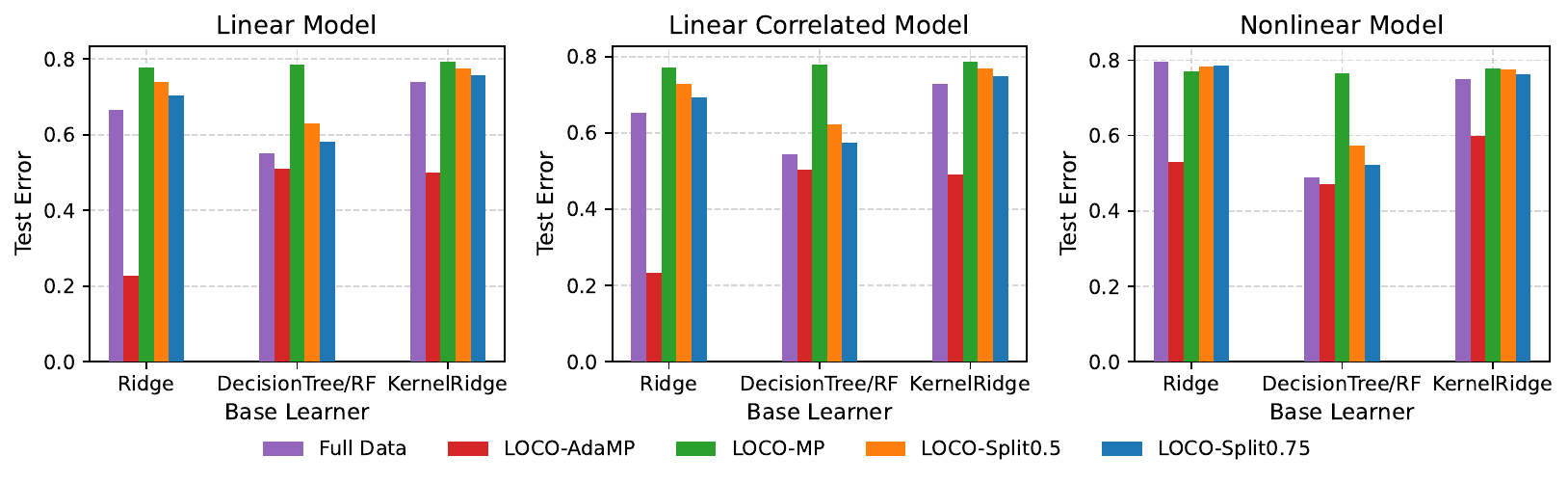}
    \captionsetup{font=normalsize}  
    \caption{\small Prediction error comparison between our LOCO-AdaMP and baseline LOCO methods; the top row is the low-dimensional setting ($M=50, N=200$) and the bottom row is the high-dimensional setting ($M=500, N=200$). ``Full Data'' denotes the corresponding base learner trained on the entire data set and is shown as reference for what can be achieved without requiring LOCO inference. The first and third columns correspond to the default independent features, whereas the second considers one correlated feature pair.}
    \label{fig:pred_abs}
\end{figure}

\begin{figure}[!htb]
    \centering
    \includegraphics[width=\linewidth]{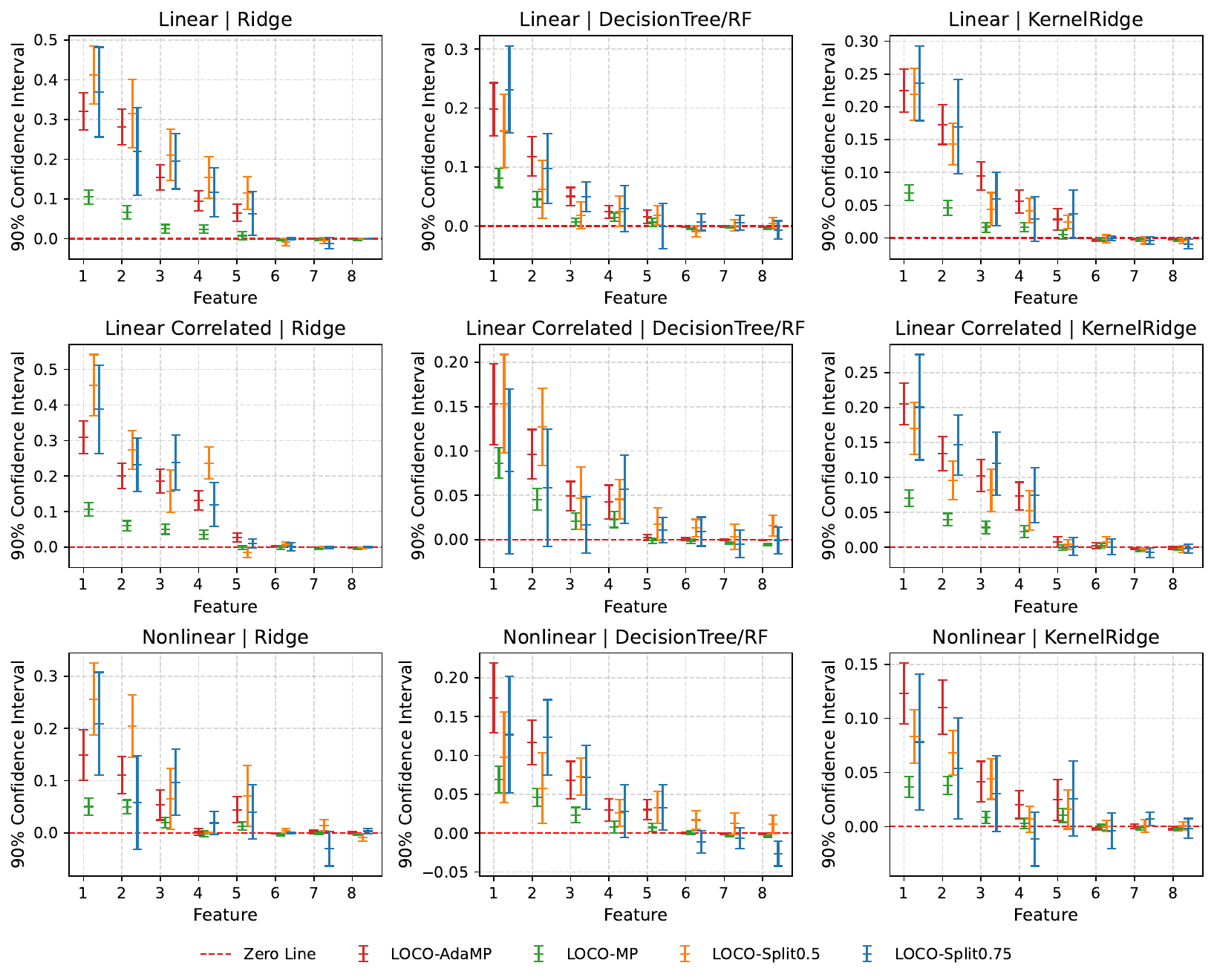}
    \caption{\small 90\% confidence intervals for the first 8 features in the low-dimensional setting ($M=50, N=200$). The first $5$ features are signal features with decreasing strengths. The first and third rows use the default setting with independent features, whereas the second uses the setting with one correlated feature pair. For the second column, decision trees are used for LOCO-MP and LOCO-AdaMP, and random forests are used for LOCO-Split.}
    \label{fig:ci50_abs}
\end{figure}

\begin{figure}[!htb]
    \centering
    \includegraphics[width=\linewidth]{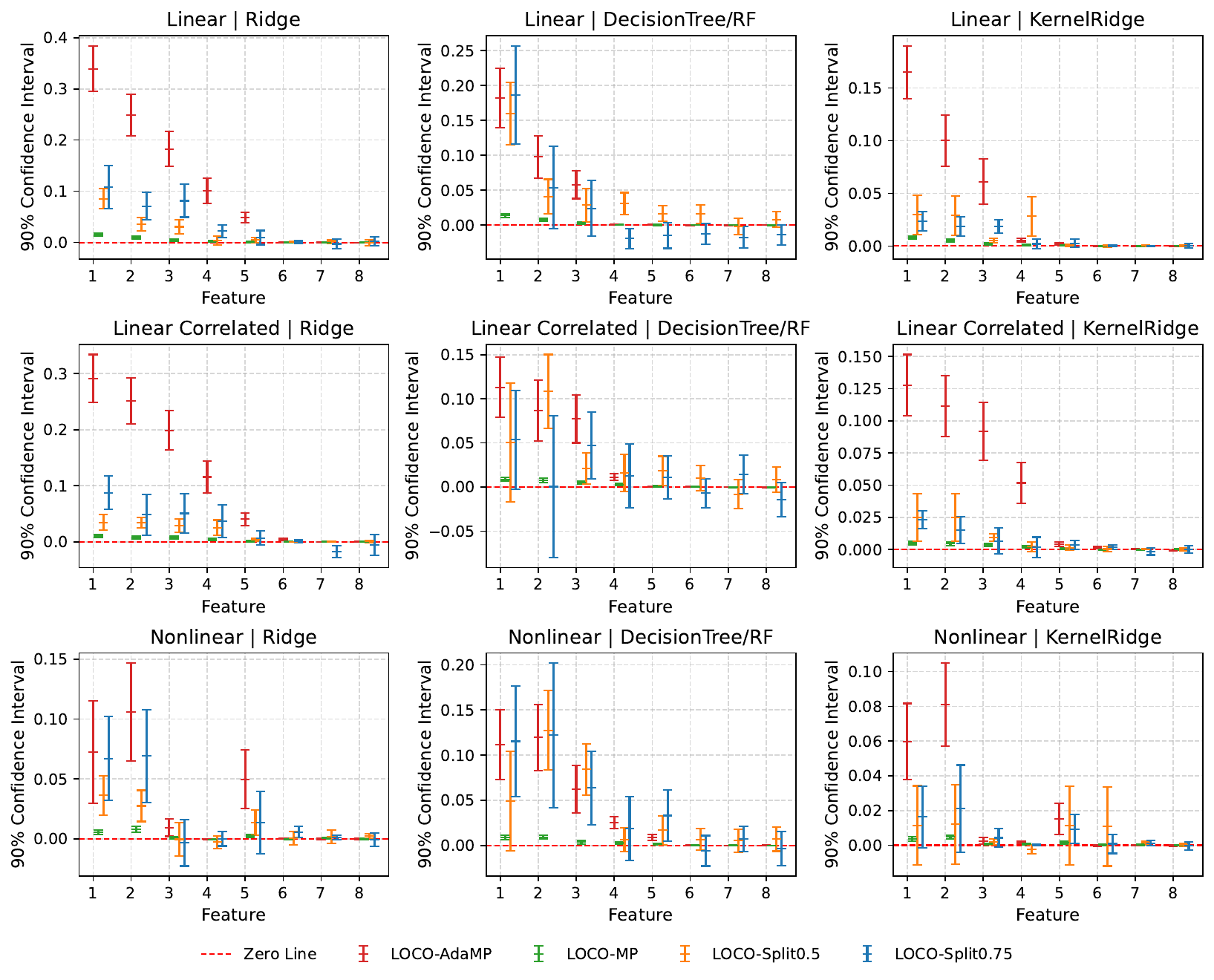}
    \caption{\small 90\% confidence intervals for the first 8 features in the high-dimensional setting ($M=500, N=200$). The first $5$ features are signal features with decreasing strengths. The first and third rows use the default setting with independent features, whereas the second uses the setting with one correlated feature pair. For the second column, decision trees are used for LOCO-MP and LOCO-AdaMP, and random forests are used for LOCO-Split.}
    \label{fig:ci500_abs}
\end{figure}

\begin{figure}[!htb]
    \centering
    \includegraphics[width=\linewidth]{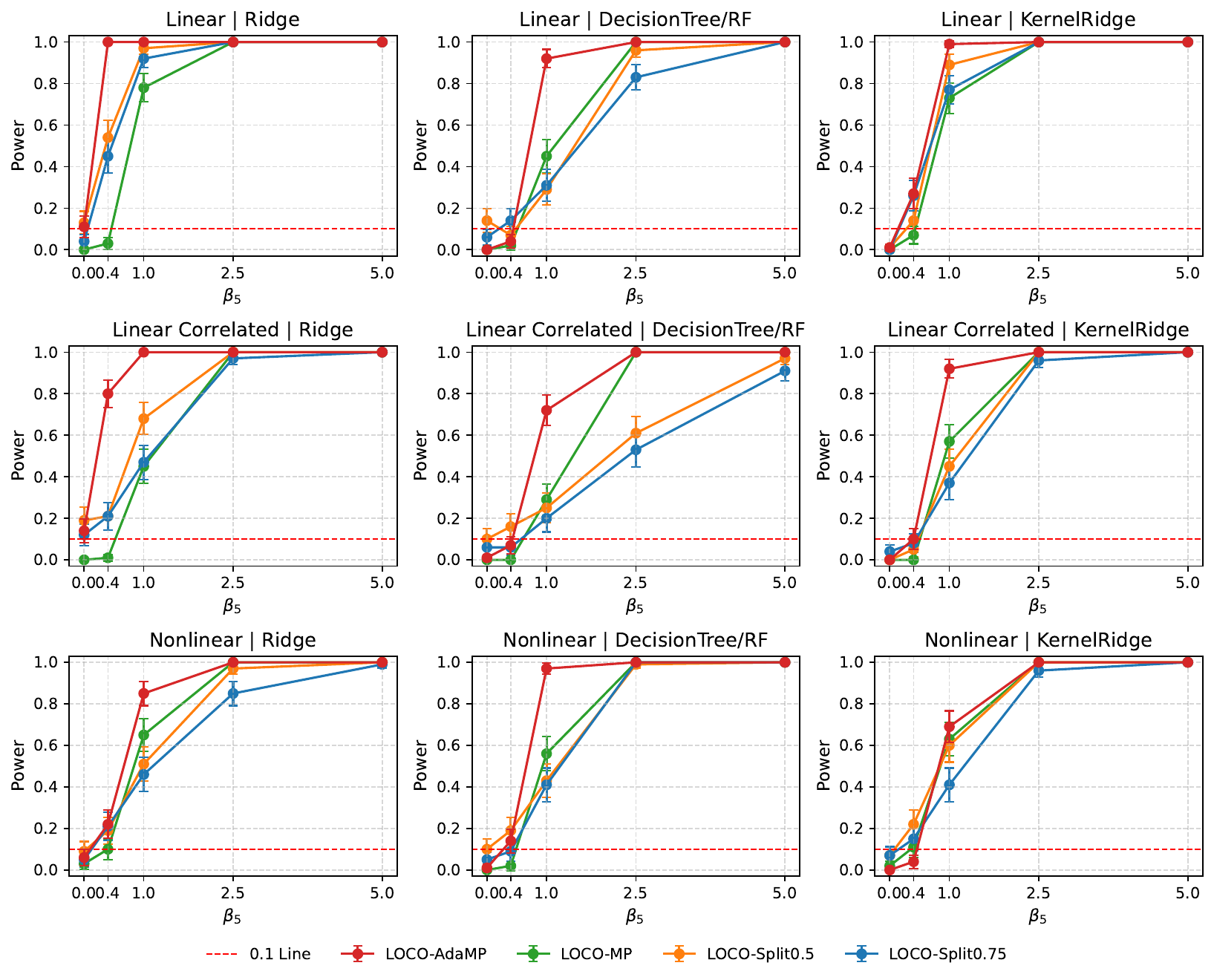}
    \caption{\small Rejection rates of one-sided hypothesis tests ($\alpha = 0.1$) for feature 5 in the low-dimensional setting $(M = 50, N = 200)$, averaged over 100 replicates. The x-axis represents values for $\beta_5$. LOCO-AdaMP quickly attains higher rejection rates when $\beta_5$ increases above zero. The first and third rows use the default setting with independent features, whereas the second uses the setting with one correlated feature pair.}
    \label{fig:power50_abs}
\end{figure}

\begin{figure}[!htb]
    \centering
    \includegraphics[width=\linewidth]{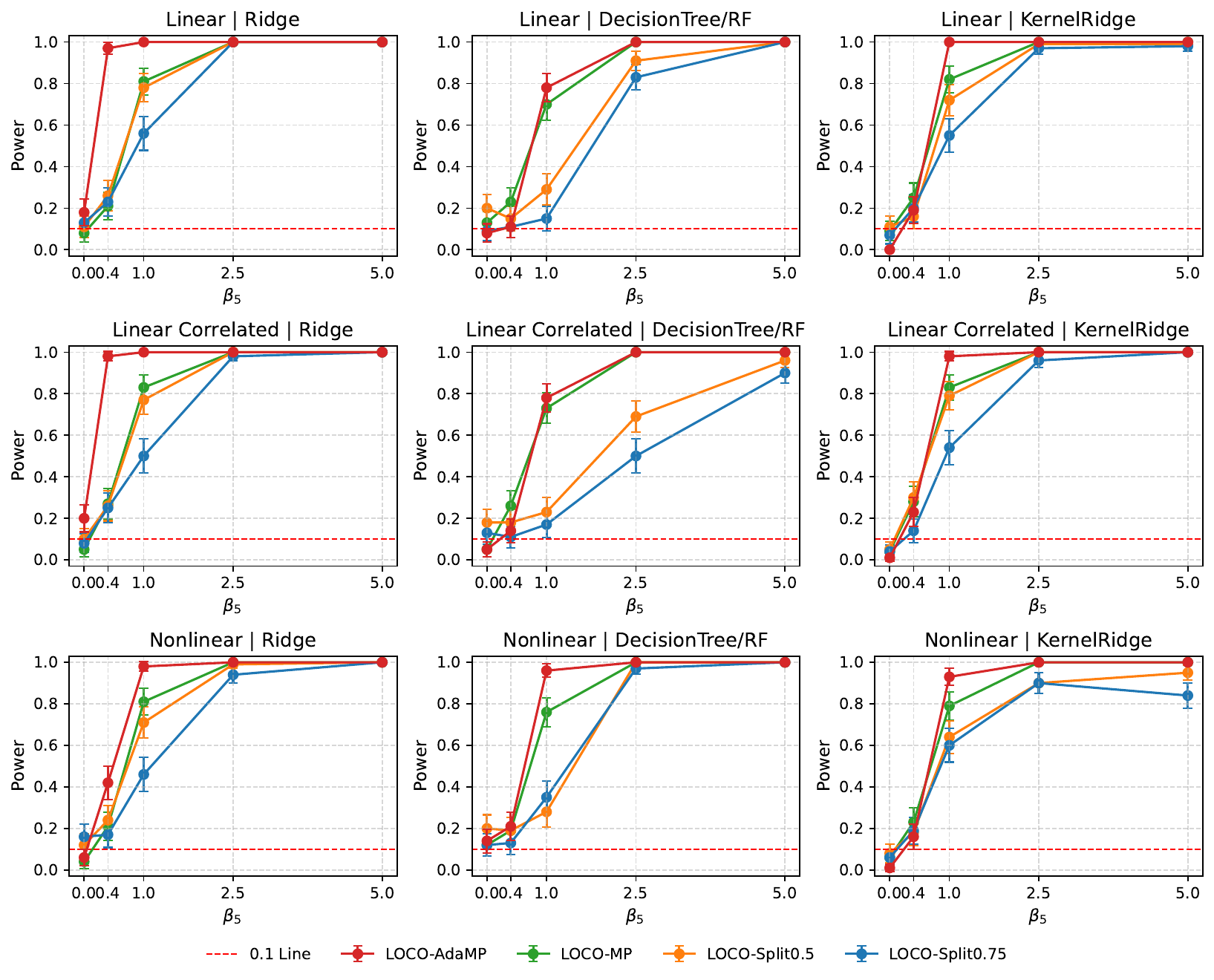}
    \caption{\small Rejection rates of one-sided hypothesis tests ($\alpha = 0.1$) for feature 5 in the high-dimensional setting $(M = 500, N = 200)$, averaged over 100 replicates. The x-axis represents values for $\beta_5$. LOCO-AdaMP quickly attains higher rejection rates when $\beta_5$ increases above zero. The first and third rows use the default setting with independent features, whereas the second uses the setting with one correlated feature pair.}
    \label{fig:power500_abs}
\end{figure}

\end{document}